\documentclass[11pt]{article}
\usepackage{lmodern}
    \PassOptionsToPackage{numbers, compress}{natbib}

\usepackage[utf8]{inputenc} 
\usepackage[T1]{fontenc}    

\RequirePackage[colorlinks,citecolor=blue,urlcolor=blue]{hyperref}

\usepackage{url}            
\usepackage{booktabs}       
\usepackage{amsfonts}       
\usepackage{nicefrac}
\usepackage{wrapfig}
\usepackage{microtype}      
\usepackage{amsmath, amsthm} 
\usepackage{amsfonts,amssymb}
\usepackage{algorithm,algorithmic}
\usepackage{nicefrac}       
\usepackage{microtype}      

\usepackage{xcolor}
\usepackage{graphicx}
\usepackage{wrapfig,lipsum}
\usepackage{subcaption}
\usepackage{thmtools, thm-restate}
\usepackage{mathrsfs} 

\usepackage[shortlabels]{enumitem}

\newcommand{\cost}{\mathsf{c}}
\newcommand{\Siter}{\ell}
\newcommand{\prm}{\mathcal{M}^{1}_+}
\newcommand{\sink}{S_\varepsilon}
\newcommand{\mmd}{\textnormal{MMD}}

\newcommand{\oteps}{\textnormal{OT}_\varepsilon}

\newcommand{\dom}{\mathcal{X}}

\newcommand{\bary}{\mathsf{B}_\varepsilon}

\newcommand{\finmeas}{\mathcal{M}}

\newcommand{\eqals}[1]{\begin{align*}#1\end{align*}}

\newcommand{\bpr}{\begin{proof}}
\newcommand{\epr}{\end{proof}}
\newcommand{\be}{\begin{equation}}
\newcommand{\ee}{\end{equation}}

\newcommand{\bd}{\begin{definition}}
\newcommand{\ed}{\end{definition}}

\newcommand{\bi}{\begin{itemize}}
\newcommand{\ei}{\end{itemize}}

\newcommand{\ba}{\begin{ass}}
\newcommand{\ea}{\end{ass}}

\newcommand{\br}{\begin{remark}}
\newcommand{\er}{\end{remark}}

\newcommand{\bp}{\begin{proposition}}
\newcommand{\ep}{\end{proposition}}

\newcommand{\blm}{\begin{lemma}}
\newcommand{\elm}{\end{lemma}}

\newcommand{\bt}{\begin{theorem}}
\newcommand{\et}{\end{theorem}}

\newcommand{\bcor}{\begin{corollary}}
\newcommand{\ecor}{\end{corollary}}

\newcommand{\bex}{\begin{example}}
\newcommand{\eex}{\end{example}}

\usepackage[capitalize]{cleveref}

\crefname{assumption}{Assumption}{Assumptions}
\crefname{equation}{Eq.}{Eqs.}
\crefname{figure}{Fig.}{Fig.}
\crefname{table}{Table}{Tables}
\crefname{section}{Sec.}{Sec.}
\crefname{theorem}{Thm.}{Thm.}
\crefname{lemma}{Lemma}{Lemmas}
\crefname{corollary}{Cor.}{Cor.}
\crefname{example}{Example}{Examples}
\crefname{remark}{Remark}{Remarks}
\crefname{algorithm}{Alg.}{Algorightms}

\crefname{appendix}{Appendix}{Appendices}

\makeatletter
\def\@endtheorem{\endtrivlist}
\makeatother

\usepackage{etoolbox}

\declaretheorem[name=Theorem,refname=Theorem]{theorem}
\declaretheorem[name=Lemma,sibling=theorem]{lemma}
\declaretheorem[name=Fact,sibling=theorem]{fact}
\declaretheorem[name=Proposition,refname=Proposition,sibling=theorem]{proposition}
\declaretheorem[name=Remark,sibling=theorem]{remark}
\declaretheorem[name=Corollary,refname=Corollary,sibling=theorem]{corollary}
\declaretheorem[name=Definition,refname=Definition]{definition}

\declaretheorem[name=Example]{example}

\providecommand{\abs}[1]{\lvert#1\rvert}

\newcommand{\func}{\ensuremath{G}}

\newcommand{\R}{\ensuremath \mathbb{R}}

\newcommand{\xx}{\ensuremath w}
\newcommand{\zz}{\ensuremath z}

\newcommand{\uu}{\ensuremath v}
\newcommand{\N}{\ensuremath \mathbb{N}}

\DeclareMathOperator*{\argmax}{argmax}
\DeclareMathOperator*{\argmin}{argmin}

\let\mathsf\relax    
\DeclareRobustCommand{\mathsf}[1]{\text{\normalfont\sffamily#1}}

\newcommand{\msf}[1]{\mathsf{#1}}
\newcommand{\mbf}[1]{\mathbf{#1}}

\newcommand{\X}{{\mathcal X}}

\newcommand{\hh}{{\mathcal{H}}}
\newcommand{\BB}{{\mathcal{W}^*}}
\newcommand{\VV}{{\mathcal{W}}}
\newcommand{\CC}{{\mathcal{D}}}

\newcommand{\cont}{{\mathcal C}}
\newcommand{\nneg}{{\cont_{+}}}
\newcommand{\posi}{{\cont_{++}}}
\newcommand{\meas}{{\mathcal{M}}}
\newcommand{\prob}{\meas_1^+}

\newcommand{\tmap}{{\msf A}}
\newcommand{\lmap}{{\msf L}}
\newcommand{\inv}{{\msf R}}

\newcommand{\dist}{{\msf{c}}}

\newcommand{\diam}{{\msf{D}}}
\newcommand{\Diam}{\ensuremath{\mathrm{diam}}}

\newcommand{\ran}{\ensuremath{\text{\rm Ran}}}

\newcommand{\supp}{\ensuremath{\text{\rm supp}}}

\newcommand{\kerfun}{{\msf k}}
\newcommand{\kersob}{{\msf h}}
\newcommand{\diameps}{{\diam/\varepsilon}}

\newcommand{\sinkconst}{{\msf{r}}}

\newcommand{\timesconst}{{\msf{m}_1}}
\newcommand{\expconst}{{\msf{m}_2}}
\newcommand{\extensionconst}{{\msf{m}_3}}

\renewcommand{\paragraph}[1]{~\newline\noindent{\bf #1.}}

\newcommand{\emparagraph}[1]{~\newline\noindent{\em #1.}}

\crefname{assumption}{Assumption}{Assumptions}
\crefname{equation}{}{}
\Crefname{equation}{Eq.}{Eqs.}
\crefname{figure}{Fig.}{Fig.}
\crefname{table}{Table}{Tables}
\crefname{section}{Sec.}{Sec.}
\crefname{theorem}{Thm.}{Thm.}
\crefname{proposition}{Prop.}{Prop.}
\crefname{fact}{Fact}{Facts}
\crefname{lemma}{Lemma}{Lemmas}
\crefname{corollary}{Cor.}{Cor.}
\crefname{example}{Example}{Examples}
\crefname{remark}{Remark}{Remarks}
\crefname{algorithm}{Alg.}{Algorightms}

\crefname{appendix}{Appendix}{Appendices}

\providecommand{\scal}[2]{\left\langle{#1},{#2}\right\rangle}

\providecommand{\nor}[1]{\left\|{#1}\right\|}
\providecommand{\supnor}[1]{\left\|{#1}\right\|_\infty}

\providecommand{\norm}[1]{\lVert#1\rVert}
\providecommand{\scal}[2]{\left\langle{#1},{#2}\right\rangle}

\newcommand{\OT}{\textnormal{OT}}
\newcommand{\kl}{\textnormal{KL}}
\newcommand{\fw}{\textnormal{FW}}
\newcommand{\eps}{\varepsilon}

\newcommand{\rmap}{{\msf{T}}}

\newcommand{\anchor}{{x_{o}}}

\newcommand{\feas}{{\mathcal{F}}}

\newcommand{\precision}{\Delta}

\newcommand{\innerfunc}{\varphi}

\renewcommand{\func}{{\msf{G}}}
\renewcommand{\sink}{{\msf{S}_\eps}}

\AfterEndEnvironment{restatable}{\noindent\ignorespacesafterend}
\AfterEndEnvironment{theorem}{\noindent\ignorespacesafterend}
\AfterEndEnvironment{remark}{\noindent\ignorespacesafterend}
\AfterEndEnvironment{example}{\noindent\ignorespacesafterend}
\AfterEndEnvironment{assumption}{\noindent\ignorespacesafterend}
\AfterEndEnvironment{lemma}{\noindent\ignorespacesafterend}
\AfterEndEnvironment{proof}{\noindent\ignorespacesafterend}
\AfterEndEnvironment{corollary}{\noindent\ignorespacesafterend}
\AfterEndEnvironment{proposition}{\noindent\ignorespacesafterend}
\AfterEndEnvironment{definition}{\noindent\ignorespacesafterend}

\title{\LARGE\bf Sinkhorn Barycenters with Free Support via\\ Frank-Wolfe Algorithm\vspace{1em}}

\author{ Giulia Luise$^{1}$ \\ {\footnotesize\em g.luise.16@ucl.ac.uk} \and  Saverio Salzo$^{2}$ \\ {\footnotesize\em saverio.salzo@iit.it} \\ \and  Massimiliano Pontil$^{1,2}$ \\ {\footnotesize\em  m.pontil@cs.ucl.ac.uk ~~} \and  Carlo Ciliberto$^{1,3}$ \\ {\footnotesize\em c.ciliberto@imperial.ac.uk} \\ $ $ \\  }

\begin{document}

\maketitle

\begin{abstract}
\noindent We present a novel algorithm to estimate the barycenter of arbitrary probability distributions with respect to the Sinkhorn divergence. Based on a Frank-Wolfe optimization strategy, our approach proceeds by populating the support of the barycenter incrementally, without requiring any pre-allocation. We consider discrete as well as continuous distributions, proving convergence rates of the proposed algorithm in both settings. Key elements of our analysis are a new result showing that the Sinkhorn divergence on compact domains has Lipschitz continuous gradient with respect to the Total Variation and a characterization of the sample complexity of Sinkhorn potentials. Experiments validate the effectiveness of our method in practice.
\end{abstract}


\section{Introduction}
\footnotetext[1]{Computer Science Department, University College London, WC1E 6BT London, United Kingdom}\footnotetext[2]{Computational Statistics and Machine Learning - Istituto Italiano di Tecnologia, 16100 Genova, Italy}\footnotetext[3]{Electrial and Electronics Engineering Department, Imperial College London, SW7 2BT, United Kingdom.}
Aggregating and summarizing collections of probability measures is a key task in several machine learning scenarios. Depending on the metric adopted, the properties of the resulting average (or {\em barycenter}) of a family of probability measures vary significantly. By design, optimal transport metrics are better suited at capturing the geometry of the distribution than Euclidean distance or other $f$-divergence \cite{cuturi14}. In particular,  Wasserstein barycenters have been successfully used in settings such as texture mixing \cite{rabin2011wasserstein}, Bayesian inference \cite{srivastava2018scalable}, imaging \cite{gramfort2015fast}, or model ensemble \cite{dognin2019wasserstein}.

The notion of barycenter in Wasserstein space was first introduced by \cite{AguehC11} and then investigated from the computational perspective for the original Wasserstein distance \cite{staib2017parallel,stochwassbary} as well as its entropic regularizations (e.g. Sinkhorn) \cite{cuturi14, BenamouCCNP15, decentralized2018}. Two main challenges in this regard are: $i$) how to efficiently identify the support of the candidate barycenter and $ii$) how to deal with continuous (or infinitely supported) probability measures. The first problem is typically addressed by either fixing the support of the barycenter a-priori \cite{staib2017parallel,decentralized2018} or by adopting an alternating minimization procedure to iteratively optimize the support point locations and their weights \cite{cuturi14,stochwassbary}. While fixed-support methods enjoy better theoretical guarantees, free-support algorithms are more memory efficient and practicable in high dimensional settings. The problem of dealing with continuous distributions has been mainly approached by adopting stochastic optimization methods to minimize the barycenter functional \cite{stochwassbary,staib2017parallel,decentralized2018}

In this work we propose a novel method to compute the barycenter of a set of probability distributions with respect to the Sinkhorn divergence \cite{genevay2018learning} that does not require to fix the support beforehand. We address both the cases of discrete and continuous probability measures. In contrast to previous free-support methods, our algorithm does not perform an alternate minimization between support and weights. Instead, we adopt a Frank-Wolfe (FW) procedure to populate the support by incrementally adding new points and updating their weights at each iteration, similarly to kernel herding strategies  \cite{bach2012equivalence} and conditional gradient for sparse inverse problem \cite{bredies2013inverse,boyd2017alternating}. Upon completion of this paper, we found that an idea with similar flavor, concerning the application a Frank-Wolfe scheme in conjunction with Sinkhorn functionals has been very recently considered in distributional regression settings for the case of Sinkhorn negentropies [35]. However, the analysis in this paper focuses on the theoretical properties of the proposed algorithm, specifically for the case of an {\em inexact} Frank-Wolfe procedure, which becomes critical in the case of continuous measures. In particular, we prove the convergence and rates of the proposed optimization method for both finitely and infinitely supported distribution settings. A central result in our analysis is the characterization of regularity properties of Sinkhorn potentials (i.e. the dual solutions of the Sinkhorn divergence problem), which extends recent work in \cite{feydy2018interpolating, genevay2018sample} and which we consider of independent interest. We empirically evaluate the performance of the proposed algorithm.

\paragraph{Contributions} The analysis of the proposed algorithm hinges on the following contributions: $i$) we show that the gradient of the Sinkhorn divergence is Lipschitz continuous on the space of probability measures with respect to the Total Variation. This grants us convergence of the barycenter algorithm in finite settings. $ii$) We characterize the sample complexity of Sinkhorn potentials of two empirical distributions sampled from arbitrary probability measures. This latter result allows us to $iii$) provide a concrete optimization scheme to approximately solve the barycenter problem for arbitrary probability measures with convergence guarantees. $iv$) A byproduct of our analysis is the generalization of the \fw{} algorithm to settings where the objective functional is defined only on a set with empty interior, which is the case for Sinkhorn divergence barycenter problem.



The rest of the paper is organized as follows: \cref{sec:background} reviews standard notions of optimal transport theory. \cref{sec:algorithm-theory} introduces the barycenter functional, and proves the Lipschitz continuity of its gradient.
\cref{sec:algorithm-practice} describes the implementation of our algorithm and \cref{sec:algorithm-convergence} studies its convergence rates. Finally, \cref{sec:experiments} evaluates the proposed methods empirically and \cref{sec:conclusion} provides concluding remarks.

\section{Background}\label{sec:background}

The aim of this section is to recall definitions and properties of Optimal Transport theory with entropic regularization. Throughout the work, we consider a compact set $\dom\subset\R^d$ and a symmetric cost function $\cost\colon\dom\times\dom\to\R$. We set  $\diam := \sup_{x,y\in\X}~\dist(x,y)$ and denote by $\prob(\dom)$ the space of probability measures on $\dom$ (positive Radon measures with mass $1$). For any $\alpha,\beta \in \prob(\dom)$, the Optimal Transport problem with entropic regularization is defined as follow \cite{peyre2017computational,cuturi2013sinkhorn,genevay2016}
\begin{equation}\label{eq:primal_pb}
\oteps(\alpha,\beta) = \min_{\pi\in\Pi(\alpha,\beta)}~\int_{\dom^2}\cost(x,y)\,d\pi(x,y) + \eps\kl(\pi|\alpha\otimes\beta),\qquad \eps\geq0
\end{equation}
where $\kl(\pi|\alpha\otimes\beta)$ is the \emph{Kullback-Leibler divergence} between the candidate transport plan $\pi$ and the product distribution  $\alpha \otimes \beta$, and $\Pi(\alpha,\beta)=\{\pi\in\mathcal{M}_{+}^1(\dom^2)\colon \mathsf{P}_{1\#}\pi=\alpha,\,\,\mathsf{P}_{2\#}\pi=\beta\}$, with $\mathsf{P}_{i}\colon\dom\times \dom\rightarrow\dom$ the projector onto the $i$-th component and $\#$ the push-forward operator. The case $\eps = 0$ corresponds to the classic Optimal Transport problem introduced by Kantorovich \cite{kantorovich1942transfer}. In particular, if $\cost = \nor{\cdot}^p$ for $p\in [1,\infty)$, then $\OT_0$ is the well-known $p$-Wasserstein distance \cite{villani2008optimal}.
Let $\eps>0$. Then, the dual problem
of \cref{eq:primal_pb},
in the sense of  Fenchel-Rockafellar,
is \cite{chizat2018scaling,feydy2018interpolating} 
\begin{equation}\label{eq:dual_pb}
\oteps(\alpha,\beta) = \max_{u,v\in \cont(\dom)} \int u(x)\,d\alpha(x) + \int v(y)\,d\beta(y) -\eps \int e^{\frac{u(x) + v(y) - \cost(x,y)}{\eps}}\,d\alpha(x)d\beta(y), 
\end{equation}
where $\cont(\dom)$ denotes the space of real-valued continuous functions on $\dom$, endowed with $\norm{\cdot}_{\infty}$. Let $\mu\in\prob(\dom)$. We denote by $\rmap_\mu\colon\cont(\X)\to\cont(\X)$ the map such that, for any $w\in\cont(\X)$,
\begin{equation}\label{eq:rmap}
  \rmap_\mu(w)\colon x\mapsto -\eps\log \int e^{\frac{w(y) - \cost(x,y)}{\eps}}\,d\mu(y).
\end{equation}
The first order optimality conditions for \cref{eq:dual_pb} are (see \cite{feydy2018interpolating} or \cref{subsec:sinkiter})
\begin{equation}\label{eq:fixed-point-sinkhorn}
  u = \rmap_\beta(v) \quad \alpha \text{- a.e.} \qquad\text{and}\qquad v = \rmap_\alpha(u) \quad \beta \text{- a.e}.
\end{equation}
Pairs $(u,v)$ satisfying \cref{eq:fixed-point-sinkhorn} exist \cite{knopp1968note} and are referred to as {\em Sinkhorn potentials}. They are unique $(\alpha,\beta)$ - a.e. up to additive constant, i.e. $(u+t,v-t)$ is also a solution for any $t\in\R$. In line with \cite{genevay2018sample,feydy2018interpolating} it will be useful in the following to assume $(u,v)$ to be the Sinkhorn potentials such that: $i)$ $u(\anchor)=0$ for an arbitrary anchor point $\anchor\in\dom$  and $ii)$ \cref{eq:fixed-point-sinkhorn} is satisfied pointwise on the entire domain $\dom$. Then, $u$ is a fixed point of the map $\rmap_{\beta\alpha} = \rmap_\beta\circ\rmap_\alpha$ (analogously for $v$). This suggests a fixed point iteration approach to minimize \cref{eq:dual_pb}, yielding the well-known Sinkhorn-Knopp algorithm which has been shown to converge linearly in $\cont(\dom)$ \cite{sinkhorn1967,knopp1968note}
. We recall a key result characterizing the differentiability of $\oteps$ in terms of the Sinkhorn potentials that will be useful in the following.  

\begin{proposition}[Prop $2$ in \cite{feydy2018interpolating}]\label{prop:derivaties} 
Let $\nabla\oteps\colon\prob(\dom)^2\to\cont(\dom)^2$ be such that, 
$\forall\alpha,\beta \in \prm(\X)$
\begin{equation}
\nabla\oteps(\alpha,\beta) =(u,v), \qquad \text{with} \qquad u = \rmap_\beta(v),~~ v = \rmap_\alpha(u)~~\text{on } \dom, \quad u(\anchor) = 0.
\end{equation}
Then, $\forall\alpha,\alpha^\prime,\beta,\beta^\prime\in\prob(\dom)$, the directional derivative of $\oteps$ along $(\mu,\nu) = (\alpha^\prime-\alpha,\beta^\prime-\beta)$ is
\begin{equation}\label{eq:directional-derivative-oteps-intro}
\oteps^\prime(\alpha,\beta; \mu,\nu) = \scal{\nabla \oteps(\alpha,\beta)}{(\mu,\nu)} = \scal{u}{\mu} + \scal{v}{\nu},
\end{equation}
where $\scal{w}{\rho} = \int w(x)\,d\rho(x)$ denotes the canonical pairing between the spaces $\cont(\dom)$ and $\meas(\dom)$.
\end{proposition}
Note that $\nabla\oteps$ is not a gradient in the standard sense. In particular note that the directional derivative in \cref{eq:directional-derivative-oteps-intro} is not defined for any pair of signed measures, but only along {\em feasible directions} $(\alpha^\prime-\alpha,\beta^\prime-\beta)$.

\paragraph{Sinkhorn Divergence} The fast convergence of Sinkhorn-Knopp algorithm makes $\oteps$ (with $\eps>0$) preferable to $\OT_0$ from a computational perspective \cite{cuturi2013sinkhorn}. However, when $\eps>0$ the entropic regularization introduces a bias in the optimal transport problem, since in general $\oteps(\mu,\mu)\neq 0$. To compensate for this bias, \cite{genevay2018learning} introduced the Sinkhorn {\em divergence} 
\begin{equation}\label{eq:sink_divergence}
\sink\colon\prob(\dom)\times\prob(\dom)\to\R, \qquad 
(\alpha,\beta) \mapsto \oteps(\alpha,\beta) - \frac{1}{2}\oteps(\alpha,\alpha) -\frac{1}{2}\oteps(\beta,\beta),
\end{equation}
which was shown in \cite{feydy2018interpolating} to be  nonnegative, bi-convex and to metrize the convergence in law under mild assumptions. We characterize the gradient of $\sink(\cdot,\beta)$ for a fixed $\beta\in\prob(\X)$, which will be key to derive our optimization algorithm for computing Sinkhorn barycenters.

\begin{remark}\label{rem:gradient-sinkhorn-divergence}
Let $\nabla_1\oteps:\prob(\dom)^2\to\cont(\dom)$ be the first component of $\nabla\oteps$ (informally, the $u$ of the Sinkhorn potentials). As in \cref{prop:derivaties}, for any $\beta\in\prob(\dom)$ the gradient of $S_\eps(\cdot,\beta)$ is
%
\begin{equation}\label{eq:grad_sink}
\nabla [S_\eps(\cdot, \beta)]\colon\prob(\X)\to\cont(\X) \qquad \alpha \mapsto \nabla_1\oteps(\alpha,\beta) - \frac{1}{2}\nabla_1\oteps(\alpha,\alpha) = u-p,
\end{equation}
with $u=\rmap_{\beta\alpha}(u)$ and $p = \rmap_{\alpha}(p)$ the Sinkhorn potentials of $\oteps(\alpha,\beta)$ and $\oteps(\alpha,\alpha)$ respectively.
\end{remark}


\section{Sinkhorn barycenters with Frank-Wolfe}\label{sec:algorithm-theory}

Given $\beta_1,\dots\beta_m\in\prob(\dom)$ and $\omega_1,\dots,\omega_m\geq0$ a set of weights such that $\sum_{j=1}^m \omega_j = 1$, the main goal of this paper is to solve the following {\em Sinkhorn barycenter} problem
\begin{equation}\label{eq:sinkhorn-barycenter}
\min _{\alpha \in \prob(\dom)} \bary(\alpha), \qquad\textnormal{with}\qquad \bary(\alpha) = \sum_{j=1}^m ~\omega_j~\sink(\alpha, \beta_j).
\end{equation}
Although the objective functional $\bary$ is convex, its  domain $\prm(\dom)$ has \textit{empty} interior in the space of finite signed measure $\meas(\X)$. Hence standard notions of Fr\'echet or G\^ateaux differentiability do not apply. 
This, in principle causes some difficulties in devising optimization methods.
To circumvent this issue, in this work we adopt  
the Frank-Wolfe (\fw{}) algorithm. Indeed, 
one key advantage of this method is that it is formulated in terms of
directional derivatives along feasible directions 
(i.e., directions that locally remain inside the constraint set). Building upon \cite{dem1967,dem1968,dunn1978conditional},
which study the algorithm in Banach spaces, we show that
the ``weak'' notion of directional differentiability of $\sink$ (and hence of $\bary$) in  \cref{rem:gradient-sinkhorn-divergence} is sufficient
to carry out the convergence analysis.
While full details are provided in \cref{sec:frank-wolfe}, below we give an overview of the main result.


\paragraph{Frank-Wolfe in dual Banach spaces} Let $\VV$ be a real Banach space with topological dual $\BB$ and let $\CC\subset\BB$ be a nonempty,  convex, closed and bounded set. For any $\xx\in\BB$ denote by $\feas_\CC(\xx)=\R_+(\CC-\xx)$ the set of feasible direction of $\CC$ at $\xx$ (namely $s=t(\xx^\prime - \xx)$ with $\xx^\prime\in\CC$ and $t>0$). Let $\func\colon\CC\to\R$ be a convex function and assume that there exists a map $\nabla\func\colon\CC\to\VV$ (not necessarily unique) such that $\scal{\nabla \func(\xx)}{s} = \func^\prime(\xx;s)$ for every $s\in\feas_\CC(\xx)$. 
In \cref{alg:abstract-FW-intro} we present a method to minimize $\func$. The algorithm is structurally equivalent to the standard \fw{} \cite{dunn1978conditional,jaggi2013revisiting}
and accounts for possible inaccuracies in solving the minimization
in step $(i)$. This will be key in \cref{sec:algorithm-convergence}
when studying the barycenter problem for $\beta_j$ with infinite
support. The following result (see proof in \cref{sec:frank-wolfe}) shows that under the additional
assumption that $\nabla\func$ is Lipschitz-continuous and with sufficiently fast decay of the errors, the above
procedure  converges in value to the minimum of $\func$
with rate $O(1/k)$. Here $\Diam(\CC)$ denotes the diameter of $\CC$
with respect to the dual norm. 
\begin{theorem}\label{thm:fw-informal}
Under the assumptions above, suppose in addition that $\nabla \func$ is $L$-Lipschitz continuous with $L>0$. Let $(\xx_k)_{k \in \N}$ be obtained according to \cref{alg:abstract-FW-intro}. Then, for every integer $k\geq1$,
\begin{equation}
\label{eq:20190418a}
\func(\xx_k) - \min \func \leq \frac{2}{k+2} L\,\Diam(\CC)^2 + \precision_k.
\end{equation}
\end{theorem}

\begin{algorithm}[t]
\caption{{\sc Frank-Wolfe in Dual Banach Spaces}}
\label{alg:abstract-FW-intro}
\begin{algorithmic}
\vspace{0.25em}
  \STATE {\bfseries Input:} initial $\xx_0\in\CC$, precision 
  $(\precision_k)_{k \in \N} \in \R_{++}^\N$, such that $\precision_k(k+2)$ is nondecreasing.
  \vspace{0.45em}
  \STATE {\bfseries For} $k=0,1,\dots$
  \STATE \qquad Take $\zz_{k+1}$ such that ${\func^\prime(\xx_k, \zz_{k+1} - \xx_k) \leq \min_{\zz \in \CC} \func^\prime(\xx_k, \zz - \xx_k) + \frac{\precision_k}{2}}$
  \STATE \qquad ${\xx_{k+1} = \xx_k + \frac{2}{k+2}(\zz_{k+1} - \xx_k)}$

\end{algorithmic}
\end{algorithm}



\paragraph{Frank-Wolfe Sinkhorn barycenters} We show that the barycenter problem \cref{eq:sinkhorn-barycenter} satisfies the setting and hypotheses of \cref{thm:fw-informal} and can be thus approached via \cref{alg:abstract-FW-intro}.

\emparagraph{Optimization domain} Let $\VV = \cont(\dom)$, with dual $\BB=\meas(\X)$. The constraint set $\CC = \prob(\dom)$ is convex, closed, and bounded.

\emparagraph{Objective functional} The objective functional $\func = \bary\colon\prob(\dom)\to\R$, defined in \cref{eq:sinkhorn-barycenter}, is convex since it is a convex combination of $\sink(\cdot,\beta_j)$, with 
$j= 1 \dots m$. The gradient $\nabla\bary\colon\prob(\dom)\to\cont(\dom)$ is  $\nabla\bary = \sum_{j=1}^m~\omega_j~ \nabla \sink(\cdot,\beta_j)$, where $\nabla \sink(\cdot,\beta_j)$ is given in \cref{rem:gradient-sinkhorn-divergence}.

\emparagraph{Lipschitz continuity of the gradient.} This is the most critical condition and is addressed in  the following theorem. 

\begin{restatable}{theorem}{TLipschitzContinuityTV}\label{thm:lip-continuity-total-variation-informal}
The gradient $\nabla\oteps$ 
defined in \cref{prop:derivaties}
is Lipschitz continuous.  In particular, the first component $\nabla_1\oteps$ is $2\eps e^{3\diameps}$-Lipschitz continuous, i.e., for every $\alpha,\alpha^\prime,\beta,\beta^\prime\in\prob(\dom)$,
\begin{equation}
  \supnor{u - u^\prime} = \supnor{\nabla_1\oteps(\alpha,\beta)-\nabla_1\oteps(\alpha^\prime,\beta^\prime)} \leq  2\eps e^{3\diameps}~(\nor{\alpha - \alpha^\prime}_{TV} + \nor{\beta-\beta^\prime}_{TV}), 
\end{equation}
where $\diam = \sup_{x,y\in\X}~\dist(x,y)$, $u = \rmap_{\beta\alpha}(u), u^\prime = \rmap_{\beta^\prime,\alpha^\prime}(u^\prime)$, and $u(\anchor)=u^\prime(\anchor)=0$. Moreover, it follows from   \cref{eq:grad_sink} that $\nabla \sink(\cdot,\beta)$ is $6\eps e^{3\diameps}$-Lipschitz continuous. The same holds for $\nabla \bary$.
\end{restatable}
%
%
%
\cref{thm:lip-continuity-total-variation-informal} is one of the main contributions of this paper. It can be rephrased by saying that the operator that maps 
a pair of distributions to their Sinkhorn potentials is Lipschitz continuous. This result is significantly deeper than the one given in \cite[Lemma 1]{decentralized2018}, which establishes the Lipschitz continuity of the gradient in the \textit{semidiscrete} case. The proof (given in \cref{sec:app-frank-wolfe-algorithm}) relies on non-trivial tools from Perron-Frobenius theory for Hilbert's metric \cite{lemmens2012nonlinear}, which is a well-established framework to study Sinkhorn potentials \cite{peyre2017computational}. We believe this result is interesting not only for the application of \fw{} to the Sinkhorn barycenter problem, 
but also for further understanding regularity properties of entropic optimal transport.

\section{Algorithm: practical Sinkhorn barycenters}\label{sec:algorithm-practice}

According to \cref{sec:algorithm-theory}, \fw{} is a valid approach to tackle the barycenter problem  \cref{eq:sinkhorn-barycenter}. Here we describe how to implement in practice the abstract procedure of \cref{alg:abstract-FW-intro} to obtain a sequence of distributions $(\alpha_k)_{k\in\N}$ minimizing $\bary$.
%
%
A main challenge in this sense resides in finding a minimizing feasible direction for $\bary^\prime(\alpha_k;\mu-\alpha_k) = \scal{\nabla\bary(\alpha_k)}{\mu-\alpha_k}$. According to \cref{rem:gradient-sinkhorn-divergence}, this amounts to solve
\begin{equation}\label{eq:inner-fw}
\mu_{k+1} \in \argmin_{\mu\in\prob(\dom)} ~\sum_{j=1}^m ~\omega_j~ \scal{u_{jk}-p_{k}}{\mu} \qquad\text{where}\qquad u_{jk}-p_{k} = \nabla\sink[(\cdot,\beta_j)](\alpha_k),
\end{equation}
with $p_k = \nabla_1\oteps(\alpha_k,\alpha_k)$ not depending on $j$. In general \cref{eq:inner-fw} would entail a minimization over the set of all probability distributions on $\dom$. However, since the objective functional is linear in $\mu$ and $\prob(\X)$ is a weakly-$*$ compact convex set, we can apply Bauer maximum principle (see e.g., \cite[Thm. 7.69]{aliprantis2006}). Hence, solutions are achieved at the extreme points of the optimization domain, namely Dirac's deltas for the case of $\prob(\X)$ \cite[p. 108]{choquet1969}. Now, denote by $\delta_x\in\prob(\dom)$ the Dirac's delta centered at $x\in\dom$. We have $\scal{w}{\delta_x} = w(x)$ for every $w\in\cont(\dom)$. Hence \cref{eq:inner-fw} is equivalent to
\begin{equation}\label{eq:inner-fw-pointwise}
\mu_{k+1} = \delta_{x_{k+1}} \qquad \text{with} \qquad x_{k+1} \in \argmin_{x\in\dom}~ \sum_{j=1}^m ~\omega_j~ \big(u_{jk}(x)-p_{k}(x)\big).
\end{equation}
Once the new support point $x_{k+1}$ has been obtained, the update in \cref{alg:abstract-FW-intro} corresponds to
\begin{equation}\label{eq:fw-bary-update}
\alpha_{k+1} = \alpha_k + \frac{2}{k+2} (\delta_{x_{k+1}} -\alpha_k) = \frac{k}{k+2} \alpha_k + \frac{2}{k+2} \delta_{x_{k+1}}.
\end{equation}
In particular, if \fw{} is initialized with a distribition with finite support, say $\alpha_0 = \delta_{x_0}$ for some $x_0\in\dom$, then also every further iterate $\alpha_k$ will have at most $k+1$ support points.  
According to \cref{eq:inner-fw-pointwise}, the inner optimization for \fw{} consists in minimizing the functional $x\mapsto\sum_{j=1}^m ~\omega_j~ \big(u_{jk}(x)-p_{k}(x)\big)$ over $\dom$. In practice, having access to such functional poses already a challenge, since it requires computing the Sinkhorn potentials $u_{jk}$ and $p_{k}$, which are infinite dimensional objects. Below we discuss how to estimate these potentials when the $\beta_j$ have finite support. We then address the general setting. 

\paragraph{Computing $\nabla_1\oteps$ for probability distributions with finite support}  
Let $\alpha,\beta\in\prob(\dom)$, with $\beta = \sum_{i=1}^{n} b_i \delta_{y_i}$ a probability measure with finite support, with $\msf{b} = (b_i)_{i=1}^{n}$ nonnegative weights summing up to $1$. It is useful to identify $\beta$ with the pair $(\mbf{Y},\msf{b})$, where $\mbf{Y}\in\R^{d \times n}$ is the matrix with $i$-th column equal to $y_i$. Let now $(u,v)\in\cont(\dom)^2$ be the pair of Sinkhorn potentials associated to $\alpha$ and $\beta$ in \cref{prop:derivaties}, recall that $u = \rmap_\beta(v)$. Denote by $\msf v\in\R^n$ the {\em evaluation vector} of the Sinkhorn potential $v$, with $i$-th entry $\msf{v}_i = v(y_i)$. According to the definition of $\rmap_\beta$ in \cref{eq:rmap}, for any $x\in\dom$
\begin{equation}\label{eq:sinkhorn-gradient-routine}
[\nabla_1\oteps(\alpha,\beta)](x) = u(x) = [\rmap_{\beta}(v)](x) = -\eps\log \sum_{i=1}^n ~e^{(\msf{v}_i - \cost(x,y_i))/\eps}~b_i,
\end{equation}
since the integral $\rmap_\beta(v)$ reduces to a sum 
over the support of $\beta$. Hence, the gradient of $\oteps$ (i.e. the potential $u$), {\em is uniquely characterized in terms of the finite dimensional vector $\msf{v}$ collecting the values 
of the potential $v$ on the support of $\beta$
}. We refer as {\sc SinkhornGradient} to the routine which associates to each triplet $(\mbf{Y},\mbf{b},\mbf{v})$ the map 
 $x\mapsto-\eps~\log \sum_{i=1}^n ~e^{(\msf{v}_i - \cost(x,y_i))/\eps}~b_i$.

\begin{algorithm}[t]
\caption{{\sc Sinkhorn Barycenter}}\label{alg:practical-FW}
\begin{algorithmic}
  \vspace{0.25em}
  \STATE {\bfseries Input:}  $\beta_j = (\mbf{Y}_j,\msf{b}_j)$ with $\mbf{Y}_j\in\R^{d \times n_j}, \msf{b}_j\in\R^{n_j},\omega_j>0$ for $j=1,\dots,m$, $x_0 \in\R^d$, $\eps>0$, $K\in\N$.
  
  \vspace{0.45em}
  \STATE {\bfseries Initialize:} $\alpha_0 = (\mbf{X}_0,\msf{a}_0)$ with $\mbf{X}_0=x_0$, $\msf{a}_0 = 1$. 

  \vspace{0.45em}
  \STATE {\bfseries For} $k=0,1,\dots,K-1$
    \vspace{0.2em}
    \STATE \qquad $\msf{p} =$ {\sc SinkhornKnopp}$(\alpha_k,\alpha_k,\eps)$
    \STATE \qquad $p(\cdot) =$ {\sc SinkhornGradient}$(\mbf{X}_k,\msf{a}_k,\msf{p})$
    \vspace{0.25em}
    \STATE \qquad {\bfseries For} $j=1,\dots m$
    \STATE \qquad\qquad $\msf{v}_j = $ {\sc SinkhornKnopp}$(\alpha_k,\beta_j,\eps)$
    \STATE \qquad\qquad $u_j(\cdot) =$ {\sc SinkhornGradient}$(\mbf{Y}_j,\msf{b}_j,\msf{v}_j)$
    \STATE \qquad {\bfseries Let} $\innerfunc\colon x 
    \mapsto \sum_{j=1}^m\omega_j ~u_j(x) - p(x)$
    \STATE \qquad $x_{k+1} = $ {\sc Minimize}$(\innerfunc)$
    \vspace{0.25em}
    \STATE \qquad $\mbf{X}_{k+1} = [\mbf{X}_k,x_{k+1}]$ ~and~ $\msf{a}_{k+1} = \frac{1}{k+2}\left[k~\msf{a}_k,2\right]$
    \STATE \qquad $\alpha_{k+1} = (\mbf{X}_{k+1},\msf{a}_{k+1})$       

  \vspace{0.45em}
    \STATE {\bfseries Return:} $\alpha_K$


\end{algorithmic}
\end{algorithm}

  

      

\paragraph{Sinkhorn barycenters: finite case}  
\cref{alg:practical-FW} summarizes \fw{} applied to the barycenter problem \cref{eq:sinkhorn-barycenter} when the $\beta_j$'s have finite support. Starting from a Dirac's delta $\alpha_0 = \delta_{x_0}$, at each iteration $k\in\N$ the algorithm proceeds by: $i)$ finding the corresponding evaluation vectors $\msf{v}_j$'s and $\msf{p}$ of the Sinkhorn potentials for $\oteps(\alpha_k,\beta_j)$ and $\oteps(\alpha_k,\alpha_k)$ respectively, via the routine {\sc SinkhornKnopp} (see \cite{cuturi2013sinkhorn,feydy2018interpolating} or \cref{algo:sinkalgo_disc}). This is possible since both $\beta_j$ and $\alpha_k$ have finite support and therefore the problem of approximating the evaluation vectors $\msf{v}_j$ and $\msf{p}$ reduces to an optimization problem over finite vector spaces that can be efficiently solved \cite{cuturi2013sinkhorn}; $ii)$ obtain the gradients $u_{j} = \nabla_1\oteps(\alpha_k,\beta_j)$ and $p = \nabla_1\oteps(\alpha_k,\alpha_k)$ via {\sc SinkhornGradient}; $iii)$ minimize $\innerfunc:x\mapsto\sum_{j=1}^n \omega_j ~ u_j(x) - p(x)$ over $\dom$ to find a new point $x_{k+1}$ (we comment on this meta-routine {\sc Minimize} below); $iv)$ finally update the support and weights of $\alpha_k$ according to \cref{eq:fw-bary-update} to obtain the new iterate $\alpha_{k+1}$. 

A key feature of \cref{alg:practical-FW} is that the support of the candidate barycenter is updated {\em incrementally} by adding at most one point at each iteration, a procedure similar in flavor to the kernel herding strategy in \cite{bach2012equivalence,lacoste2015sequential}. This contrasts with previous methods for barycenter estimation \cite{cuturi14,BenamouCCNP15,staib2017parallel,decentralized2018}, which require the support set, or at least its cardinality, to be fixed beforehand. However, indentifying the new support point requires solving the nonconvex problem \cref{eq:inner-fw-pointwise}, a task addressed by the meta-routine {\sc Minimize}. This problem is typically smooth (e.g., a linear combination of Gaussians when $\cost(x,y) = \nor{x-y}^2$) and first or second order nonlinear optimization methods can be adopted to find stationary points. We note that all free-support methods in the literature for barycenter estimation are also affected by nonconvexity since they typically require solving a bi-convex problem (alternating minimization between support points and weights) which is not jointly convex \cite{cuturi14,stochwassbary}. We conclude by observing that if we restrict to the setting of \cite{staib2017parallel,decentralized2018} with fixed support set, then {\sc Minimize} can be solved exactly by evaluating the functional in \cref{eq:inner-fw-pointwise} on each candidate support point.

\paragraph{Sinkhorn barycenters: general case} When the $\beta_j$'s have infinite support, it is not possible to apply Sinkhorn-Knopp in practice. In line with \cite{genevay2018sample, staib2017parallel}, we can randomly sample empirical distributions $\hat\beta_j = \frac{1}{n}\sum_{i=1}^n \delta_{x_{ij}} $ from each $\beta_j$ and apply Sinkhorn-Knopp to $(\alpha_k,\hat\beta_j)$ in \cref{alg:abstract-FW-intro} rather than to the ideal pair $(\alpha_k,\beta_j)$. This strategy is motivated by \cite[Prop 13]{feydy2018interpolating}, where it was shown that Sinkhorn potentials vary continuously with the input measures. However, it opens two questions: $i)$ whether this approach is theoretically justified (consistency) and $ii)$ how many points should we sample from each $\beta_j$ to ensure convergence (rates). We answer these questions in \cref{thm:sinkhorn-barycenters-infinite-case} in the next section. 

\section{Convergence analysis}\label{sec:algorithm-convergence}

We finally address the convergence of \fw{} applied to both the finite and infinite settings discussed in \cref{sec:algorithm-practice}. We begin by considering the finite setting. 

\begin{restatable}{theorem}{TLSinkhornBarycenterFiniteCase}
\label{thm:sinkhorn-barycenters-finite-case}
Suppose that $\beta_1,\dots\beta_m\in\prob(\dom)$ have finite support and let $\alpha_k$ be the $k$-th iterate of \cref{alg:practical-FW} applied to \cref{eq:sinkhorn-barycenter}. Then,
\begin{equation}\label{eq:convergence_finite_case}
\bary(\alpha_k) - \min_{\alpha\in\prob(\dom)}\bary(\alpha) \leq \frac{48\,\eps\, e^{3\diameps}}{k+2}.
\end{equation}
\end{restatable}
The result follows by the convergence result of \fw{} in  \cref{thm:fw-informal} applied with the Lipschitz constant computed in \cref{thm:lip-continuity-total-variation-informal}, and recalling that $\Diam(\prob(\X))=2$ with respect to the Total Variation. We note that \cref{thm:sinkhorn-barycenters-finite-case} assumes {\sc SinkhornKnopp} and {\sc Minimize} in \cref{alg:practical-FW} to yield exact solutions. In \cref{sec:app-frank-wolfe-algorithm} we comment how approximation errors in this context affect the bound in \cref{eq:convergence_finite_case}. 

\paragraph{General setting} As mentioned in \cref{sec:algorithm-practice}, when the $\beta_j$'s are not finitely supported we adopt a sampling approach. More precisely we propose to {\em replace} in \cref{alg:practical-FW} the ideal Sinkhorn potentials of the pairs $(\alpha,\beta_j)$ with those of $(\alpha,\hat\beta_j)$, where each $\hat\beta_j$ is an empirical measure randomly sampled from $\beta_j$. In other words we are performing the \fw{} algorithm with a (possibly rough) approximation of the correct gradient of $\bary$. According to \cref{thm:fw-informal}, \fw{} allows  errors in the gradient estimation (which are captured into the precision $\Delta_k$ in the statement). To this end, the following result \textit{quantifies} the approximation error between $\nabla_1\oteps(\cdot,\beta)$ and $\nabla_1\oteps(\cdot,\hat\beta)$ in terms of the sample size of $\hat\beta$. 

\begin{restatable}[Sample Complexity of Sinkhorn Potentials]{theorem}{TLSampleComplexitySinkhornPotentials}
\label{thm:sample-complexity-sinkhorn-gradients}
Suppose that $\cost \in \cont^{s+1}(\X\times\X)$ with $s>d/2$.
Then,
there exists a constant $\overline\sinkconst=\overline\sinkconst(\dom,\cost,d)$ such that for any $\alpha,\beta\in\prob(\dom)$ and any empirical measure $\hat\beta$ of a set of $n$ points independently sampled from $\beta$, we have, for every $\tau\in(0,1]$ 
\begin{equation}\label{eq:uniform-approximation-empirical-continuous-sinkhorn-potentials}
    \supnor{u - u_n} = \lVert \nabla_1\oteps(\alpha,\beta)-\nabla_1\oteps(\alpha,\hat\beta) \rVert_{\infty}\leq \frac{8\varepsilon~\overline\sinkconst e^{3\diameps}\log\frac{3}{\tau}}{\sqrt{n}}
\end{equation}
with probability at least $1-\tau$, where $u= \rmap_{\beta\alpha}(u), u_n = \rmap_{\hat\beta\alpha}(u_n)$ and $u(\anchor) = u_n(\anchor) = 0$. 
\end{restatable}

\noindent\cref{thm:sample-complexity-sinkhorn-gradients} is one of the main results of this work. We point out that it {\em cannot} be obtained by means of the Lipschitz continuity of $\nabla_1\oteps$ in \cref{thm:lip-continuity-total-variation-informal}, since empirical measures do not converge in $\nor{\cdot}_{TV}$ to their target distribution \cite{devroye1990no}. Instead, the proof consists in considering the weaker {\em Maximum Mean Discrepancy (MMD)} metric associated to a universal kernel \cite{song2008learning}, which metrizes the topology of the convergence in law of $\prob(\dom)$ \cite{sriperumbudur2011universality}. Empirical measures converge in MMD metric to their target distribution \cite{song2008learning}. Therefore, by proving the Lipschitz continuity of $\nabla_1\oteps$ with \textit{respect to \mmd{}} in \cref{prop:lipschitz-continuity-mmd} we are able to conclude that \cref{eq:uniform-approximation-empirical-continuous-sinkhorn-potentials} holds. This latter result relies on higher regularity properties of Sinkhorn potentials, which have been recently shown \cite[Thm.2]{genevay2018sample} to be uniformly bounded in Sobolev spaces under the additional assumption $c\in\cont^{s+1}(\X\times\X)$. For sufficiently large $s$, the Sobolev norm is in duality with the MMD \cite{muandet2017kernel} and allows us to derive the required Lipschitz continuity. We conclude  noting that while \cite{genevay2018sample} studied the sample complexity of the Sinkhorn {\em divergence}, \cref{thm:sample-complexity-sinkhorn-gradients} is a sample complexity result for Sinkhorn {\em potentials}. In this sense, we observe that the constants appearing in the bound are tightly related to those in \cite[Thm.3]{genevay2018sample} and have similar behavior with respect to $\eps$. We can now study the convergence of \fw{} in continuous settings.






\begin{restatable}{theorem}{TLSInkhornBarycenterInfiniteDim}
\label{thm:sinkhorn-barycenters-infinite-case}
Suppose that $\cost \in \cont^{s+1}(\X\times\X)$ with $s>d/2$.
Let $n \in \N$ and
 $\hat\beta_1,\dots,\hat\beta_m$ be empirical distributions with $n$ support points, each independently sampled from $\beta_1,\dots,\beta_m$. 
Let $\alpha_k$ be the $k$-th iterate of \cref{alg:practical-FW} applied to $\hat\beta_1,\dots,\hat\beta_m$. Then for any $\tau\in(0,1]$, the following holds with probability larger than $1-\tau$ 
\begin{equation}
  \bary(\alpha_k) - \min_{\alpha\in\prob(\dom)} \bary(\alpha) \leq
  \frac{64 \bar\sinkconst \varepsilon e^{3\diameps} \log\frac{3}{\tau} }{\min(k,\sqrt{n})}.
\end{equation}
\end{restatable} 
The proof is shown in \cref{sec:sample-complexity-sinkhorn-potentials}.
A consequence of \cref{thm:sinkhorn-barycenters-infinite-case} is that the accuracy of \fw{} depends simultaneously on the number of iterations and the sample size used in the approximation of the gradients: by choosing $n = k^2$ we recover the $O(1/k)$ rate of the finite setting, while for $n=k$ we have a rate of $O(k^{-1/2})$, which is reminiscent of typical sample complexity results, highlighting the statistical nature of the problem.

\begin{remark}[Incremental Sampling]
The above strategy requires sampling the empirical distributions for $\beta_1,\dots,\beta_m$ beforehand. A natural question is whether it is be possible to do this {\em incrementally}, sampling new points and updating $\hat\beta_j$ accordingly, as the number of \fw{} iterations increase. To this end, one can perform an intersection bound and see that this strategy is still consistent, but the bound in \cref{thm:sinkhorn-barycenters-infinite-case} 
worsens the logarithmic term, which becomes $\log (3m k/\tau)$.
\end{remark}

\section{Experiments}
\label{sec:experiments}
In this section we show the performance of our method in a range of experiments \footnote{\url{https://github.com/GiulsLu/Sinkhorn-Barycenters}}.   

\paragraph{Discrete measures: barycenter of nested ellipses} We compute the barycenter of $30$ randomly generated nested ellipses on a $50\times50$ grid similarly to \cite{cuturi14}. We interpret each image as a probability distribution in $2$D. The cost matrix is given by the squared Euclidean distances between pixels. \cref{fig:ellipses} reports $8$ samples of the input ellipses and the barycenter obtained with \cref{alg:practical-FW}. It shows qualitatively that our approach captures key geometric properties of the input measures. 

\begin{figure}[t]
\begin{minipage}[t]{0.3\textwidth}
\centering
\includegraphics[height=4.05cm,trim={0 -0.05cm 0 0},clip]{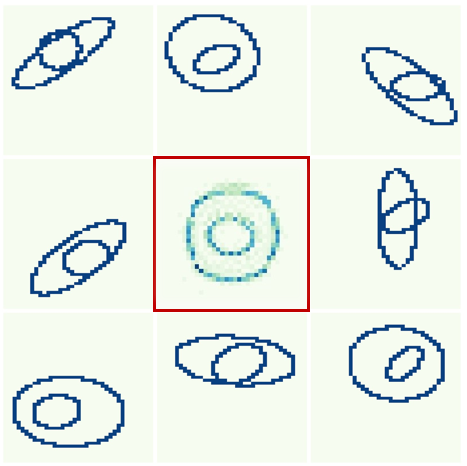} \caption{Ellipses}
\label{fig:ellipses}
\end{minipage}
\begin{minipage}[t]{0.7\textwidth}
\centering
  \includegraphics[height=4cm]{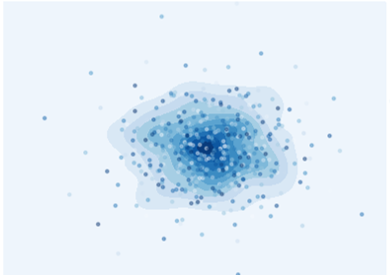}
  \includegraphics[height=4cm]{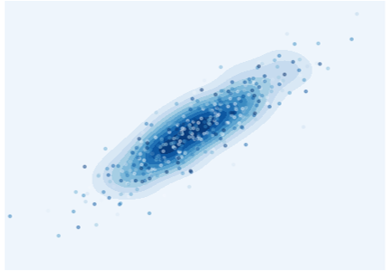}  
\caption{Barycenters of Gaussians}
\label{fig:gauss}
\end{minipage}
\end{figure}


\paragraph{Continuous measures: barycenter of Gaussians}
We compute the barycenter of $5$ Gaussian distributions $\mathcal{N}(m_i,C_i)$ $i=1,\dots,5$ in $\R^2$, with mean $m_i\in\R^2$ and covariance $C_i$ randomly generated. We apply \cref{alg:practical-FW} to empirical measures obtained by sampling $n=500$ points from each $\mathcal{N}(m_i,C_i)$, $i=1,\dots,5$. Since the (Wasserstein) barycenter of Gaussian distributions can be estimated accurately (see \cite{AguehC11}), in \cref{fig:gauss}
we report both the output of our method (as a scatter plot) and the true Wasserstein barycenter (as level sets of its density). We observe that our estimator recovers both the mean and covariance of the target barycenter. See the supplementary material for additional experiments also in the case of mixtures of Gaussians. 


\paragraph{Image ``compression'' via distribution matching} 
Similarly to \cite{stochwassbary}, we test \cref{alg:practical-FW} in the special case of computing the ``barycenter'' of a single measure $\beta\in\prm(\dom)$. While the solution of this problem is the distribution $\beta$ itself, we can interpret the intermediate iterates $\alpha_k$ of \cref{alg:practical-FW} as compressed version of the original measure. In this sense $k$ would represent the level of compression since $\alpha_k$ is supported on {\em at most} $k$ points. \cref{fig:cheeta} (Right) reports iteration $k=5000$ of \cref{alg:practical-FW} applied to the $140\times140$ image in \cref{fig:cheeta} (Left) interpreted as a probability measure $\beta$ in $2$D. We note that the number of points in the support is $\sim 3900$: indeed, \cref{alg:practical-FW} selects the most relevant support points points multiple times to accumulate the right amount of mass on each of them (darker color = higher weight). This shows that \fw{} tends to greedily search for the most relevant support points, prioritizing those with higher weight 

\begin{figure}[t]
\begin{minipage}{0.6\textwidth}
  \centering
  \includegraphics[height = 4.7cm]{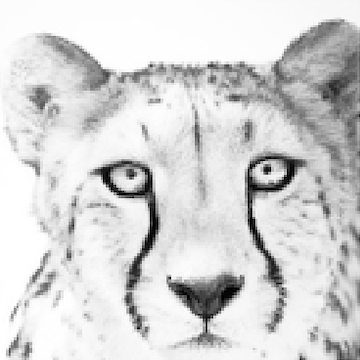}\quad
  \includegraphics[height = 4.7cm]{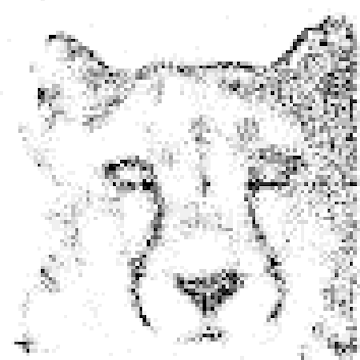}
  \caption{(left) original image 140x140 pixels, sample (right) }
  \label{fig:cheeta} 
\end{minipage}
\begin{minipage}{0.39\textwidth}
    \centering
 \includegraphics[height = 4.7cm,trim={0 0.5cm 0 0.5cm},clip]{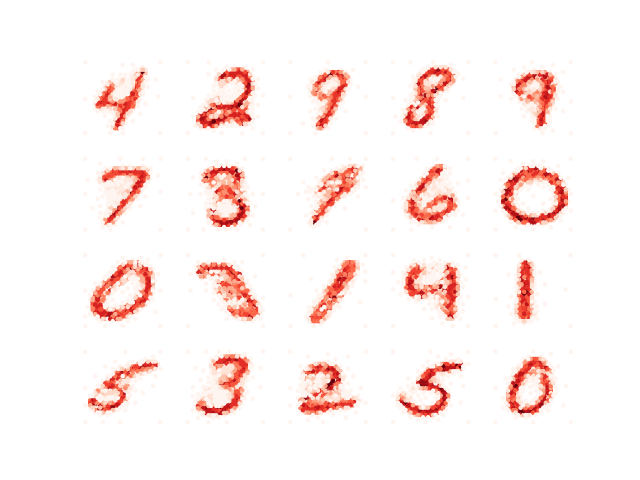}
 \caption{$k$-means}
 \label{fig:centroids}
\end{minipage}
\end{figure}

\paragraph{k-means on MNIST digits} We tested our algorithm on a $k$-means clustering experiment. We consider a subset of $500$ random images from the MNIST dataset. Each image is suitably normalized to be interpreted as a probability distribution on the grid of $28\times28$ pixels with values scaled between $0$ and $1$. We initialize $20$ centroids according to the  $k$-means++ strategy \cite{kmeans++}. \cref{fig:centroids} deipcts the $20$ centroids obtained by performing $k$-means with \cref{alg:practical-FW}. We see that the structure of the digits is successfully detected, recovering also minor details (e.g. note the difference between the $2$ centroids).

\paragraph{Real data: Sinkhorn propagation of weather data}
We consider the problem of Sinkhorn {\em propagation} similar to the one in \cite{Solomon:2014:WPS}. The goal is to predict the distribution of missing measurements for weather stations in the state of Texas, US by ``propagating'' measurements from neighboring stations in the network. The problem can be formulated as minimizing the functional $\sum_{(v,u)\in\mathcal{V}} \omega_{uv}\sink(\rho_v,\rho_u)$ over the set $\{\rho_v\in\prob(\R^2) | v\in\mathcal{V}_0\}$ with: $\mathcal{V}_0\subset\mathcal{V}$ the subset of stations with missing measurements, $G = (\mathcal{V},\mathcal{E})$ the whole graph of the stations network, $\omega_{uv}$ a weight inversely proportional to the 	geographical distance between two vertices/stations $u,v\in\mathcal{V}$. The variable $\rho_v\in\prob(\R^2)$ denotes the distribution of measurements at station $v$ of daily {\em temperature} and {\em atmospheric pressure} over one year. This is a generalization of the barycenter problem \cref{eq:sinkhorn-barycenter} (see also \cite{peyre2017computational}).  
From the total $|\mathcal{V}|=115$, we randomly select $10\%,20\%$ or $30\%$ to be {\em available} stations, and use \cref{alg:practical-FW} to propagate their measurements to the remaining ``missing'' ones.
We compare our approach (\fw{}) with the Dirichlet (DR) baseline in \cite{Solomon:2014:WPS} in terms of the error $d(C_T,\hat C)$ between the covariance matrix $C_T$ of the groundtruth distribution and that of the predicted one. Here $d(A,B) = \norm{\log(A^{-1/2} B A^{-1/2})}$ is the geodesic distance on the cone of positive definite matrices. The average prediction errors are: $2.07$ (\fw{}), $2.24$ (DR) for $10\%$, $1.47$ (\fw{}), $1.89$(DR) for $20\%$ and $1.3$ (\fw{}), $1.6$ (DR) for $30\%$. \cref{fig:propagation} qualitatively reports the improvement $\Delta = d(C_T,C_{DR}) - d(C_T,C_{FW})$ of our method on individual stations: a higher color intensity corresponds to a wider gap in our favor between prediction errors, from light green $(\Delta\sim 0)$ to red $(\Delta\sim 2)$. Our approach tends to propagate the distributions to missing locations with higher accuracy.

\begin{figure}[t]
  \centering
    \includegraphics[height = 4cm]{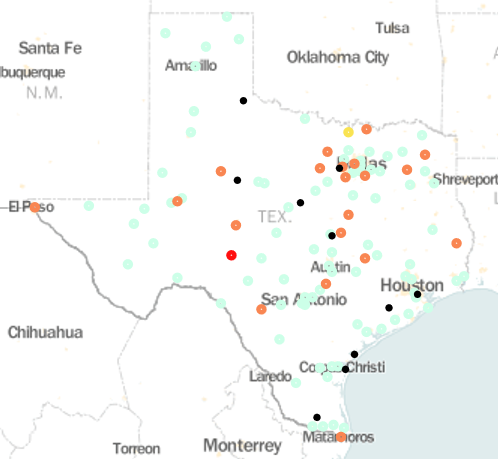}\qquad%
    \includegraphics[height = 4cm]{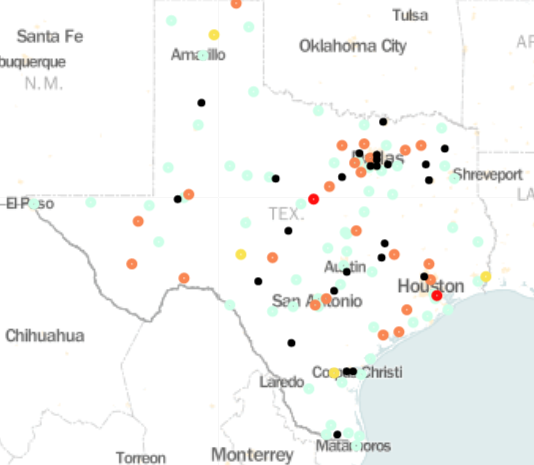}\qquad%
    \includegraphics[height = 4cm]{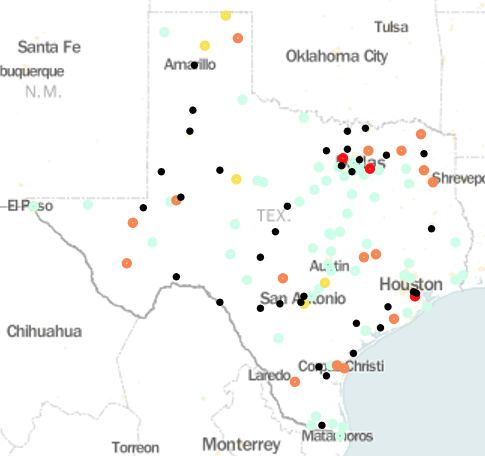}
    \caption{From Left to Right: propagation of weather data with $10\%,20\%$ and $30\%$ stations with available measurements (see text).  \label{fig:propagation}}
\end{figure}

\section{Conclusion}
\label{sec:conclusion}

We proposed a Frank-Wolfe-based algorithm to find the Sinkhorn barycenter of probability distributions with either finitely or infinitely many support points. Our algorithm belongs to the family of barycenter methods with free support since it adaptively identifies support points rather than fixing them a-priori. In the finite settings, we were able to guarantee convergence of the proposed algorithm by proving the Lipschitz continuity of gradient of the barycenter functional in the Total Variation sense. Then, by studying the sample complexity of Sinkhorn potential estimation, we proved the convergence of our algorithm also in the infinite case. We empirically assessed our method on a number of synthetic and real datasets, showing that it exhibits good qualitative and quantitative performance. While in this work we have considered \fw{} iterates that are a convex combination of Dirac's delta, models with higher regularity (e.g. mixture of Gaussians) might be more suited to approximate the barycenter of distributions with smooth density. Hence, future work will investigate how the perspective adopted in this work could be extended also to other barycenter estimators. 

{
\bibliographystyle{plain}
\bibliography{biblio}
}

\newpage

\appendix

\crefname{assumption}{Assumption}{Assumptions}
\crefname{equation}{}{}
\Crefname{equation}{Eq.}{Eqs.}
\crefname{figure}{Figure}{Figures}
\crefname{table}{Table}{Tables}
\crefname{section}{Section}{Sections}
\crefname{theorem}{Theorem}{Theorems}
\crefname{proposition}{Proposition}{Propositions}
\crefname{fact}{Fact}{Facts}
\crefname{lemma}{Lemma}{Lemmas}
\crefname{corollary}{Corollary}{Corollaries}
\crefname{example}{Example}{Examples}
\crefname{remark}{Remark}{Remarks}
\crefname{algorithm}{Algorithm}{Algorithms}
\crefname{enumi}{}{}

\crefname{appendix}{Appendix}{Appendices}

\numberwithin{equation}{section}
\numberwithin{lemma}{section}
\numberwithin{proposition}{section}
\numberwithin{theorem}{section}
\numberwithin{corollary}{section}
\numberwithin{definition}{section}
\numberwithin{algorithm}{section}
\numberwithin{fact}{section}
\numberwithin{remark}{section}

\section*{\Huge\textbf{Supplementary Material}}

Below we give an overview of the structure of the supplementary material and highlight the main novel results of this work.


\paragraph{ \cref{sec:frank-wolfe}: abstract Frank-Wolfe algorithm in dual Banach spaces} This section contains full details on Frank-Wolfe algorithm. The novelty stands in the relaxation of the differentiability assumptions.\\

\paragraph{\cref{sec:PFtheory}: DAD problems and convergence of Sinkhorn-Knopp algorithm} This section is a brief review of basic concepts from the nonlinear Perrom-Frobeius theory, DAD problems, and applications to the study of Sinkorn algorithm.\\

\paragraph{ \cref{subsec:lipschitz-total-variation}: Lipschitz continuitity of the gradient of the Sinkhorn divergence with respect to Total Variation} This section contains one of the main contributions of our work, \cref{prop:lipschitz-continuity-total-variation2}, from which we derive \cref{thm:lip-continuity-total-variation-informal} in the main text. \\

\paragraph{\cref{sec:app-frank-wolfe-algorithm}: Frank-Wolfe algorithm for Sinkhorn barycenters} This section contains the  complete analysis of FW algorithm for Sinkhorn barycenters, which takes into account the error in the computation of Sinkhorn potentials and the error in their minimization.  The main result is the convergence of the Frank-Wolfe scheme for finitely supported distributions in \cref{thm:full_convergence_FW_with_error}.\\

\paragraph{\cref{sec:sample-complexity-sinkhorn-potentials}: Sample complexity of Sinkhorn potential and convergence of \cref{alg:practical-FW} in case of continuous measures} This section contains the discussion and the proofs of two of main results of the work \cref{thm:sample-complexity-sinkhorn-gradients}, \cref{thm:sinkhorn-barycenters-infinite-case}.\\

\paragraph{ \cref{sec:additional_exp}: additional experiments} This section contains additional experiment on barycenters of mixture of Gaussian, barycenter of a mesh in 3D (dinosau) and additional figures on the experiment on Sinkhorn propagation described in \cref{sec:experiments}.

\section{The Frank-Wolfe algorithm in dual Banach spaces}\label{sec:frank-wolfe}

In this section we detail the convergence analysis of the Frank-Wolfe algorithm in abstract dual Banach spaces and under mild directional differentiablility assumptions so to cover the setting of Sinkhorn barycenters described in \cref{sec:algorithm-theory} of the paper.

Let $\VV$ be a real Banach space and let
 be $\BB$ its topological dual.
Let $\CC\subset \BB$ be a nonempty, closed, 
convex, and bounded set and let
 $\func\colon\CC \to \R$ be a convex function. We address the following optimization problem
\begin{equation}
\label{eq:minprob}
\min_{\xx \in \CC} \func(\xx),
\end{equation}
assuming that the set of solutions is nonemtpy.

We recall the concept of the tangent cone of feasible directions.
\begin{definition}\label{def:cone-of-feasible-directions}
Let $\xx\in \CC$. Then \emph{the cone of feasible directions of $\CC$ at $\xx$} is
$\feas_{\CC}(\xx) = \R_+ (\CC - \xx)$ and the \emph{tangent cone of $\CC$ at $\xx$}
is
\begin{equation*}
\mathcal{T}_{\CC}(\xx) = \overline{\feas_{\CC}(\xx)} = \big\{ \uu \in \BB \,\vert\, (\exists (t_{k})_{k \in \N} \in \R_{++}^{\N})(t_k \to 0) (\exists (\xx_{k})_{k \in \N} \in \CC^\N)\ t_k^{-1}(\xx_k - \xx) \to \uu \big\}.
\end{equation*}
\end{definition}

\begin{remark}
\normalfont
$\feas_{\CC}(\xx)$ is the cone generated by $\CC-\xx$, and it is a convex cone.
Indeed, if $t>0$ and $\uu \in \feas_{\CC}(\xx)$, then $t \uu \in \feas_{\CC}(\xx)$. Moreover, if $\uu_1, \uu_2 \in \feas_{\CC}(\xx)$, then there exists $t_1,t_2>0$ and $\xx_1,\xx_2 \in \CC$ such that $\uu_i = t_i(\xx_i - \xx)$,
$i=1,2$. Thus, 
\begin{equation*}
\uu_1 + \uu_2 = (t_1 + t_2)
\Big( \frac{t_1}{t_1 + t_2} \xx_1 
+ \frac{t_2}{t_1 + t_2} \xx_2 - \xx \Big) \in \R_+ (\CC - \xx).
\end{equation*}
So, $\mathcal{T}_{\CC}(\xx)$ is a closed convex cone too.
\end{remark}

\begin{definition}\label{def:directional-derivative}
Let $\xx \in \CC$ and $\uu \in \feas_{\CC}(\xx)$. Then, 
\emph{the directional
derivative of $\func$ at $\xx$ in the direction $\uu$} is
\begin{equation*}
\func^\prime(\xx;\uu) = \lim_{t \to 0^+} \frac{\func(\xx + t \uu) - \func(\xx)}{t}  \in \left[-\infty, +\infty\right[.
\end{equation*}
\end{definition}

\begin{remark}\label{rem:properties-of-directional-derivative}
\normalfont
The above definition is well-posed. Indeed, 
since $v$ is a feasible direction of $\CC$ at $x$, there exists $t_1>0$
and $\xx_1 \in \CC$ such that $\uu = t_1 (\xx_1-\xx)$; hence
\begin{equation*}
(\forall\, t\in \left]0,1/t_1\right])\quad
x + t \uu = x + t\,t_1(\xx_1-\xx) = (1 - t\,t_1) \xx+ t\,t_1 \xx_1 \in \CC.
\end{equation*}
Moreover, since $\func$ is convex, the function $t\in \left]0,1/t_1\right] \mapsto (\func(\xx + t \uu) - \func(\xx))/t$ is increasing,
hence
\begin{equation}
\label{eq:20190410a}
\lim_{t \to 0^+} \frac{\func(\xx + t \uu) - \func(\xx)}{t}= 
\inf_{t \in \left]0,1/t_1\right]} \frac{\func(\xx + t \uu) - \func(\xx)}{t}.
\end{equation}
\end{remark}

It is easy to prove that the function
\begin{equation*}
\uu \in \feas_{\CC}(\xx) \mapsto \func^\prime(\xx;\uu) \in \left[-\infty,+\infty\right[
\end{equation*}
is positively homogeneous and sublinear (hence convex), that is,
\begin{enumerate}[{\rm (i)}]
\item $(\forall\, \uu \in \feas_{\CC}(\xx))(\forall\,t \in \R_+)$\ 
$\func^\prime(\xx;t \uu) = t \func^\prime(\xx;\uu)$;
\item $(\forall\,\uu_1,\uu_2 \in \feas_{\CC}(\xx))$\ 
$\func^\prime(\xx;\uu_1 + \uu_2) \leq \func^\prime(\xx;\uu_1) + \func^\prime(\xx;\uu_2)$.
\end{enumerate}

We make the following assumptions about $\func$:
\begin{enumerate}[{\rm H$1$}]
\item\label{H1} $(\forall\, \xx \in \CC)$\ the function  
$\uu\mapsto \func^\prime(\xx;\uu)$ is finite, that is, $\func^\prime(\xx;\uu) \in \R$.
\item\label{H2} The \emph{curvature of $\func$} is finite, that is,
\begin{equation}\label{eq:curvature}
C_{\func} = \sup_{\substack{\xx,\zz \in \CC\\ \gamma \in [0,1]}}
\frac{2}{\gamma^2}\big( \func(\xx + \gamma(\zz-\xx)) - \func(\xx) - \gamma \func^\prime(\xx, \zz - \xx) \big)<+\infty.
\end{equation}
\end{enumerate}

\begin{remark}
\normalfont
For every $\xx,\zz \in \CC$, we have
\begin{equation}
\label{eq:sort-of-convexity-directional-derivative}
\func(\zz) - \func(\xx) \geq \func^\prime(\xx,\zz-\xx).
\end{equation}
This follows from \eqref{eq:20190410a} with $\xx_1=\zz$ and $t = 1$ ($t_1 =1$).
\end{remark}

The (inexact) Frank-Wolfe algorithm is detailed in \cref{alg:abstract-FW}.

\begin{algorithm}
\caption{Frank-Wolfe in Dual Banach Spaces}
\label{alg:abstract-FW}
Let $(\gamma_k)_{k \in \N} \in \R_{++}^\N$ be such that $\gamma_0= 1$ and, for every $k \in \N$, $1/\gamma_k \leq 1/\gamma_{k+1} \leq 1/2 + 1/\gamma_{k}$ (i.e., $\gamma_k = 2/(k+2)$).
Let $\xx_0 \in \CC$ and $(\Delta_k)_{k \in \N} \in \R_{+}^\N$ be such that $(\Delta_k/\gamma_k)_{k \in \N}$ is nondecreasing.
Then
\begin{equation*}
\begin{array}{l}
\text{for}\;k=0,1,\dots\\[0.7ex]
\left\lfloor
\begin{array}{l}
\text{find }\zz_{k+1}\in \CC \text{ is such that } \func^\prime(\xx_k; \zz_{k+1} - \xx_k) 
\leq \inf_{\zz \in \CC} \func^\prime(\xx_k; \zz - \xx_k) +   \frac 1 2 \Delta_k\\[1ex]
\xx_{k+1} = \xx_k + \gamma_k(\zz_{k+1} - \xx_k)
\end{array}
\right.
\end{array}
\end{equation*}
\end{algorithm}

\begin{remark}\ 
\normalfont
\begin{enumerate}[{\rm (i)}]
\item 
\cref{alg:abstract-FW} does not 
require the sub-problem $\min_{\zz \in \CC} \func^\prime(\xx_k, \zz - \xx_k)$ to have solutions. Indeed it only requires
computing a $\precision_k$-minimizer of $\func^\prime(\xx_k;\cdot - \xx_k)$ on $\CC$, which always exists.
\item 
Since $\CC$ is weakly-$*$ compact (by Banach-Alaoglu theorem),
if $\func^\prime(\xx_k,\cdot - \xx_k)$ is weakly-$*$ continuous on $\CC$, then the sub-problem $\min_{\zz \in \CC} \func^\prime(\xx_k, \zz - \xx_k)$ admits solutions.
Note that this occurs when the directional derivative
$\func^\prime(\xx;\cdot)$ is linear and can be represented in $\VV$.
This case is addressed in the subsequent \cref{p:inexactgrad}.
\end{enumerate}
\end{remark}

\vspace{0.5ex}
\begin{theorem}
\label{thm:FWA}
Let $(\xx_k)_{k \in \N}$ be defined according to \cref{alg:abstract-FW}.
Then, for every integer $k\geq 1$,
\begin{equation}
\label{eq:20190418a}
\func(\xx_k) - \min \func \leq C_{\func} \gamma_k + \Delta_k.
\end{equation}
\end{theorem}
\begin{proof}
Let $\xx_* \in \CC$ be a solution of problem \eqref{eq:minprob}.
It follows from \cref{H2} and the definition of $\xx_{k+1}$ in \cref{alg:abstract-FW}, 
that 
\begin{equation*}
\func(\xx_{k+1}) \leq \func(\xx_k) + \gamma_k \func^\prime(\xx_k;\zz_{k+1} - \xx_k) + \frac{\gamma_k^2}{2} C_{\func}.
\end{equation*}
Moreover, it follows from the definition of $\zz_{k+1}$ in \cref{alg:abstract-FW} 
and \eqref{eq:sort-of-convexity-directional-derivative} that
\begin{align*}
\func^\prime(\xx_k; \zz_{k+1} - \xx_k) 
&\leq \inf_{\zz \in \CC} \func^\prime(\xx_k; \zz - \xx_k) 
+ \frac 1 2 \Delta_k\\
 &\leq \func^\prime(\xx_k; \xx_* - \xx_k) + \frac1 2 \Delta_k\\
&\leq -( \func(\xx_k) - \func(\xx_*))+ \frac1 2 \Delta_k.
\end{align*}
Then,
\begin{equation}
\label{eq:20190418b}
\func(\xx_{k+1}) - \func(\xx_*) \leq (1 - \gamma_k) (\func(\xx_k) - \func(\xx_*))
+\frac{\gamma_k^2}{2} \Big(C_{\func}+  \frac{\Delta_k}{\gamma_k}\Big).
\end{equation}
Now, similarly to \cite[Theorem~2]{Jaggi2013}, we can prove \eqref{eq:20190418a} by induction.
Since $\gamma_0 = 1$, $1/\gamma_1 \leq 1/2 + 1/\gamma_0$,
and $\Delta_0/\gamma_0 \leq \Delta_1/ \gamma_1$, 
it follows from \eqref{eq:20190418b} that
\begin{equation}
    \func(\xx_{1}) - \func(\xx_*) \leq \frac{1}{2} 
    \Big(C_{\func} +\frac{\Delta_0}{\gamma_0}\Big)
    \leq \gamma_1 \Big(C_{\func} +\frac{\Delta_1}{\gamma_1}\Big),
\end{equation}
hence \eqref{eq:20190418a} is true for $k=1$.
Set, for the sake of brevity, $C_k = C_{\func} +\Delta_k/\gamma_k$ and
suppose that \eqref{eq:20190418a} holds for $k \in \N$, $k\geq 1$. Then, it follows from \eqref{eq:20190418b} and the properties of $(\gamma_k)_{k \in \N}$ that
\begin{align*}
\func(\xx_{k+1}) - \func(\xx_*) &\leq (1 - \gamma_k) \gamma_k C_k
+\frac{\gamma_k^2}{2}C_k\\
& = C_k\gamma_k \Big( 1 - \frac{\gamma_k}{2}\Big)\\
& \leq C_k\gamma_k \Big( 1 - \frac{\gamma_{k+1}}{2}\Big) \\
&\leq C_k \dfrac{1}{1/\gamma_{k+1} - 1/2}\Big( 1 - \frac{\gamma_{k+1}}{2}\Big)\\[0.8ex]
&= C_k \gamma_{k+1}\\[0.8ex]
&\leq  C_{k+1} \gamma_{k+1}.
\qedhere
\end{align*}
\end{proof}

\begin{corollary}
Under the assumptions of \cref{thm:FWA}, suppose in addition that $\precision_k = \precision \gamma_k^{\zeta}$, for some $\zeta \in [0,1]$ and $\precision \geq 0$. Then we have
\begin{equation}
\func(\xx_k) - \min \func \leq C_{\func} \gamma_k + \precision \gamma_k^{\zeta}.
\end{equation}
\end{corollary}
\begin{proof}
It follows from \cref{thm:FWA} by noting that the sequence $\precision_k/\gamma_k = 1/\gamma_k^{1 - \zeta}$ is nondecreasing.
\end{proof}

\begin{proposition}
\label{p:inexactgrad}
Suppose that there exists a mapping $\nabla \func\colon \CC \to \VV$
such that\footnote{This mapping does not need to be unique.}, 
\begin{equation}
\label{eq:20190410e}
(\forall\, \xx \in \CC)(\forall\,\zz \in \CC)\quad \scal{\nabla \func(\xx)}{\zz-\xx} = \func^\prime(\xx;\zz-\xx).
\end{equation}
Then the following holds.
\begin{enumerate}[{\rm (i)}]
    \item\label{p:inexactgrad_i} Let $k \in \N$
and suppose that there exists $u_k \in \VV$ such that
$\nor{u_k - \nabla \func(\xx_k)} \leq \Delta_{1,k}/4$
and that $\zz_{k+1} \in \CC$ satisfies 
\begin{equation*}
\scal{u_k}{\zz_{k+1}} \leq \min_{\zz \in \CC} \scal{u_k}{\zz} 
+ \frac{\Delta_{2,k}}{2},
\end{equation*}
for some $\Delta_{1,k},\Delta_{2,k}>0$. Then
\begin{equation}
\label{eq:inexactgrad}
\func^\prime(\xx_k; \zz_{k+1} - \xx_k) 
\leq \min_{\zz \in \CC} \func^\prime(\xx_k; \zz - \xx_k) +
\frac1 2( \Delta_{1,k} \Diam(\CC)+ \Delta_{2,k}).
\end{equation}
\item\label{p:inexactgrad_ii} 
Suppose that $\nabla \func\colon \CC \to \VV$ is $L$-Lipschitz continuous for some $L>0$. Then, for every $\xx,\zz \in \CC$
and $\gamma \in [0,1]$,
\begin{equation*}
\func(\xx + \gamma(\zz-\xx)) - \func(\xx) - \gamma\scal{\zz - \xx}
{\nabla \func(\xx)} \leq \frac{L}{2} 
\gamma^2\nor{\zz-\xx}^2
\end{equation*}
and hence $C_{\func} \leq L \mathrm{diam}(\CC)^2$.
\end{enumerate}
\end{proposition}
\begin{proof}
\cref{p:inexactgrad_i}:
We have
\begin{align}
\label{eq:20190410c}
\nonumber\scal{\nabla \func(\xx_k)}{\zz_{k+1} - \xx_k} 
&= \scal{u_k}{\zz_{k+1} - \xx_k} + \scal{\nabla \func(\xx_k) - u_k}{\zz_{k+1} - \xx_k}\\[1ex]
& \leq \min_{\zz \in \CC} \scal{u_k}{\zz - \xx_k} 
+ \frac{\Delta_{2,k}}{2} + 
\frac{\Delta_{1,k}}{4}\mathrm{diam}(\CC).
\end{align}
Moreover, 
\begin{align*}
(\forall\, \zz \in \CC)\quad \scal{u_k}{\zz - \xx_k} 
&= \scal{\nabla \func(\xx_k)}{\zz - \xx_k}
+ \scal{u_k - \nabla \func(\xx_k)}{\zz - \xx_k}\\[1ex]
&\leq  \scal{\nabla \func(\xx_k)}{\zz - \xx_k}
+ \frac{\Delta_{1,k}}{4} \Diam(\CC),
\end{align*}
hence 
\begin{equation}
\label{eq:20190410d}
\min_{\zz \in \CC} \scal{u_k}{\zz - \xx_k} \leq \min_{\zz \in \CC} \scal{\nabla \func(\xx_k)}{\zz - \xx_k}
+ \frac{\Delta_{1,k}}{2} \Diam(\CC).
\end{equation}
Thus, \eqref{eq:inexactgrad} follows from \eqref{eq:20190410c}, \eqref{eq:20190410d}, and \eqref{eq:20190410e}.

\cref{p:inexactgrad_ii}:
Let $\xx,\zz \in \CC$, and define
 $\psi\colon[0,1]\to \VV^*$ such that,
 $\forall\, \gamma \in [0,1]$, 
$\psi(\gamma) = \func(\xx + \gamma(\zz-\xx))$.
Then, it is easy to see that for every $\gamma \in \left]0,1\right[$, $\psi$ is differentiable at $\gamma$
and $\psi^\prime(\gamma) = \func^\prime(\xx + \gamma(\zz-\xx);\zz-\xx) = \scal{\nabla \func(\xx+\gamma(\zz-\xx))}{\zz-\xx}$. Moreover, 
$\psi$ is continuous on $[0,1]$. Therefore,
the fundamental theorem of calculus yields
\begin{equation*}
    \psi(\gamma) - \psi(0) = \int_0^\gamma \psi\prime(t) d t
\end{equation*}
and hence
\begin{align*}
    \func(\xx+\gamma(\zz-\xx)) - \func(\xx) - \scal{\nabla \func(\xx)}{\zz-\xx}
     &= \int_0^\gamma \scal{\nabla \func(\xx + t(\zz-\xx)) - \nabla \func(\xx)}{\zz-\xx} dt\\[1ex]
     &\leq \int_0^\gamma\nor{\nabla \func(\xx + t(\zz-\xx)) - \nabla \func(\xx)} \nor{\zz-\xx} dt \\[1ex]
     &\leq \int_0^\gamma L t\nor{\zz-\xx}^2 dt 
     \\[1ex]
     &=  L \frac{\gamma^2}{2}\nor{\zz-\xx}^2.
     \qedhere
\end{align*}
\end{proof}

The following result is an extension of a classical result on the directional differentiability of a max function \cite[Theorem~4.13]{bonnans2013perturbation} which relaxes the inf-compactness condition and allows the parameter space to be a convex set, instead of the entire Banach space. This result provide a prototype of functions (of which the entropic regularization of the Wasserstein distance is an instance) which are directionally differentiable only along the feasible directions of their domain and satisfies the hypotheses of \cref{p:inexactgrad}.

\begin{proposition}
\label{p:diff_of_max}
Let $Z$ and $\VV$ be real Banach spaces and
let $\BB$ be the topological dual of $\VV$. Let $\CC\subset \BB$
be a nonempty closed convex set, and let
$g\colon Z\times \BB\to \R$ be such that
\begin{enumerate}[$1)$]
    \item for every $z \in Z$, $g(z,\cdot)\colon \BB \to \R$
    is G\^ateaux differentiable with derivative in $\VV$, and the partial derivative with respect to the second variable $D_2 g\colon Z \times \BB \to \VV$ is continuous.
    \item for every $\xx \in \CC$, $S(\xx):=\argmax_{Z} g(\cdot, \xx) \neq \varnothing$.
\item there exists a continuous mapping $\varphi\colon \CC \to Z$ such that, for every $\xx \in \CC$, $\varphi(\xx) \in S(\xx)$.
\end{enumerate}
Let $\func\colon \CC \to \R$ be defined
as
\begin{equation}
    \func(\xx) = \max_{z \in Z} g(z,\xx).
\end{equation}
Then, $\func$ is continuous, directionally differentiable, and, for every $\xx \in \CC$ and $\uu \in \feas_{\CC}(\xx)$
\begin{equation}
\label{eq:20190513b}
    \func^\prime(\xx;\uu) = \max_{z \in S(\xx)} \scal{D_2 g(z,\xx)}{\uu} = \scal{D_2 g(\varphi(\xx),\xx)}{\uu}.
\end{equation}
\end{proposition}
\begin{proof}
The function $\func$ is well defined, 
since by assumption $2)$, for every $\xx \in \CC$,
$\argmax_{Z}g(\cdot, \xx) \neq \varnothing$.
Let $\xx,u \in \CC$ with $\xx \neq u$.
Then, since $\varphi(\xx) \in S(\xx)$, we have
$\func(\xx) = g(\varphi(\xx),\xx)$ and hence
\begin{multline}
\label{eq:20190511f}
    \frac{\func(u) - \func(\xx) - \scal{D_2 g(\varphi(\xx), \xx)}{u - \xx}}{\nor{u - \xx}} \\
    \geq \frac{g(\varphi(\xx), u) - g(\varphi(\xx), \xx)-\scal{D_2 g(\varphi(\xx), \xx)}{u - \xx}}{\nor{u-\xx}} \to 0,
\end{multline}
since $g(\varphi(\xx),\cdot)$ is Fr\'echet differentiable\footnote{continuously G\^ateaux differentiable function are Fr\'echet differentiable \cite[pp.34-35]{bonnans2013perturbation}.} at $\xx$ with gradient $D_2 g(\varphi(\xx), \xx)$.
Now, 
$\varphi(u) \in S(u)$, and hence 
$\func(u) = g(\varphi(u), u)$. 
Moreover, $g(\varphi(u), \xx) \leq \func(\xx)$. Therefore,
\begin{multline}
\label{eq:20190503b}
    \frac{\func(u) - \func(\xx) - \scal{D_2 g(\varphi(\xx), \xx)}{u - \xx}}{\nor{u - \xx}} \\
    \leq \frac{g(\varphi(u), u) - g(\varphi(u), \xx)-\scal{D_2 g(\varphi(\xx), \xx)}{u - \xx}}{\nor{u-\xx}}.
\end{multline}
Let $\varepsilon>0$. Since $D_2 g$ is continuous, there exists $\delta>0$ such that, for every $z^\prime\in Z$ and $\xx^\prime \in \BB$
\begin{equation}
\label{eq:20190503a}
    \nor{z^\prime - \varphi(\xx)}\leq \delta\ \text{and}\ 
    \nor{\xx^\prime - \xx} \leq \delta\ \implies\ 
    \nor{D_2 g(z^\prime,\xx^\prime) - D_2 g(\varphi(\xx), \xx)} \leq \varepsilon.
\end{equation}
Moreover, since $\varphi\colon \CC \to Z$ is continuous,
there exists $\eta>0$ such that, 
\begin{equation}
\label{eq:20190503c}
\nor{u - \xx} \leq \eta\ \implies\ \nor{\varphi(u) - \varphi(\xx)} \leq \delta.
\end{equation}
Let $z^\prime \in Z$ and suppose that $\nor{z^\prime - \varphi(\xx)} \leq \delta$ 
and $\nor{u - \xx} \leq \delta$.
Define $\psi\colon [0,1] \to \R$ such that,
for every $s \in [0,1]$, $\psi(s) = g(z^\prime,\xx+s(u - \xx))$.
Then, $\psi$ is continuously differentiable on $[0,1]$ and $\psi^\prime(s) = \scal{D_2 g(z^\prime,\xx+s (u - \xx))}{u - \xx}$. Therefore,
\begin{equation}
    \psi(1) - \psi(0) = \int_0^1 \psi^\prime(s) d s
\end{equation}
and hence, it follows from \eqref{eq:20190503a} that
\begin{align*}
    \lvert g(z^\prime,u) &- g(z^\prime, \xx) - 
     \scal{D_2 g(\varphi(\xx),\xx)}{u -\xx} \rvert \\
    &= \Big\lvert \int_0^1 \scal{D_2 g(z^\prime,\xx + s (u - \xx)) 
    - D_2 g(\varphi(\xx), \xx)}{u - \xx} ds \Big\rvert\\
    &\leq \int_0^1 \nor{D_2 g(z^\prime,\xx + s (u - \xx)) - D_2 g(\varphi(\xx),\xx)} \nor{u - \xx} ds\\
    &\leq \varepsilon \nor{u - \xx}.
\end{align*}
Therefore, we derive from \eqref{eq:20190503c}, that for every $u \in \CC$ such that 
$\nor{u - \xx} \leq \min\{\eta, \delta\}$,
we have 
\begin{equation*}
    \bigg\lvert \frac{g(\varphi(u),u) - g(\varphi(u), \xx) - 
     \scal{D_2 g(\varphi(\xx),\xx)}{u -\xx}}{\nor{u - \xx}}  \bigg\rvert \leq \varepsilon.
\end{equation*}
This shows that
\begin{equation}
\label{eq:20190511g}
    \lim_{\substack{u \in \CC\\u \to \xx}}\frac{g(\varphi(u),u) - g(\varphi(u), \xx) - 
     \scal{D_2 g(\varphi(\xx),\xx)}{u -\xx}}{\nor{u - \xx}} = 0.
\end{equation}
Then, we derive from \cref{eq:20190511f}, \eqref{eq:20190503b}, and \cref{eq:20190511g}  that
\begin{equation}
\label{eq:20190513a}
    \lim_{\substack{u \in \CC\\u \to \xx}}\frac{\func(u) - \func(\xx) - 
     \scal{D_2 g(\varphi(\xx),\xx)}{u -\xx}}{\nor{u - \xx}} = 0.
\end{equation}
This implies that $\lim_{u \in \CC,u \to \xx} \func(u)=\func(\xx)$.
Moreover, if $\uu \in \feas_{\CC}(\xx)$, there exists $\lambda>0$
and $u \in \CC$ such that $\uu = \lambda(u - \xx)$ and, for every $t \in \left]0,1/\lambda\right]$,
\begin{multline}
\frac{\func(\xx + t \uu) - \func(\xx)}{t}  - \scal{D_2 g(\varphi(\xx),\xx)}{\uu}\\
=  \nor{\lambda(u - \xx)} \frac{\func(\xx + t \lambda(u - \xx)) - \func(\xx) - \scal{D_2 g(\varphi(\xx),\xx)}{t \lambda(u - \xx)}}{\nor{t \lambda(u - \xx)}} 
\end{multline}
and the right hand side goes to zero as $t \to 0^+$, because of \cref{eq:20190513a}.
Therefore, for every $z \in S(\xx)$, since $\func(\xx) = g(z,\xx)$ and $\func(\xx + t\uu) \geq g(z, \xx + t\uu)$, we have
\begin{equation*}
\scal{D_2 g(\varphi(\xx),\xx)}{\uu} = \lim_{t\to 0^+}\frac{\func(\xx + t \uu) - \func(\xx)}{t} 
\geq \lim_{t\to 0^+}\frac{g(z, \xx + t \uu) - g(z, \xx)}{t} = \scal{D_2 g(z,\xx)}{\uu}
\end{equation*}
and \cref{eq:20190513b} follows.
\end{proof}

\section{DAD problems and convergence of Sinkhorn-Knopp algorithm}
\label{sec:PFtheory}

In this section we review the basic concepts of the nonlinear Perron-Frobenius theory \cite{lemmens2012nonlinear} which provides tools for dealing with DAD problems and ultimately to study the key properties of the Sinkhorn potentials. This analysis will allow us to provide
in \cref{subsec:lipschitz-total-variation} an upper bound estimate for the Lipschitz constant of the gradient of $\bary$, which is needed in the Frank-Wolfe algorithm. 

\subsection{Hilbert's metric and the Birkhoff-Hopf theorem}

In the rest of the appendix we will assume $\X\subset\R^d$ to be a compact set. We denote by $\cont(\X)$ the space of continuous functions on $\X$ endowed with the sup norm, namely $\supnor{f} = \sup_{x\in\X} \abs{f(x)}$. 
Let $\nneg(\X)$ be the cone of non-negative continuous functions, that is, $f\in\cont(\X)$ such that $f(x)\geq0$ for every $x\in\X$. Also, we denote by $\posi(\X)$ the set of continuous and (strictly) positive functions on $\X$, which  turns out to be the interior of $\nneg(\X)$.

Let $\dist:\X\times\X\to\R_+$ be a positive, symmetric, and continuous  function and define $\kerfun:\X\times\X\to\R_{++}$ as
\begin{equation}
\label{eq:kerfunc}
    (\forall x,y\in\X) \qquad \kerfun(x,y) = e^{ -\frac{\cost(x,y)}{\varepsilon} }.
\end{equation}
Set $\diam = \sup_{x,y\in\X}~\dist(x,y)$. Then, we have $\kerfun(x,y)\in[e^{-\diameps},1]$ for all $x,y\in\X$. 
Let $\alpha\in\prob(\X)$. The operator $\lmap_\alpha\colon \cont(\X)\to\cont(\X)$ is defined as 
\begin{equation}
\label{def:lmap}
    (\forall f\in\cont(\X))\qquad \lmap_\alpha f\colon x\mapsto \int \kerfun(x,z) f(z)~d\alpha(z).
\end{equation}
Note that $\lmap_\alpha$ is linear and continuous. In particular, since $k(x,y)\in[0,1]$ for all $x,y\in\X$, we have
\begin{equation}
\label{eq:L-maps-nneg-to-nneg}
    (\forall\,f\in\nneg(\X))\qquad\lmap_\alpha f \geq 0
\end{equation}
and
\begin{equation}
\label{eq:L_norm}
    (\forall\,f\in\cont(\X))\qquad \nor{\lmap_\alpha f}_\infty \leq \nor{f}_\infty.
\end{equation}



\paragraph{Hilbert's Metric} The cone $\nneg(\X)$ induces a partial ordering $\leq$ on $\cont(\X)$, such that
\begin{equation}
(\forall\, f,f^\prime \in \cont(\X))
\qquad f \leq f^\prime \Leftrightarrow\ f^\prime - f \in \nneg(\X).
\end{equation}

According to \cite{lemmens2012nonlinear}, we say that a function $f^\prime\in\nneg(\X)$ {\em dominates} $f\in\cont(\X)$ if there exist $t,s\in\R$ such that
\begin{equation}
	t f^\prime \leq f \leq s f^\prime.
\end{equation}
This notion induces an equivalence relation on $\nneg(\X)$, denoted $f\sim f^\prime$, meaning that $f$ dominates $f^\prime$ and $f^\prime$ dominates $f$. The corresponding equivalence classes are called {\em parts} of $\nneg(\X)$. 
Let $f,f^\prime \in \nneg(\X)$ be such that $f \sim f^\prime$.
We define
\begin{equation}
	M(f/f^\prime) = \inf\{s \in \R \,\vert\, f \leq s f^\prime \} 
	\qquad \text{and} \qquad m(f/f^\prime) 
	= \sup\{t\in \R \,\vert\, t f^\prime\leq f\}.
\end{equation}
Note that $m(f/f^\prime)\leq M(f/f^\prime)$. Moreover,
for every $f,f^\prime\in\nneg(\X)$ such that $f \sim f^\prime$, we have that $\supp(f) = \supp(f^\prime)$
 and if $f^\prime\neq 0$ (hence $f\neq 0$), then
 \begin{equation}
     M(f/f^\prime) = 
     \max_{x \in \supp(f^\prime)} \frac{f(x)}{f^\prime(x)}>0
     \quad\text{and}\quad
     m(f/f^\prime) = \min_{x \in \supp(f^\prime)} \frac{f(x)}{f^\prime(x)}>0.
 \end{equation}
The {\em Hilbert's metric} is defined as
\begin{equation}
	d_H(f,f^\prime) = \log\frac{M(f/f^\prime)}{m(f/f^\prime)},
\end{equation}
for all $f\sim f^\prime$ with $f\neq 0$ and $f^\prime\neq 0$, $d_H(0,0) = 0$ and $d_H(f,f^\prime) = +\infty$ otherwise. Direct calculation shows that \cite[Proposition~2.1.1]{lemmens2013birkhoff} 
\begin{enumerate}[{\rm (i)}]
    \item $d_H(f,f^\prime) \geq 0$ and $d_H(f,f^\prime) = d_H(f^\prime,f)$, for every $f,f^\prime\in\nneg(\X)$;
    \item $d_H(f,f^{\prime\prime}) \leq d_H(f,f^\prime) + d_H(f^\prime,f^{\prime\prime})$, for every $f,f^\prime, f^{\prime\prime}\in\nneg(\X)$ with $f \sim f^\prime$ and 
    $f^\prime \sim f^{\prime\prime}$;
    \item $d_H(s f,t f^\prime) = d_H(f,f^\prime)$, for every $f,f^\prime\in\nneg(\X)$ and $s,t>0$.
\end{enumerate}
Note that $d_H$ is not a metric on the parts of $\nneg(\X)$. However the set $\posi(\X)\cap \partial B_1(0) = \{f\in\posi(\X) ~|~ \supnor{f} = 1\}$ equipped with $d_H$ is a complete metric space \cite{nussbaum1988hilbert}. Also, $d_H$ induces a metric on the rays of the parts of $\nneg(\X)$ \cite[Lemma 2.1]{lemmens2013birkhoff}. 

We now focus on $\posi(\X)$. 
A direct consequence of Hilbert's metric properties is the following.

\begin{lemma}[Hilbert's Metric on $\posi(\X)$]\label{lem:hilbert-metric-on-the-interior}
The interior of $\nneg(\X)$ corresponds to the set of (strictly) positive functions $\posi(\X)$ and is a part of $\nneg(\X)$ with respect to the equivalence relation induced by dominance.
For every $f,f^\prime\in\posi(\X)$,
\begin{equation}
	M(f/f^\prime) = \max_{x\in\X} ~\frac{f(x)}{f^\prime(x)} \qquad m(f/f^\prime) = \min_{x\in\X}~\frac{f(x)}{f^\prime(x)},
\end{equation}
and $M(f/f^\prime) \geq m(f/f^\prime) >0$. Therefore
\begin{equation}
\label{eq:20190509b}
	d_H(f,f^\prime) = \log \max_{x,y\in\X}~\frac{f(x)~f^\prime(y)}{f(y)~f^\prime(x)}.
\end{equation}
\end{lemma}

\begin{proof}
Since $\X$ is compact it is straightfoward to see that $\posi(\X)$ is the interior of $\nneg(\X)$. By applying \cite[Lemma 1.2.2]{lemmens2012nonlinear} we have that $\posi(\X)$ is a part of $\nneg(\X)$. The characterization of $M(f/f^\prime)$ and $m(f/f^\prime)$ follow by direct calculation from the definition using the fact that $\inf_\X h = \min_\X h >0$ for any $h
\in\posi(\X)$ since $\X$ is compact. Finally, the characterization of Hilbert's metric on $\posi(\X)$ is obtained by recalling that $(\min_{x\in\X}h(x))^{-1} = \max_{x\in\X}h(x)^{-1}$ for every $h\in\posi(\X)$.
\end{proof}

\begin{lemma}[Ordering properties of $\lmap_\alpha$]\label{lem:order-preserving}
Let $\alpha\in\prob(\X)$. Then the following holds:
\begin{enumerate}[{\rm (i)}]
\item\label{lem:order-preserving_i} 
the operator $\lmap_\alpha$ is {\em order-preserving} (with respect to the cone $\nneg(\X)$),
that is,
\begin{equation}
(\forall\, f, f^\prime \in \cont(\X))
\qquad f\leq f^\prime\ \Rightarrow\ 
\lmap_\alpha f \leq \lmap_\alpha f^\prime;
\end{equation}
\item\label{lem:order-preserving_ii} 
$\lmap_\alpha$ maps parts of $\nneg(\X)$ to parts of $\nneg(\X)$, that is,
\begin{equation}
(\forall\, f, f^\prime \in \cont(\X))
\qquad f\sim f^\prime\ \Rightarrow\ \lmap_\alpha f \sim \lmap_\alpha f^\prime;
\end{equation}
\item\label{lem:order-preserving_iii} 
$\lmap_\alpha(\nneg(\X)) \subset \posi(\X)\cup\{0\}$ and $\lmap_\alpha(\posi(\X))\subset\posi(\X)$.
\end{enumerate}
\end{lemma}
\begin{proof}
\ref{lem:order-preserving_i}:
Let $f,f^\prime \in \cont(\X)$ with
$f\leq f^\prime$. Then $f^\prime-f\in\nneg(\X)$ and by linearity of $\lmap_\alpha$ combined with \cref{eq:L-maps-nneg-to-nneg}, we have 
$\lmap_\alpha f^\prime - \lmap_\alpha f = \lmap_\alpha(f-f^\prime) \geq 0$.

\ref{lem:order-preserving_ii}:
Let $f,f^\prime \in \nneg(\X)$ with
$f\sim f^\prime$. Then there exist
 $t,s\in\R$ and $s^\prime,t^\prime \in \R$ such that
 $t f^\prime \leq f \leq s f^\prime$
 and $t^\prime f \leq f^\prime \leq s^\prime f$.
Since $L_\alpha$ is linear and order-preserving, we have
 $\lmap_\alpha f \sim \lmap_\alpha f^\prime$.
 
 \ref{lem:order-preserving_iii}:
 Let $f\in\nneg(\X)$. By \cref{eq:L-maps-nneg-to-nneg} and 
\cref{eq:L_norm}, for any $x\in\X$
\begin{equation}\label{eq:max-upper-bound-L-alpha}
0 \leq (\lmap_\alpha f)(x) \leq \supnor{\lmap_\alpha f} \leq \int f(x)~ d\alpha(x) = \nor{f}_{L^1(\X,\alpha)}. 
\end{equation}
Moreover,
\begin{equation}\label{eq:min-lower-bound-L-alpha}
	 \lmap_\alpha f(x) = \int k(y,x) f(y)~d\alpha(y) \geq e^{-\diameps}~\nor{f}_{L^1(\X,\alpha)}.
\end{equation}
Therefore, if $\nor{f}_{L^1(\X,\alpha)} = 0$ then by \cref{eq:max-upper-bound-L-alpha} $\lmap_\alpha f = 0$ while, if $\nor{f}_{L^1(\X,\alpha)}>0$ then by \cref{eq:min-lower-bound-L-alpha} $\lmap_\alpha f\in\posi(\X)$. We conclude that the operator $\lmap_\alpha$ maps $\nneg(\X)$ in $\posi(\X)\cup\{0\}$. Moreover, $\lmap_\alpha(\posi(\X))\subset\posi(\X)$, since for every $f\in\posi(\X)$ we have $\nor{f}_{L^1(\X,\alpha)}\geq \min_{\X} f>0$.
\end{proof}

Following \cite[Section~A.4]{lemmens2012nonlinear} we now introduce a quantity which plays a central role in our analysis. 
\begin{definition}[Projective Diameter of $\lmap_\alpha$] 
Let $\alpha\in\prob(\X)$. The {\em projective diameter of $\lmap_\alpha$} is
\begin{equation}\label{eq:projective-diameter}
	\Delta(\lmap_\alpha) = \sup\{d_H(\lmap_\alpha f, \lmap_\alpha f^\prime) ~|~ f,f^\prime\in\nneg(\X),~ \lmap_\alpha f\sim \lmap_\alpha f^\prime\}.
\end{equation}
\end{definition}


The following result shows that it is possible to find a finite upper bound on $\Delta(\lmap_\alpha)$ that is independent on $\alpha$. 

\begin{proposition}[Upper bound on the Projective Diameter of $\lmap_\alpha$]\label{rem:projective-diameter-upper-bound}
Let $\alpha\in\prob(\X)$. Then
\begin{equation}
    \Delta(\lmap_\alpha) \leq 2\diameps.
\end{equation}
\end{proposition}
\begin{proof}
Let $f,f^\prime \in \nneg(\X)$.
Recall that $\lmap_\alpha$ maps $\nneg(\X)$ into $\posi(\X)\cup\{0\}$ (see  \cref{lem:order-preserving}\cref{lem:order-preserving_iii}) and that $\{0\}$ and $\posi(\X)$ are two parts of $\nneg(\X)$ with respect to the relation $\sim$ (see \cite[Lemma 1.2.2]{lemmens2012nonlinear}). Now, if $\lmap_\alpha f = \lmap_\alpha f^\prime = 0$, then we have $d_H(\lmap_\alpha f, \lmap_\alpha f^\prime) = d_H(0,0) = 0$. Therefore it is sufficient to study the case that $\lmap_\alpha f,\lmap_\alpha f^\prime\in\posi(\X)$. 
Following the characterization of Hilbert's metric on $\posi(\X)$ given in \cref{lem:hilbert-metric-on-the-interior},
we have
\begin{align*}
d_H(\lmap_\alpha f,\lmap_\alpha f^\prime) & = \log~ \max_{x,y\in\X}~ \frac{(\lmap_\alpha f)(x)~(\lmap_\alpha f^\prime)(y)}{(\lmap_\alpha f)(y)~(\lmap_\alpha f^\prime)(x)} \\
& = \log~ \max_{x,y\in\X}~\frac{\int \kerfun(x,z)f(z)~d\alpha(z)~\int \kerfun(y,w)f^\prime(w)~d\alpha(w)}{\int \kerfun(y,z)f(z)~d\alpha(z)~\int \kerfun(x,w)f^\prime(w)~d\alpha(w)} \\
& = \log~ \max_{x,y\in\X}~\frac{\int \kerfun(x,z)\kerfun(y,w)~f(z)f^\prime(w)~d\alpha(z)d\alpha(w)}{\int \kerfun(y,z)\kerfun(x,w)~f(z)f^\prime(w)~d\alpha(z)d\alpha(w)} \\
& = \log~ \max_{x,y\in\X}~\frac{\int \frac{\kerfun(x,z)\kerfun(y,w)}{\kerfun(y,z)\kerfun(x,w)}~\kerfun(y,z)\kerfun(x,w)~f(z)f^\prime(w)~d\alpha(z)d\alpha(w)}{\int \kerfun(y,z)\kerfun(x,w)~f(z)f^\prime(w)~d\alpha(z)d\alpha(w)} \\
& \leq \log~ \max_{x,y,z,w\in\X}~\frac{\kerfun(x,z)\kerfun(y,w)}{\kerfun(y,z)\kerfun(x,w)}.
\end{align*}
Since,  for every $x,y\in\X$, $\dist(x,y)\in[0,\diam]$, we have $k(x,y)\in[e^{-\diameps},1]$ and hence
\begin{equation*}
d_H(\lmap_\alpha f, \lmap_\alpha f^\prime) \leq 2\diameps.
\qedhere
\end{equation*}
\end{proof}

A consequence of \cref{rem:projective-diameter-upper-bound} is a special case of Birkhoff-Hopf theorem. 

\begin{theorem}[Birkhoff-Hopf Theorem]\label{thm:birkhoff-hopf}
Let $\lambda = \frac{e^{\diameps}-1}{e^{\diameps}+1}$ and  $\alpha\in\prob(\X)$. Then, for every $f,f^\prime\in\nneg(\X)$ such that $f\sim f^\prime$, we have
\begin{equation}
\label{eq:contraction1}
	d_H(\lmap_\alpha f,\lmap_\alpha f^\prime) \leq \lambda~d_H(f,f^\prime).
\end{equation}
\end{theorem}

\begin{proof}
The statement is a direct application of the Birkhoff-Hopf theory \cite[Sections A.4 and A.7]{lemmens2012nonlinear}
The {\em Birkhoff contraction ratio} of $\lmap_\alpha$ is defined as
\begin{equation*}
    \kappa(\lmap_\alpha) = 
    \inf\big\{ \hat{\lambda} \in \R_+ ~ \big\vert ~ d_H(\lmap_\alpha f, \lmap_\alpha f^\prime)\leq \hat{\lambda} d_H(f,f^\prime)~~ \forall f,f^\prime\in\nneg(\X),~~ f\sim f^\prime\big\}. 
\end{equation*}
Then it follows from 
Birkhoff-Hopf theorem
 \cite[Theorem~A.4.1]{lemmens2012nonlinear} that
\begin{equation}
    \kappa(\lmap_\alpha) = \tanh\left(\frac{1}{4}\Delta(\lmap_\alpha)\right).
\end{equation}
Recalling 
the upper bound on the projective diameter f $\lmap_\alpha$ given in \cref{rem:projective-diameter-upper-bound}, we have
\begin{equation*}
    \kappa(\lmap_\alpha) \leq \tanh\left(\frac{\diam}{2\varepsilon}\right) = \frac{e^{\diameps}-1}{e^{\diameps}+1} = \lambda,
\end{equation*}
and \eqref{eq:contraction1}
follows.
\end{proof}

\subsection{DAD problems}
\label{subsec:sinkiter}

\paragraph{The map $\tmap_\alpha$} 
Let $\alpha \in \prm(\X)$.
We define
the map $\tmap_\alpha\colon\posi(\X)\to\posi(\X)$, such that
\begin{equation}
\label{eq:Ta}
(\forall\,f\in\posi(\X))\qquad
	\tmap_\alpha(f) = \inv \circ \lmap_\alpha (f) = 1/(\lmap_\alpha f),
\end{equation}
where 
$\inv\colon\posi(\X)\to\posi(\X)$ is defined by $\inv (f) = 1/f$ with
\begin{equation}
    (1/f)\colon x\mapsto \frac{1}{f(x)}.
\end{equation}
Note that $\tmap_\alpha$ is well defined since, by \cref{lem:order-preserving}\cref{lem:order-preserving_iii}, $\lmap_\alpha(\posi(\X))\subset\posi(\X)$ and,
for every $f \in \posi(\X)$, $\min_\X f>0$, being $\X$ compact.
Moreover, it follows from \cref{eq:20190509b} in \cref{lem:hilbert-metric-on-the-interior}, that, for any two
$f,f^\prime\in\posi(\X)$
\begin{equation}\label{eq:hilbert-metric-for-inversion}
    d_H(1/f,1/f^\prime) = \log~\max_{x,y\in\X} \frac{f(y)f^\prime(x)}{f(x)f^\prime(y)} = d_H(f,f^\prime).
\end{equation}

We highlight here the connection between $\rmap_\alpha$ introduced in the main text in \cref{eq:rmap} and $\tmap_\alpha$, namely for any $\alpha\in\prob(\X)$ and $u\in\cont(\X)$
\begin{equation}
    \rmap_\alpha(u) = \eps\log(\tmap_\alpha(e^{u/\eps})).
\end{equation}

\paragraph{Dual $\oteps$ Problem} 
focus on the dual problem \cref{eq:dual_pb} of the optimal transport problem with entropic regularization. Let $\alpha,\beta\in\prob(\X)$ and $\varepsilon>0$, we consider
\begin{equation}\label{eq:ot-dual-problem}
	\max_{u,v\in\cont(\X)} ~ \int u(x)~d\alpha + \int v(y)~d\beta(y) - \varepsilon\int e^{\frac{u(x) + v(y) - \dist(x,y)}{\varepsilon}}~d\alpha(x)d\beta(y). 
\end{equation}
The optimality conditions for problem \eqref{eq:ot-dual-problem}
are
\begin{equation}
\label{eq:Dopt}
\begin{cases}
 e^{-\frac{u(x)}{\varepsilon}} = {\displaystyle\int_{\X}} 
 e^{\frac{v(y) - \cost(x,y)}{\varepsilon}}\,d\beta(y)\quad(\forall\, x \in \supp(\alpha))\\[3ex]
 e^{-\frac{v(y)}{\varepsilon}} = {\displaystyle\int_{\X}} e^{\frac{u(x) - \cost(x,y)}{\varepsilon}}\,d\alpha(x)\quad(\forall\, y \in \supp(\beta)),
\end{cases}
\end{equation}
which are equivalent to
\begin{equation}
\label{eq:Dopt2}
\begin{cases}
g(y)^{-1} = {\displaystyle\int_{\X}} e^{\frac{- \cost(x,y)}{\varepsilon}} f(x)\,d\alpha(x)\quad(\forall\, y \in \supp(\beta))\\[3ex]
f(x)^{-1} = {\displaystyle\int_{\X}} e^{\frac{-\cost(x,y)}{\varepsilon}} g(y)\,d\beta(y)\quad(\forall\, x \in \supp(\alpha)),
\end{cases}
\end{equation}
where $f = e^{u/\varepsilon} \in \posi(\X)$ and $g = e^{v/\varepsilon} \in \posi(\X)$.
In the rest of the section we will consider 
the following \emph{DAD problem} \cite{lemmens2012nonlinear,nussbaum1993entropy}\begin{equation}
\label{eq:DAD}
    (\forall\,y \in \X)\ \ 
    \int_{\X} f(x) \kerfun(x,y) g(y)\,d\alpha (x)=1
    \ \ \text{and}\ \ 
    (\forall\,x \in \X)\ \ 
    \int_{\X} f(x) \kerfun(x,y) g(y)\,d\beta (y) = 1.
\end{equation}
It is clear that a solution of \cref{eq:DAD}
is also a solution of \cref{eq:Dopt2}.
However, the vice versa is in general not true, even though
there is a canonical way to build solutions of \cref{eq:DAD}
starting from solutions of \cref{eq:Dopt2}: indeed
if $(f,g)$ is a solution of \cref{eq:Dopt2},
then the functions $\bar f,\bar g\colon \X \to \R$ defined through $\bar{f}(x)^{-1} = \int_{\X} \kerfun(x,y) g(y)\,d\beta(y)$ and $\bar{g}(y)^{-1} = \int_{\X} \kerfun(x,y) g(y)\,d\beta(y)$ provide a solution of \cref{eq:DAD}.
So, the dual $\oteps$ problem \cref{eq:ot-dual-problem} admits a solution if and only if the corresponding DAD problem \cref{eq:DAD} admits a solution.
Recalling the definition of $\tmap_\alpha$ in \eqref{eq:Ta},
problem \eqref{eq:DAD} can be more compactly written as
\begin{equation}\label{eq:fixed-point-f-g}
	f = \tmap_\beta(g) \qquad \text{and} \qquad g = \tmap_\alpha(f),
\end{equation}
or equivalently, by setting $\tmap_{\beta\alpha} = \tmap_\beta\circ\tmap_\alpha$ and 
$\tmap_{\alpha\beta} = \tmap_\alpha\circ\tmap_\beta$,
\begin{equation}\label{eq:fixed-point-f-f}
f =\tmap_{\beta\alpha}(f) \qquad \text{and} \qquad g = \tmap_{\alpha\beta}(g).
\end{equation}
This shows that the solutions of the DAD problem 
\cref{eq:DAD}
are the fixed points of $\tmap_{\alpha\beta}$ and $\tmap_{\beta\alpha}$ respectively. 
Note that the operators $\tmap_{\beta \alpha}$
and $\tmap_{\alpha\beta}$
are positively homogeneous,
that is, for every $t \in \R_{++}$ and $f \in \posi(\X)$,
$\tmap_{\beta\alpha}(t f) = t\tmap_{\beta\alpha}(f)$
and
$\tmap_{\alpha\beta}(t f) = t\tmap_{\alpha\beta}(f)$.
Thus, if $f$ is a fixed point of $\tmap_{\beta\alpha}$,
then $t f$ is also a fixed point of $\tmap_{\beta\alpha}$,
for every $t>0$.
If $(f,g)$ is a solution of the DAD problem \cref{eq:DAD}, then the pair $(u,v)$, with $u = \varepsilon\log f$ and $v = \varepsilon\log g$ is a solution of \cref{eq:ot-dual-problem}. We refer 
to these solutions as {\em Sinkhorn potentials} of the pair $(\alpha,\beta)$. 
Finally, note that, it follows from \cref{eq:Dopt}
that solutions of \cref{eq:ot-dual-problem}
are determined $(\alpha,\beta)$-a.e. on $\X$ and 
up to a translation of the form $(u+t,v-t)$, for some $t\in\R$.

The following result is essentially the specialization of \cite[Thm. 7.1.4]{lemmens2012nonlinear} to the case of the map $\tmap_{\beta\alpha}$. We report the proof here for convenience and completeness.

\begin{theorem}[Hilbert's metric contraction for $\tmap_{\beta\alpha}$]\label{thm:fixed-point-sinkhorn-iteration}
The map $\tmap_{\beta\alpha}:\posi(\X)\to\posi(\X)$ has a unique fixed point up to positive scalar multiples. Moreover, let $\lambda = \frac{e^{\diameps}-1}{e^{\diameps}+1}$. Then, for every $f,f^\prime\in\posi(\X)$, 
\begin{equation}\label{eq:birkhoff-contration-varphi-a-b}
	d_H(\tmap_{\beta\alpha}(f),\tmap_{\beta\alpha}(f^\prime)) \leq \lambda^2 ~d_H(f,f^\prime).
\end{equation}
\end{theorem}

\begin{proof}
By combining \cref{eq:hilbert-metric-for-inversion} with \cref{thm:birkhoff-hopf} we obtain that, for any $f,f^\prime\in\posi(\X)$
\begin{equation}
	d_H(\tmap_\alpha(f),\tmap_\alpha(f^\prime)) ~=~ d_H(1/(\lmap_\alpha f),1/(\lmap_\alpha f^\prime)) ~=~  d_H(\lmap_\alpha f,\lmap_\alpha f^\prime) ~\leq~ \lambda~ d_H(f,f^\prime).
\end{equation}
Since the same holds for $\tmap_\beta$ then \cref{eq:birkhoff-contration-varphi-a-b} is satisfied. 
Now, let $C = \posi(\X) \cap \partial B_1(0)$. Let $\overline\tmap_{\beta\alpha}\colon C\to C$ be the map
such that
\begin{equation}
	(\forall f\in C) \qquad \overline\tmap_{\beta\alpha}(f) = \frac{\tmap_{\beta\alpha} (f)}{\supnor{\tmap_{\beta\alpha} (f)}}. 
\end{equation}
Then, since $d_H(s f,t f^\prime) = d_H(f,f^\prime)$ for any $s,t>0$ and $f,f^\prime\in C$, we have
\begin{equation}
	d_H(\overline\tmap_{\beta\alpha}(f),\overline\tmap_{\beta\alpha}(f^\prime)) = d_H(\tmap_{\beta\alpha}(f),\tmap_{\beta\alpha}(f^\prime)) \leq \lambda^2 ~ d_H(f,f^\prime). 
\end{equation}
Since $(C,d_H)$ is a complete metric space \cite[Theorem~1.2]{nussbaum1988hilbert} and $\overline\tmap_{\beta\alpha}$ is a contraction, we can apply Banach's contraction theorem and conclude that there exists a unique fixed point of $\overline\tmap_{\beta\alpha}$, namely a function  $\bar f\in C$ such that 
\begin{equation}
	\bar f = \overline\tmap_{\beta\alpha}(\bar f) = \frac{\tmap_{\beta\alpha}(\bar f)}{\supnor{\tmap_{\beta\alpha}(\bar f)}}.
\end{equation}
Hence $\bar f$ is an eigenvector for $\tmap_{\beta\alpha}$ with eigenvalue 
$t=\lVert \tmap_{\beta\alpha}(\bar{f}) \rVert_{\infty}>0$. Now, we note that
\begin{equation}
\label{eq:20190517a}
 (\forall\,f,g \in \posi(\X))\quad  
 \scal {g \lmap_\alpha f}{\beta}
= \scal{f \lmap_\beta g}{\alpha}
= \int_{\X\times\X} f(x) k(x,y) g(y) d (\alpha\otimes\beta)(x,y).
\end{equation}
Set $\bar{g} = \tmap_{\alpha} (\bar{f})$,
so that $\tmap_\beta (\bar{g}) = t \bar{f}$.
Then, recalling the definitions of
$\tmap_\alpha$ and $\tmap_\beta$,
we have
$\bar{g} \lmap_\alpha \bar{f} \equiv 1$ and 
$t^{-1} \equiv \bar{f} \lmap_\beta \bar{g}$. Hence
$t^{-1} = \scal{\bar{f} \lmap_\beta \bar{g}}{\alpha} =  
\scal{\bar{g} \lmap_\alpha \bar{f}}{\beta} = 1$. Therefore $\bar{f}$ is a fixed point 
of $\tmap_{\beta\alpha}$.
Finally,
if $\bar f^\prime\in\posi(\X)$ is a fixed point of $\tmap_{\beta\alpha}$, then,
since $\tmap_{\beta\alpha}$
is positively homogeneous,
we have
\begin{equation}
    \overline\tmap_{\beta\alpha}(\bar{f}^\prime/\lVert \bar{f}^\prime \rVert_{\infty}) = 
    \dfrac{\tmap_{\beta\alpha}(\bar{f}^\prime/\lVert \bar{f}^\prime \rVert_{\infty})}
    {\lVert \tmap_{\beta\alpha}(f^\prime/\lVert \bar{f}^\prime \rVert_{\infty}) \rVert_{\infty}} =
    \dfrac{\tmap_{\beta\alpha}(\bar{f}^\prime)}
    {\lVert \tmap_{\beta\alpha}(\bar{f}^\prime)\rVert_{\infty}} = \dfrac{\bar{f}^\prime}{\lVert \bar{f}^\prime \rVert_{\infty}},
\end{equation}
that is,
$\bar f^\prime/\supnor{\bar f^\prime}$ is a fixed point of $\overline\tmap_{\beta\alpha}$. 
Thus, $\bar f^\prime/\supnor{\bar f^\prime} = \bar f$ and hence $\bar f^\prime$ is a multiple of $\bar f$.
\end{proof}

\begin{corollary}[Existence and uniqueness of Sinkhorn potentials]
Let $\alpha, \beta \in \prm(\X)$. Then,
the DAD problem \cref{eq:DAD} admits a solution $(f,g)$ and every other solution is of type $(t f, t^{-1} g)$, for some $t>0$.
Moreover, there exists a pair $(u,v) \in \cont(\X)^2$ of Sinkhorn potentials and every other pair of Sinkhorn potentials 
 is of type $(u + s, v- s)$, 
for some $s \in \R$.
In particular, for every $x_o \in \X$,
there exist a unique pair $(u,v)$
of Sinkhorn potentials such that $u(x_0) = 0$.
\end{corollary}
\begin{proof}
It follows from \cref{thm:fixed-point-sinkhorn-iteration} and the discussion after \cref{eq:fixed-point-f-f}.
\end{proof}

\paragraph{Bounding $(f,g)$ point-wise} We conclude this section by providing additional properties of the solutions $(f,g)$ of the DAD problem \cref{eq:fixed-point-f-g}. In particular, we show that there exists one such solution for which it is possible to provide a point-wise upper and lower bound independent on $\alpha$ and $\beta$. 

\begin{remark}
Let $f \in \posi(\X)$ and set $g=\tmap_\alpha (f)$. Then,
recalling \cref{eq:Ta} and \eqref{eq:L_norm}, we have that, for every $x \in \X$,
\begin{equation*}
	1= g(x) (\lmap_\alpha~ f)(x) 
	\leq g(x) \supnor{\lmap_\alpha f} \leq g(x) \supnor{f}
\end{equation*}
and 
\begin{equation*}
	1= g(x) (\lmap_\alpha~ f)(x)  \geq g(x)(\min_{\X} f) \int \kerfun(x,z)~d\alpha(z) \geq g(x)(\min_{\X} f)e^{-\diameps}.
\end{equation*}
Therefore,
\begin{equation}
\label{eq:20190430a}
    \min_{\X} g \geq \frac{1}{\supnor{f}}
    \quad\text{and}\quad
    \supnor{g} \leq \frac{e^{\diameps}}{\min_{\X} f}.
\end{equation}
\end{remark}

\begin{lemma}(Auxiliary Cone)
\label{lem:auxiliary-cone}
Consider the set
\begin{equation}
\label{eq:defconeK}
	K = \{f \in \nneg(\X) ~ | ~ f(x) \leq f(y)~ e^{\diameps} ~~ \forall x,y\in\X \}. 
\end{equation}
Let $\alpha \in \prm(\X)$. Then the following holds.
\begin{enumerate}[{\rm (i)}]
\item\label{lem:auxiliary-cone_i} 
$K$ is a closed convex cone and $K\subset\posi(\X)\cup\{0\}$;
\item\label{lem:auxiliary-cone_ii} $\lmap_\alpha(\cont_+(\X)) \subset K$;
\item\label{lem:auxiliary-cone_iii} $\mathsf{R}(K) \subset K$;
\item\label{lem:auxiliary-cone_iv}  
$\ran(\tmap_{\alpha}) \subset K$; 
\item\label{lem:auxiliary-cone_v} If $f \in K$ and $g=\tmap_\alpha f$,
then $g \in K$ and $1\leq (\min_{\X}g)\supnor{f} \leq \supnor{g} \supnor{f} \leq e^{2 \diameps}$.
\item\label{lem:auxiliary-cone_vi}
If $f \in K$ is such that $f(x_o)=1$ for some $x_o \in \X$, then $\nor{\varepsilon \log f}_{\infty} \leq \diam$.
\end{enumerate}
\end{lemma}
\begin{proof}
\ref{lem:auxiliary-cone_i}:
We see that for any $f\in K$, 
\begin{equation}
\label{eq:20190430b}
    \max_{\X} f \leq (\min_{\X} f)~ e^{\diameps},
\end{equation}
so, if $f(x) = 0$ for some $x\in\X$, then $f(x) = 0$ on all $\X$. Hence $K\subseteq\posi(\X)\cup\{0\}$. 
It is straightforward to verify that $K$ is a convex cone. Moreover $K$ is also closed. Indeed if $(f_n)_{n \in \N}$ is a sequence in $K$ which converges uniformly to $f \in \cont(\X)$, then, for every
$x,y \in \X$ and every $n \in \N$, $f_n(x)\leq f_n(y) e^{\diameps}$ and hence,
letting $n\to +\infty$,
we have $f(x) \leq f(y) e^{\diameps}$.

\ref{lem:auxiliary-cone_ii}:
For every $f\in \cont_+(\X)$ and $x,y\in\X$, we have
\begin{align*}
	(\lmap_\alpha f)(x) & = \int \kerfun(x,z) f(z) ~d\alpha(z) \\
					& = \int \frac{\kerfun(x,z)}{\kerfun(y,z)} ~ \kerfun(y,z)f(z)~d\alpha(z) \\
					& \leq e^{\diameps} \int \kerfun(y,z) f(z) ~d\alpha(z) \\
					& = e^{\diameps} (\lmap_\alpha f)(y).
\end{align*}

\ref{lem:auxiliary-cone_iii}:
For every $f\in K$,
\begin{equation*}
(\forall\, x,y \in \X)\qquad
	f(x) \leq f(y)~ e^{\diameps} ~\Leftrightarrow~ \frac{1}{f(y)} \leq \frac{1}{f(x)}~ e^{\diameps}.
\end{equation*}

\ref{lem:auxiliary-cone_iv}
It follows from \cref{lem:auxiliary-cone_ii} and \cref{lem:auxiliary-cone_iii}
and the definitions of $\tmap_\alpha$.

\ref{lem:auxiliary-cone_v}:
It follows from \ref{lem:auxiliary-cone_iv}, \eqref{eq:20190430a}, 
and \eqref{eq:20190430b}.

\ref{lem:auxiliary-cone_vi}:
Let $f \in K$ be such that $f(x_o)=1$. Then $\min_{\X} f \leq 1 \leq \max_{\X} f$.
Thus, it follows from \eqref{eq:20190430b} that
\begin{equation}
\label{eq:20190516b}
\max_{\X} f \leq e^{\diameps}\quad \text{and} \quad
\min_{\X} f \geq e^{-\diameps}
\end{equation}
and hence, for every $x \in \X$, $-\diam \leq \varepsilon \log f(x) \leq \diam$.
\end{proof}

As a direct consequence of \cref{lem:auxiliary-cone} we can establish a uniform point-wise upper and lower bound for the value of DAD solutions.

\begin{corollary}
\label{cor:dad-solutions-bounded}
Let $\alpha,\beta\in\prob(\X)$.
Let $x_o \in \X$ and let $(f,g)$ be
the solution of \cref{eq:fixed-point-f-g} such that $f(x_o) = 1$. 
Then $\nor{f}_\infty \leq e^{\diameps}$ and $\nor{g}_{\infty} \leq e^{2\diameps}$. Moreover,
the corrisponding pair $(u,v)$
of Sinkhorn potentials satifies 
$\supnor{u}\leq \diam$ and $\supnor{v} \leq 2\diam$.
\end{corollary}
\begin{proof}
Since $f$ and $g$ are fixed points of $\tmap_{\beta\alpha}$ and $\tmap_{\alpha\beta}$
respectively, it follows from \cref{lem:auxiliary-cone}\cref{lem:auxiliary-cone_iv} that $f,g \in K$.
Then, \cref{lem:auxiliary-cone}\cref{lem:auxiliary-cone_vi} 
yields $\norm{f}_{\infty} \leq e^{\diameps}$, whereas by the second of \cref{eq:20190430a} and \cref{eq:20190516b} we derive that  $\nor{g}_{\infty} \leq e^{2\diameps}$.
\end{proof}

\subsection{Sinkhorn-Knopp algorithm in infinite dimension}

In the context of optimal transport,  Sinkhorn-Knopp algorithm is often presented and studied in finite dimension \cite{cuturi2013sinkhorn,peyre2017computational}.
The algorithm 
originates from so called
\emph{matrix scaling problems},
also called \emph{DAD problems}, which consists in finding, for a given matrix $A$ with nonnegative entries, two diagonal matrices $D_1$, $D_2$
such that $D_1 A D_2$ is doubly stochastic \cite{sinkhorn1967}.
In our setting it is crucial to analyze the algorithm in infinite dimension.

\cref{thm:fixed-point-sinkhorn-iteration} shows that $\tmap_{\beta\alpha}$ is a contraction with respect to the Hilbert's metric. This suggests a direct approach to find the solutions of the DAD problem by adopting a fixed-point strategy, which turns out to applying the operators $\tmap_\alpha$ and $\tmap_{\beta}$ alternatively, starting from some $f^{(0)}\in\posi(\X)$. This is exactly the approach to the Sinkhorn algorithm pioneered by \cite{menon1967,Frank1989}
and further developed in an infinite dimensional setting in \cite{nussbaum1993entropy}.
In this section we review the algorithm and give the convergence properties
for the special  kernel $\kerfun$ in \cref{eq:kerfunc}. In particular we provide rate of convergence in the sup norm  $\supnor{\cdot}$.

\begin{algorithm}
\caption{Sinkhorn-Knopp algorithm (infinite dimensional case)}
\label{alg:Sinkhorn_cont}
Let $\alpha,\beta \in \prm(\X)$. Let $f^{(0)} \in \posi(\X)$ and define, 
\begin{equation*}
\begin{array}{l}
\text{for}\;\Siter=0,1,\dots\\[0.7ex]
\left\lfloor
\begin{array}{l}
g^{(\Siter+1)} = \tmap_{\alpha}(f^{(\Siter)})\\[1ex]
f^{(\Siter+1)} = \tmap_{\beta}(g^{(\Siter+1)})
\end{array}
\right.
\end{array}
\end{equation*}
\end{algorithm}


\begin{theorem}[Convergence of Sinkhorn-Knopp algorithm]
\label{thm:Sinkhornalgo}
Let $(f^{(\Siter)})_{\Siter \in \N}$ be defined according to \cref{alg:Sinkhorn_cont}.
Let $x_o \in \X$ and let $(f, g)$ be the solution of the 
DAD problem \eqref{eq:Dopt2}
such that $f(x_o) = 1$. Then,
defining, for every $\Siter \in \N$, 
$\tilde{f}^{(\Siter)} = f^{(\Siter)}/f^{(\Siter)}(x_o)$ and
$\tilde{g}^{(\Siter+1)} = g^{(\Siter+1)} f^{(\Siter)}(x_o)$, we have
\begin{equation}
\label{eq:sinkalgo2}
\begin{cases}
\lVert \log \tilde{f}^{(\Siter)} - \log f\rVert_{\infty} \leq \lambda^{2\Siter}
\bigg(\dfrac{\diam}{\varepsilon}+ \log \dfrac{\nor{f^{(0)}}_{\infty}}{\min_{\X} f^{(0)}} \bigg) \\[2.5ex]
\lVert \log \tilde{g}^{(\Siter+1)} - \log g\rVert_{\infty} \leq e^{3\diameps} \lVert \log \tilde{f}^{(\Siter)} - \log f\rVert_{\infty}.
\end{cases}
\end{equation}
Moreover, let the potentials $(u,v) = (\varepsilon\log f, \varepsilon\log g)$ and, for every $\Siter \in \N$, $(\tilde{u}^{(\Siter)}, \tilde{v}^{(\Siter)}) 
= (\varepsilon\log \tilde{f}^{(\Siter)}, \varepsilon\log \tilde{g}^{(\Siter)})$.
Then we have
\begin{equation}
\label{eq:sinkalgo3}
\lVert \tilde{u}^{(\Siter)} - u\rVert_{\infty} \leq \lambda^{2\Siter}
\bigg(\frac{\diam + \max\nolimits_{\X} u^{(0)} - \min\nolimits_{\X} u^{(0)}}{\varepsilon}\bigg).
\end{equation}
\end{theorem}
\begin{proof}
Let $\mathcal{A}$ be the set in \cref{lem:relation-supnor-hilbert}.
Clearly, for every $\Siter \in \N$, we have 
$f^{(\Siter+1)} = \tmap_{\beta \alpha} (f^{(\Siter)})$
and  $\bar{f}, \tilde{f}^{\Siter} \in \mathcal{A}$.
Thus, it follows from \cref{thm:fixed-point-sinkhorn-iteration}
and \cref{eq:relation-supnor-hilbert2} in \cref{lem:relation-supnor-hilbert} that,
for every $\Siter \in \N$,
\begin{equation*}
\lVert \log \tilde{f}^{(\Siter)} - \log f\rVert_{\infty} \leq  d_H(\tilde{f}^{\Siter}, f) 
 =  d_H(\tmap_{\beta \alpha}^{(\Siter)}(f^{(0)}), f)  \leq  \lambda^{2\Siter} d_H(f^{(0)},f).
\end{equation*}
Moreover, recalling \cref{eq:20190509b}, we have
\begin{equation*}
d_H(f^{(0)},f) = d_H (1/f^{(0)}, \lmap_{\beta} g)
= \log \max_{x,y \in \X} \frac{f^{(0)}(y) \lmap_\beta g (y)}{f^{(0)}(x) \lmap_\beta g (x)}
\leq \log \bigg[e^{\diameps} \max_{x,y \in \X} \frac{f^{(0)}(y) }{f^{(0)}(x) } \bigg] 
\end{equation*}
where we used the fact that $\lmap_\beta(\cont_{++}(\X)) \subset K$
and the definition \cref{eq:defconeK}. Thus, the first inequality in \cref{eq:sinkalgo2} follows.
The second inequality in \cref{eq:sinkalgo2} and \cref{eq:sinkalgo3} follow 
directly from \cref{lem:Lipschitz2} and the fact that $u^{(0)} = \varepsilon \log f^{(0)}$.
\end{proof}



\begin{algorithm}
\caption{Sinkhorn-Knopp algorithm (finite dimensional case)}
\label{algo:sinkalgo_disc}
Let $\mathsf{M} \in \R_{++}^{n_1\times n_2}$,
$\mathsf{a} \in \R^{n_1}_+$,
with $\mathsf{a}^\top \mathsf{1}_{n_1} = 1$,
and $\mathsf{b} \in \R^{n_2}_+$, with
$\mathsf{b}^\top \mathsf{1}_{n_2}=1$.
Let $\mathsf{f}^{(0)} \in \R^{n_1}_{++}$ and define
\begin{equation*}
\begin{array}{l}
\text{for}\;\Siter=0,1,\dots\\[0.7ex]
\left\lfloor
\begin{array}{l}
 \mathsf{g}^{(\Siter+1)}= \dfrac{\mathsf{b}}{\mathsf{M}^\top \mathsf{f}^{(\Siter)}}\\[2ex]
 \mathsf{f}^{(\Siter+1)}= \dfrac{\mathsf{a}}{\mathsf{M} \mathsf{g}^{(\Siter+1)}}. 
\end{array}
\right.
\end{array}
\end{equation*}
\end{algorithm}

\begin{proposition}
\label{rmk:discrete-sinkhorn}
Suppose that $\alpha$ and $\beta$ 
are probability measures with finite support. Then \cref{alg:Sinkhorn_cont} can be reduced to 
the finite dimensional \cref{algo:sinkalgo_disc}. More specifically,
suppose that $\alpha = \sum_{i=1}^{n_1} a_{i} \delta_{x_{i}}$,
and $\beta = \sum_{i=1}^{n_2} b_{i} \delta_{y_{i}}$, where
$\mathsf{a} = (a_i)_{1 \leq i \leq n_1} \in \R^{n_1}_+$, $\sum_{i=1}^n a_i = 1$ and
$\mathsf{b} = (b_i)_{1 \leq i \leq n_2} \in \R^{n_2}_+$, $\sum_{i=1}^n b_{i} = 1$.
Let $\mathsf{K} \in \R^{n_1\times n_2}$ be such that $\mathsf{K}_{i_1,i_2} = \kerfun(x_{i_1},y_{i_2})$
and let $\mathsf{M} = \mathrm{diag}(\mathsf{a})\mathsf{K}\mathrm{diag}(\mathsf{b}) \in \R^{n_1\times n_2}$.
Let $(\mathsf{f}^{(\Siter)})_{\Siter \in \N}$ and  $(f^{(\Siter)})_{\Siter \in \N}$ be defined according to \cref{algo:sinkalgo_disc} and \cref{alg:Sinkhorn_cont} respectively, with 
 $\mathsf{f}^{(0)} = (f^{(0)}(x_i))_{1 \leq i \leq n_1}$.
 Then, for every $\Siter \in \N$,
 \begin{equation*}
(\forall\, x \in \X)(\forall\, y \in \X)\ 
g^{(\Siter +1)}(y)^{-1} = \sum_{i_1=1}^{n_1} k(x_{i_1},y)a_{i_1} \mathsf{f}^{(\Siter)}_{i_1}
\ \text{and}\ f^{(\Siter +1)}(x)^{-1} = \sum_{i_2=1}^{n_2} k(x,y_{i_2}) 
b_{i_2} \mathsf{g}_{i_2}^{(\Siter+1)}.
\end{equation*}
Moreover, setting $u^{(\Siter)} = \varepsilon \log f^{(\Siter)}$, $v^{(\Siter)} = \varepsilon \log g^{(\Siter)}$, $\mathsf{u}^{(\Siter)} = \varepsilon \log \mathsf{f}^{(\Siter)}$,
and $\mathsf{v}^{(\Siter)} = \varepsilon \log \mathsf{g}^{(\Siter)}$, we have
 \begin{equation}
 \label{eq:20190523c}
 \begin{cases}
 \displaystyle
(\forall\, y \in \X)\quad
v^{(\Siter +1)}(y) = - \varepsilon \log \sum_{i_1=1}^{n_1} \exp( \mathsf{u}_{i_1}^{(\Siter)} - \cost(x_{i_1},y) ) a_{i_1}\\[1ex]
 \displaystyle
(\forall\, x \in \X)\quad u^{(\Siter +1)}(x) = - \varepsilon \log \sum_{i_2=1}^{n_2} 
\exp(\mathsf{v}_{i_2}^{(\Siter + 1)} - \cost(x,y_{i_2}) ) b_{i_2}.
 \end{cases}
\end{equation}
\end{proposition}
\begin{proof}
Since $\alpha$ and $\beta$ have finite support, 
we derive from the definitions of $f^{(\Siter+1)}$ and $g^{(\Siter+1)}$
in \cref{alg:Sinkhorn_cont} and that of $\tmap_\alpha$ and $\tmap_\beta$ that
\begin{equation*}
\begin{cases}
    \displaystyle(\forall\, x \in \X)\quad 
    g^{(\Siter+1)}(y)^{-1} = (\lmap_\alpha f^{(\Siter)}))(y)
    = \sum_{i_1=1}^{n_1} a_{i_1} \kerfun(x_{i_1}, y) f^{(\Siter)}(x_{i_1})\\[2ex]
    \displaystyle (\forall\, y \in \X)\quad 
    f^{(\Siter+1)}(x)^{-1} = (\lmap_\beta g^{(\Siter+1)}))(x)
    = \sum_{i_2=1}^{n_2} \kerfun(x, y_{i_2})b_{i_2} g^{(\Siter+1)}(y_{i_2}).
\end{cases}
\end{equation*}
Now, multiplying the above equations by $b_{i_2}$ and $a_{i_1}$ respectively, and recalling that $\mathsf{M}_{i_1,i_2} = a_{i_1} \kerfun(x_{i_1},y_{i_2}) b_{i_2}$, we have
\begin{equation*}
    \begin{bmatrix}
    b_{1} g^{(\Siter+1)}(y_1)^{-1}\\
    \vdots\\[0.8ex]
    b_{n_2} g^{(\Siter+1)}(y_{n_2})^{-1}
    \end{bmatrix}
    = \mathsf{M}^\top
    \begin{bmatrix}
    f^{(\Siter)}(x_{1})\\
    \vdots\\
    f^{(\Siter)}(x_{n_1})
    \end{bmatrix},
    \ \ 
        \begin{bmatrix}
    a_{1} f^{(\Siter+1)}(x_1)^{-1}\\
    \vdots\\[0.8ex]
    a_{n_1} f^{(\Siter+1)}(x_{n_1})^{-1}
    \end{bmatrix}
    = \mathsf{M}
    \begin{bmatrix}
    g^{(\Siter+1)}(y_{1})\\
    \vdots\\[0.8ex]
    g^{(\Siter+1)}(y_{n_2})
    \end{bmatrix},
\end{equation*}
and hence
\begin{equation*}
   \begin{bmatrix}
     g^{(\Siter+1)}(y_1)\\
    \vdots\\
    g^{(\Siter+1)}(y_{n_2})
    \end{bmatrix}
    = \mathsf{b} \bigg/ \mathsf{M}^\top
    \begin{bmatrix}
    f^{(\Siter)}(x_{1})\\
    \vdots\\
    f^{(\Siter)}(x_{n_1})
    \end{bmatrix},
    \ \ 
        \begin{bmatrix}
     f^{(\Siter+1)}(x_1)\\
    \vdots\\
    f^{(\Siter+1)}(x_{n_1})
    \end{bmatrix}
    = \mathsf{a} \bigg/ \mathsf{M}
    \begin{bmatrix}
    g^{(\Siter+1)}(y_{1})\\
    \vdots\\
    g^{(\Siter+1)}(y_{n_2})
    \end{bmatrix}.
\end{equation*}
Therefore, since 
$\mathsf{f}^{(0)} = (f^{(0)}(x_i))_{1 \leq i \leq n_1}$, recalling \cref{algo:sinkalgo_disc}, it follows by induction that, for every $\Siter \in \N$, $\mathsf{f}^{(\Siter)} = (f^{(\Siter)}(x_i))_{1 \leq i \leq n_1}$
and $\mathsf{g}^{(\Siter)} = (g^{(\Siter)}(x_i))_{1 \leq i \leq n_1}$. Thus, the first part of the statement follows. The second part follows directly from the definitions of $u^{(\Siter)}$, $v^{(\Siter)}$,
$\mathsf{u}^{(\Siter)}$, and $\mathsf{v}^{(\Siter)}$.
\end{proof}

\begin{remark}\ 
\normalfont
\begin{enumerate}[{\rm (i)}]
    \item 
Algorithm \cref{algo:sinkalgo_disc} is the classical (discrete) Sinkhorn algorithm
which was recently studied in several papers \cite{cuturi2013sinkhorn}.
It follows from \cref{thm:Sinkhornalgo} that 
considering the solution $(f,g)$ of the DAD problem
such that $f(x_1) = 1$ and
defining 
$\tilde{\mathsf{f}}^{(\Siter)} = \mathsf{f}^{(\Siter)}/\mathsf{f}^{(\Siter)}_0$
and $\tilde{\mathsf{g}}^{(\Siter)} = \mathsf{g}^{(\Siter)}\mathsf{f}^{(\Siter)}_0$, 
and $\mathsf{f}_i = f(x_i)$ and  $\mathsf{g}_j = g(y_j)$,
we have
\begin{equation*}
\lVert \log \tilde{\mathsf{f}}^{(\Siter)} - \log \mathsf{f}\rVert_{\infty} \leq \lambda^{2\Siter} \bigg(\dfrac{\diam}{\varepsilon}+ \log \dfrac{\max_i \mathsf{f}_i^{(0)}}{\min_i \mathsf{f}_i^{(0)}} \bigg).
\end{equation*}
\item The procedure {\sc SinkhornKnopp} discussed in the paper and called in
\cref{alg:practical-FW}, actually output the vector $\mathsf{v}=\varepsilon \log \mathsf{g}^{(\Siter)}$ for sufficiently large $\Siter$.
\item Referring to \cref{sec:algorithm-practice}
in the paper, we recognize that
the expressions on the right hand side of \cref{eq:20190523c} are precisely $\rmap_\alpha(u^{(\Siter)})(x)$ and 
$\rmap_\beta(v^{(\Siter+1)})(x)$ respectively.
\end{enumerate}
\end{remark}

\section{Lipschitz continuity of the gradient of Sinkhorn divergence with respect to the Total Variation}\label{subsec:lipschitz-total-variation}

In this section we show that the gradient of the Sinkhorn divergence is Lipschitz continuous  with respect to the Total Variation on $\prob(\X)$. 

We start by characterizing the relation between Hilbert's metric between functions of the form $f = e^{u/\varepsilon}$ and the $\supnor{\cdot}$ norm between functions of the form $u = \varepsilon \log f$. 

\begin{lemma}
\label{lem:relation-supnor-hilbert}
Let $f,f^\prime \in \posi(\X)$ and set $u = \varepsilon\log f$ and $u^\prime = \varepsilon\log f^\prime$. Then
\begin{equation}\label{eq:relation-supnor-hilbert}
	d_H(f,f^\prime) \leq 2\supnor{\log f - \log f^\prime} \quad \text{or, equivalently} \quad d_H(e^{u/\varepsilon},e^{u^\prime/\varepsilon}) \leq \frac{2}{\varepsilon} \supnor{u-u^\prime}.
\end{equation}
Moreover, let $ x_o\in\X$, consider the sets $\mathcal{A} = \{h\in\posi(\X) ~|~ h(x_o) = 1\}$
and $\mathcal{B} = \{w \in \cont(\X) ~|~ w(x_o) = 0\}$. Suppose that
 $f,f^\prime\in \mathcal{A}$ (or equivalently that $u,u^\prime\in \mathcal{B}$). Then
\begin{equation}
\label{eq:relation-supnor-hilbert2}
\frac{1}{2} d_H(f,f^\prime) \leq \supnor{\log f - \log f^\prime} \leq d_H(f,f^\prime).
\end{equation}
and
\begin{equation}
\label{eq:relation-supnor-hilbert3}
\frac{\varepsilon}{2} d_H(e^{u/\varepsilon},e^{u^\prime/\varepsilon}) \leq \supnor{u - u^\prime} \leq \varepsilon~ d_H(e^{u/\varepsilon},e^{u^\prime/\varepsilon}).
\end{equation}
\end{lemma}
\begin{proof}
We have
\begin{align*}
d_H(f,f^\prime) & = \log~\max_{x,y\in\X}~\frac{f(x)f^\prime(y)}{f(y)f^\prime(x)} \\
& = \log~\max_{x\in\X}~\frac{f(x)}{f^\prime(x)} + \log~\max_{y\in\X}~\frac{f^\prime(y)}{f(y)} \\
& = \max_{x\in\X} \log\frac{f(x)}{f^\prime(x)} + \max_{y\in\X}~\log\frac{f^\prime(y)}{f(y)} \\
& \leq 2\max_{x\in\X}~\left|\log\frac{f(x)}{f^\prime(x)}\right| \\
& = 2\supnor{\log(f/f^\prime)}\\[1ex]
& = 2\supnor{\log f - \log f^\prime}
\end{align*}
and \eqref{eq:relation-supnor-hilbert} follows. Suppose that
 $f,f^\prime\in \mathcal{A}$. Then
\begin{align*}
\supnor{\log f - \log f^\prime} & = \max\left\{\log \max_{x\in\X} \frac{f(x)}{f^\prime(x)}, \log \max_{x\in\X} \frac{f^\prime(x)}{f(x)}\right\} \\ 
& = \max\left\{\log \max_{x\in\X} \frac{f(x)f^\prime(\bar x)}{f(\bar x)f^\prime(x)}, \log \max_{x\in\X} \frac{f(\bar x)f^\prime(x)}{f(x)f^\prime(\bar x)}\right\} \\
& \leq \max\left\{\log \max_{x,y\in\X} \frac{f(x)f^\prime(y)}{f(y)f^\prime(x)}, \log \max_{x,y\in\X} \frac{f(y)f^\prime(x)}{f(x)f^\prime(y)}\right\} \\
& = d_H(f,f^\prime),
\end{align*}
since $f(x_o)/f^\prime(_ox) = f^\prime(x_o)/f(x_o) = 1$. Therefore, 
\eqref{eq:relation-supnor-hilbert2} follows.
\end{proof}

\begin{lemma}
\label{lem:20190510a}
For every $x,y \in \R_{++}$ we have
\begin{equation}
    \lvert \log x- \log y \rvert \leq 
    \max\big\{x^{-1}, y^{-1}\big\}
    \lvert x - y \rvert.
\end{equation}
\end{lemma}

The following result allows to extend the previous observations on a pair $f,f^\prime$ to the corresponding $g = \tmap_\alpha f$ and $g^\prime = \tmap_\alpha f^\prime$.

\begin{lemma}
\label{lem:Lipschitz2}
Let $x_o \in \X$ and $K\subset\nneg(\X)$ the cone from \cref{lem:auxiliary-cone}. Let $f,f^\prime \in K$ be such that
$f(x_o)=f^\prime(x_o)=1$, and set $g = \tmap_\alpha f$ and $g^\prime = \tmap_\alpha f^\prime$. Then,
\begin{equation}
    \supnor{\log g - \log g^\prime} \leq e^{3 \diameps} \supnor{\log f - \log f^\prime}.
\end{equation}
\end{lemma}
\begin{proof}
It follows from \cref{eq:Ta} and \cref{lem:20190510a} that
\begin{equation*}
    \lvert \log g - \log g^\prime \rvert 
    = \Big\lvert \log \frac{g}{g^\prime} \Big\rvert 
    = \Big\lvert \log \frac{\lmap_\alpha f^\prime}{\lmap_\alpha f} \Big\rvert
    \leq \max\big\{g^\prime,
    g\big\} \lvert \lmap_\alpha f - \lmap_\alpha f^\prime \rvert.
\end{equation*}
Therefore, since $1 \leq \supnor{f}, \supnor{f^\prime}$, and recalling \cref{lem:auxiliary-cone}\cref{lem:auxiliary-cone_v} and \cref{eq:L_norm},
we have
\begin{align*}
    \supnor{\log g - \log g^\prime} 
    &\leq \max\{\supnor{g},\supnor{g^\prime}\} 
    \supnor{\lmap_\alpha f-\lmap_\alpha f^\prime}\\
    &\leq \max\{\supnor{f}\supnor{g},\supnor{f^\prime}\supnor{g^\prime}\} 
    \supnor{\lmap_\alpha f-\lmap_\alpha f^\prime}\\
    &\leq e^{2\diameps} \supnor{f - f^\prime}\\
    &= e^{2\diameps} \lVert e^{\log f} - e^{\log f^\prime} \rVert_{\infty}.
\end{align*}
Now, since $f,f^\prime\leq e^{\diameps}$,
we have $\log f, \log f^\prime \leq \diameps$. Thus,
the statement follows by noting that the exponential function 
is Lipschitz continuous on $\left]-\infty, \diameps\right]$
with constant $e^{\diameps}$.
\end{proof}

We are ready to prove the main result of the section.

\begin{theorem}[Lipschitz continuity of the Sinkhorn potentials with respect to the total variation]\label{prop:lipschitz-continuity-total-variation2}
Let $\alpha,\beta,\alpha',\beta'\in\prob(\X)$ and let $x_o \in \X$. Let $(u,v),(u^\prime,v')\in\cont(\X)^2$ be the two pairs of Sinkhorn potentials corresponding to the solution of the regularized OT problem in \cref{eq:ot-dual-problem} for $(\alpha,\beta)$ and $(\alpha',\beta')$ respectively such that $u(x_o) = u^\prime(x_o) = 0$.  
Then
\begin{equation}
\label{eq:20190503d}
    \supnor{u - u^\prime} \leq 2\varepsilon  e^{3\diameps}\nor{(\alpha-\alpha',\beta - \beta')}_{TV}.
\end{equation}
Hence, the map which, for each pair of probability distributions $(\alpha,\beta)\in\prob(\X)^2$ associates the component $u$ of the corresponding Sinkhorn potentials is $2\varepsilon e^{3\diameps}$-Lipschitz continuous {\em with respect to the total variation}. 
\end{theorem}
\begin{proof}
The functions $f = e^{u/\varepsilon}$ and $f^\prime = e^{u^\prime/\varepsilon}$ are fixed points of the maps $\tmap_{\beta\alpha}$ and $\tmap_{\beta'\alpha'}$ respectively. Then,
it follows from \cref{thm:fixed-point-sinkhorn-iteration} that
\begin{align*}
	d_H(f,f^\prime) & = d_H(\tmap_{\beta\alpha}(f),\tmap_{\beta'\alpha'}(f^\prime)) \\
	& \leq d_H(\tmap_{\beta\alpha}(f),\tmap_{\beta'\alpha'}(f)) + d_H(\tmap_{\beta'\alpha'}(f),\tmap_{\beta'\alpha'}(f^\prime)) \\
	& \leq d_H(\tmap_{\beta\alpha}(f),\tmap_{\beta'\alpha'}(f)) + \lambda^2 d_H(f,f^\prime),
\end{align*}
hence,
\begin{equation}
\label{eq:20190511b}
	d_H(f,f^\prime) \leq \frac{1}{1 - \lambda^2}~d_H(\tmap_{\beta\alpha}(f),\tmap_{\beta'\alpha'}(f)). 
\end{equation}
Moreover, using \cref{eq:relation-supnor-hilbert}, we have
\begin{align}
\nonumber	d_H(\tmap_{\beta\alpha}(f),\tmap_{\beta'\alpha'}(f)) & \leq d_H(\tmap_{\beta\alpha}(f),\tmap_{\beta'\alpha}(f)) + d_H(\tmap_{\beta'\alpha}(f),\tmap_{\beta'\alpha'}(f)) \\
\nonumber	& \leq d_H(\tmap_{\beta}(g),\tmap_{\beta'}(g)) + \lambda d_H(\tmap_\alpha(f),\tmap_{\alpha'}(f)) \\ 
\label{eq:20190511a}	& \leq 2\supnor{\log~\frac{\tmap_\beta(g)}{\tmap_{\beta'}(g)}} +2\lambda \supnor{\log~\frac{\tmap_\alpha(f)}{\tmap_{\alpha'}(f)} }.
\end{align}
Now, note that by \cref{lem:20190510a}
\begin{equation}
    	\bigg\lvert \log~\frac{\tmap_\beta(g)}{\tmap_{\beta^\prime}(g)}\bigg\rvert = \bigg\lvert \log~\frac{\lmap_{\beta^\prime} g}{\lmap_\beta g}\bigg\rvert \leq \max\{1/\lmap_\beta g, 1/\lmap_{\beta^\prime} g\} \lvert (\lmap_{\beta'}-\lmap_\beta) g  \rvert
\end{equation}
and that, for every $x \in \X$,
\begin{equation}
\begin{aligned}\label{eq:pairing-for-TV-lipschitz}
    [(\lmap_{\beta^\prime}-\lmap_\beta) g](x)  & = \int \kerfun(x,z) g(z)~d(\beta-\beta^\prime)(z) \\
    & = \scal{\kerfun(x,\cdot) g}{\beta-\beta^\prime}
     \leq \supnor{g}\nor{\beta - \beta^\prime}_{TV},
\end{aligned}
\end{equation}
and, similarly, $[(\lmap_{\beta}-\lmap_{\beta^\prime}) g](x) \leq \supnor{g}\nor{\beta - \beta'}_{TV}$.
Therefore, since $1/(\lmap_\beta g) = \tmap_\beta(g) = f$ and $\lmap_{\beta^\prime}g \geq e^{-\diameps} \min g$, it follows from \cref{lem:auxiliary-cone}\cref{lem:auxiliary-cone_v}
and \cref{eq:20190430b} (applied to $g$) that 
\begin{equation}
\label{eq:20190511c}
\supnor{\log~\frac{\tmap_\beta(g)}{\tmap_{\beta'}(g)} } \leq
\max\left\{\supnor{f},\frac{e^{\diameps}}{\min g}\right\} \supnor{g}\nor{\beta-\beta'}_{TV} \leq
e^{2\diameps}~\nor{\beta-\beta'}_{TV}. 
\end{equation}
Analogously, we have 
\begin{equation}
\label{eq:20190511d}
	\supnor{\log~\frac{\tmap_\alpha(f)}{\tmap_{\alpha'}(f)}} \leq e^{2\diameps}~\nor{\alpha-\alpha'}_{TV}.
\end{equation}
Putting \cref{eq:20190511b}, \cref{eq:20190511a}, \cref{eq:20190511c}, and \cref{eq:20190511d} together, we have
\begin{equation}
	d_H(f,f^\prime) \leq \frac{2e^{2\diameps}}{1-\lambda^2}\left(\lambda \nor{\alpha-\alpha'}_{TV} + \nor{\beta - \beta'}_{TV}\right). 
\end{equation}
Now, note that since $e^{\diameps}\geq1$
\begin{equation}
    \frac{1}{1-\lambda^2} = \frac{(e^{\diameps} + 1)^2}{4e^\diameps} \leq e^{\diameps}.
\end{equation}
Finally, recalling \cref{eq:relation-supnor-hilbert3}, we have
\begin{equation}
	\supnor{u - u^\prime} \leq 2\varepsilon e^{3\diameps}\nor{(\alpha-\alpha',\beta - \beta')}_{TV},
\end{equation}
where $\nor{(\alpha-\alpha',\beta - \beta')}_{TV} = \nor{\alpha-\alpha'}_{TV} + \nor{\beta - \beta'}_{TV}$ is the total variation norm on $\meas(\X)^2$.
\end{proof}

\begin{corollary}
\label{cor:Lipschitz2}
Under the assumption of \cref{prop:lipschitz-continuity-total-variation2}, we have
\begin{equation}
\supnor{u - u^\prime}+\supnor{v - v^\prime} \leq 2 \varepsilon e^{3\diameps}(1+ \varepsilon e^{3\diameps})
\nor{(\alpha-\alpha',\beta - \beta')}_{TV}.
\end{equation}
\end{corollary}
\begin{proof}
It follows from \cref{prop:lipschitz-continuity-total-variation2} and \cref{lem:Lipschitz2}.
\end{proof}

We finally address the issue of the differentiability of the Sinkhorn divergence.
We first recall a few facts about the directional differentiability of $\oteps$ briefly recalled in \cref{sec:background} of the main text. For a more in-depth analysis on this topic we refer the reader to \cite{feydy2018interpolating} (in particular Proposition $2$).

\begin{fact}
\label{rem:directional-derivatives-oteps}
Let $x_{o}\in\X$, $\alpha,\beta\in\prob(\X)$ and $(u,v)\in\cont(\X)^2$ be the pair of corresponding Sinkhorn potentials with $u(x_o) = 0$. The function $\oteps$
is directionally differentiable and
the directional derivative of $\oteps$ in $(\alpha,\beta)$ along a feasible direction $(\mu,\nu)\in\feas_{\prob(\X)^2}\big((\alpha,\beta)\big)$ (see \cref{def:directional-derivative}) is
\begin{equation}\label{eq:directional-derivative-oteps}
    \oteps^\prime(\alpha,\beta;\mu,\nu) = \int u(x)~d\mu(x) + \int v(y)~d\nu(y) = \scal{(u,v)}{(\mu,\nu)}.
\end{equation}
Let $\nabla\oteps\colon\prob(\X)^2\to\cont(\X)^2$ 
be the operator that maps every pair of probability distributions $(\alpha,\beta)\in\prob(\X)^2$ to the corresponding pair of Sinkhorn potentials $(u,v)\in\cont(\X)^2$ with $u(x_o)=0$. Then \cref{eq:directional-derivative-oteps} can be written as
\begin{equation}
\oteps^\prime(\alpha,\beta;\mu,\nu) = \scal{\nabla \oteps(\alpha,\beta)}{(\mu,\nu)}.
\end{equation}
\end{fact}

\begin{remark}
\normalfont
In \cref{rem:directional-derivatives-oteps}, the requirement $u(x_o) = 0$ is only a convention to remove ambiguities. Indeed, for every $t\in\R$, replacing the Sinkhorn potential $(u+t,u-t)$ in \cref{def:cone-of-feasible-directions} does not affect \cref{eq:directional-derivative-oteps}.
\end{remark}

\begin{fact}
\label{f:Sinkhorn}
Let $\beta \in \prm(\X)$ and let $\nabla_1 \oteps$ be the first component of the gradient operator defined in \cref{rem:directional-derivatives-oteps}. Then the
 Sinkhorn divergence function 
$S_\varepsilon(\cdot, \beta)\colon \prm(\X) \to \R$ in \cref{eq:sink_divergence} is directionally differentiable
and, for every $\alpha \in \prm(\X)$ and every $\mu \in \feas_{\prm(\X)}(\alpha)$,
\begin{equation*}
[S_\varepsilon(\cdot, \beta)]^\prime(\alpha; \mu) 
= \scal{\nabla_1 \oteps(\alpha,\beta)- \nabla_1 \oteps(\alpha,\alpha)}{\mu}.
\end{equation*}
So, one can define
$\nabla S_\varepsilon(\cdot, \beta)\colon \prob(\X) \to \cont(\X)$ such that, for every $\alpha \in \prm(\X)$, $\nabla [S_\varepsilon(\cdot, \beta)](\alpha) = \nabla_1 \oteps(\alpha,\beta) - \nabla_1 \oteps(\alpha,\alpha)$ and we have
\begin{equation}
[S_\varepsilon(\cdot, \beta)]^\prime(\alpha; \mu) 
=\scal{\nabla S_\varepsilon(\cdot, \beta)}{\mu}.
\end{equation}
Finally, if  $\kerfun$
in \cref{eq:kerfunc} is a positive definite kernel,
then the Sinkhorn divergence $S_\varepsilon(\cdot, \beta)$ is convex.
\end{fact}

We are now ready to prove \cref{thm:lip-continuity-total-variation-informal} in the paper.
We recall also the statement for reader's convenience.

\TLipschitzContinuityTV*
\begin{proof}
The first part is just a consequence of \cref{prop:lipschitz-continuity-total-variation2}
and \cref{rem:directional-derivatives-oteps}. The second part, follows from the first part and \cref{f:Sinkhorn}.
\end{proof}

\begin{remark}
\normalfont
It follows from the optimality conditions
\cref{eq:Dopt} that, for every  $x \in \supp(\alpha)$
and $y \in \supp(\beta)$,
\begin{equation*}
1 = \int_{\X} e^{\frac{u(x) + v(y) - \cost(x,y)}{\varepsilon}} d\beta(y)
    \quad\text{and}\quad
    1 = \int_{\X} e^{\frac{u(x) + v(y) - \cost(x,y)}{\varepsilon}} d\alpha(x),
\end{equation*}
hence,
\begin{equation}
    \int_{\X} e^{\frac{u\oplus v - \cost}{\varepsilon}} d \alpha\otimes\beta = 1.
\end{equation}
Then, recalling the definition of $\oteps$ in \cref{eq:dual_pb}
and that of its gradient, given above,
we have
\begin{equation}
    \oteps(\alpha,\beta) = 
    \scal{\nabla \oteps(\alpha,\beta)}{(\alpha,\beta)} - \varepsilon.
\end{equation}
Since, $\nabla \oteps$ is bounded and Lipschitz continuous, it  follows that $\oteps$ is Lipschitz continuous with respect to the total variation.
\end{remark}

We end the section by providing an independent proof of \cref{rem:directional-derivatives-oteps}, which is based on \cref{p:diff_of_max} and \cref{cor:Lipschitz2}.
\begin{proposition}
The function 
$\oteps\colon \prm(\dom)^2 \to \R$, defined in \cref{eq:dual_pb}, is continuous with respect to the total variation,
directionally differentiable, and,
for every $(\alpha,\beta) \in \prm(\dom)^2$
and every feasible direction 
$(\mu,\nu) \in \feas_{\prm(\dom)^2}(\alpha,\beta)$, we have
\begin{equation}
 \oteps^\prime(\alpha,\beta;\mu,\nu)
 = \scal{(u,v)}{(\mu,\nu)},
\end{equation}
where $(u,v) \in \cont(\X)^2$ is any
solution of problem \cref{eq:dual_pb}.
\end{proposition}
\begin{proof}
Let $g\colon \cont(\X)^2\times \finmeas(\X)^2 \to \R$
be such that,
\begin{equation}
    g((u,v),(\alpha,\beta)) = \scal{u}{\alpha} + \scal{v}{\beta}-\varepsilon\scal{\exp( (u \oplus v - \cost)/\varepsilon)}{\alpha\otimes \beta}.
\end{equation}
Then, for every $(\alpha,\beta) \in \prm(\X)^2$,
\begin{equation}
\label{eq:20190512a}
    \oteps(\alpha,\beta) = \max_{(u,v) \in \cont(\X)^2} g((u,v),(\alpha,\beta)).
\end{equation}
Thus, $\oteps$ is of the type considered in \cref{p:diff_of_max}.
Let $(u,v) \in \cont(\X)$. Then the function $g((u,v), \cdot)$
admits directional derivatives and, for every $(\alpha,\beta), (\mu,\nu) \in \finmeas(\X)^2$, we have
\begin{multline}
\label{eq:20190511e}
    [g((u,v),\cdot)]^\prime((\alpha,\beta);(\mu,\nu)) \\[1ex]=
    \Big\langle u - \varepsilon e^{\frac{u}{\varepsilon}} \int_{\X} e^{\frac{v - c(\cdot,y)}{\varepsilon}} d\beta(y), \mu \Big\rangle
    + \Big\langle v - \varepsilon e^{\frac{v}{\varepsilon}} \int_{\X} e^{\frac{u - c(x,\cdot)}{\varepsilon}} d\alpha(x), \nu \Big\rangle.
\end{multline}
Indeed, for every $t >0$,
\begin{align*}
    \frac{1}{t} \big[g((u,v),& (\alpha,\beta) + t(\mu,\nu)) - g((u,v),(\alpha,\beta)) \big] \\
    &= \frac{1}{t}
    \big[ \scal{u}{\alpha+ t \mu} + \scal{v}{\beta + t \nu}-\varepsilon\scal{\exp( (u \oplus v - \cost)/\varepsilon)}{(\alpha + t \mu)\otimes (\beta + t \nu)}\\
    &\qquad- \scal{u}{\alpha} -\scal{v}{\beta}
    +\varepsilon\scal{\exp( (u \oplus v - \cost)/\varepsilon)}{\alpha\otimes \beta} \big]\\
    &= 
     \scal{u}{\mu} + \scal{v}{\nu} 
    -  \varepsilon\scal{\exp( (u \oplus v - \cost)/\varepsilon)}{\alpha\otimes \nu}
     -  \varepsilon\scal{\exp( (u \oplus v - \cost)/\varepsilon)}{\mu\otimes \beta}\\[0.8ex]
    &\qquad - t\varepsilon\scal{\exp( (u \oplus v - \cost)/\varepsilon)}{\mu\otimes \nu},
\end{align*}
hence
\begin{multline*}
    [g((u,v),\cdot)]^\prime((\alpha,\beta);(\mu,\nu)) \\[0.8ex]
    = \scal{u}{\mu} + \scal{v}{\nu} 
    -  \varepsilon\scal{\exp( (u \oplus v - \cost)/\varepsilon)}{\alpha\otimes \nu}
     -  \varepsilon\scal{\exp( (u \oplus v - \cost)/\varepsilon)}{\mu\otimes \beta}
\end{multline*}
and \cref{eq:20190511e} follows. Thus,
the function $g$ is G\^ateaux differentiable
with respect to the second variable, with derivative
\begin{align*}
    D_{2}g ((u,v),(\alpha,\beta))
    &= \Big(u - \varepsilon e^{\frac{u}{\varepsilon}} \int_{\X} e^{\frac{v - c(\cdot,y)}{\varepsilon}} d\beta(y), 
    v - \varepsilon e^{\frac{v}{\varepsilon}} \int_{\X} e^{\frac{u - c(x,\cdot)}{\varepsilon}} d\alpha(x)\Big)\\
    &= (u,v) - \varepsilon (e^{\frac{u}{\varepsilon}} \lmap_{\beta} e^{\frac{v}{\varepsilon}}, e^{\frac{v}{\varepsilon}} \lmap_{\alpha} e^{\frac{u}{\varepsilon}}) \in \cont(\X)^2,
\end{align*}
which is jointly continuous, since
the maps $(u,\alpha)\mapsto \lmap_{\alpha} e^{u/\varepsilon}$ and
$(v,\beta)\mapsto \lmap_{\beta} e^{v/\varepsilon}$
are continuous.
Moreover, it follows from \cref{cor:Lipschitz2} that there exists a continuous selection of Sinkhorn potentials.
Therefore, it follows from \cref{p:diff_of_max} that
$\oteps$ is directionally differentiable and 
\begin{equation}
    \oteps^\prime((\alpha,\beta);(\mu,\nu))
    = \max_{(u,v)\text{ solution of } \cref{eq:20190512a}}
    \scal{D_{2}g ((u,v),(\alpha,\beta))}{(\mu,\nu)}.
\end{equation}
However, if $(u,v)$ is a solution of  \cref{eq:20190512a}, it follows from the optimality conditions \cref{eq:Dopt} that
\begin{equation}
    e^{\frac{u}{\varepsilon}} \int_{\X} e^{\frac{v - c(\cdot,y)}{\varepsilon}} d\beta(y) = 1
    \quad\text{and}\quad
    e^{\frac{v}{\varepsilon}} \int_{\X} e^{\frac{u - c(x,\cdot)}{\varepsilon}} d\alpha(x) = 1,
\end{equation}
hence
\begin{equation}
\scal{D_{2}g ((u,v),(\alpha,\beta))}{(\mu,\nu)}
= \scal{(u - \varepsilon, v - \varepsilon)}{(\mu, \nu)} = \scal{(u , v )}{(\mu, \nu)},
\end{equation}
where we used the fact that, since $(\mu,\nu) = t(\mu_1 - \mu_2, \nu_1- \nu_2)$
for some $t>0$ and $\mu_1,\mu_2,\nu_1,\nu_2 \in \prm(\X)$, we have $\scal{1}{\mu} = t\scal{1}{\mu_1 - \mu_2} = 0$ and $\scal{1}{\nu}=t \scal{1}{\nu_1 - \nu_2} =0$.
\end{proof}

\section{The Frank-Wolfe algorithm for Sinkhorn barycenters }\label{sec:app-frank-wolfe-algorithm}
In this section we finally analyze the Frank-Wolfe algorithm for the Sinkhorn barycenters and give convergence results.
The following result is a direct consequence of \cref{thm:Sinkhornalgo} and \cref{rem:directional-derivatives-oteps}.

\begin{theorem}
\label{thm:Sinkhorn2}
Let $(\tilde{u}^{(\Siter)})_{\Siter \in \N}$ be generated according to 
\cref{alg:Sinkhorn_cont}.
Then, 
\begin{equation}
\label{eq:Sinkhorn2}
(\forall\, \Siter \in \N)\quad\lVert \tilde{u}^{(\Siter)} - \nabla_1 \oteps(\alpha,\beta)\rVert_{\infty} \leq 
\lambda^{2\Siter}
\bigg(\frac{\diam + \max\nolimits_{\X} u^{(0)} - \min\nolimits_{\X} u^{(0)}}{\varepsilon}\bigg),
\end{equation}
where $u^{(\Siter)} = \varepsilon \log f^{(\Siter)}$ and 
$\tilde{u}^{(\Siter)} = u^{(\Siter)} - u^{(\Siter)}(x_o)$.
\end{theorem}

Therefore, in view of \cref{f:Sinkhorn}, \cref{thm:Sinkhorn2}, and \cref{p:inexactgrad},
we can address the problem of the Sinkhorn barycenter \cref{eq:sinkhorn-barycenter} via the Frank-Wolfe \cref{alg:abstract-FW}. Note that,
according to \cref{p:inexactgrad}\ref{p:inexactgrad_ii}, since the diameter of $\prob(\X)$ with respect to $\nor{\cdot}_{TV}$ is $2$,  have that the curvature of $\bary$ is upper bounded by
\begin{equation}
\label{eq:OTcurvature}
    C_{\bary} \leq 24\varepsilon e^{3\diameps}.
\end{equation}
Let $k \in \N$ 
and $\alpha_k$ be the current iteration.
For every $j \in \{1,\dots, m\}$, 
we can compute $\nabla_1 \oteps(\alpha_k,\beta_j)$ and 
$\nabla_1 \oteps(\alpha_k,\alpha_k)$ by the Sinkhorn-Knopp algorithm. 
Thus, by \cref{eq:Sinkhorn2}, we find $\Siter \in \N$ large enough so that $\lVert \tilde{u}_j^{(\Siter)} - \nabla_1 \oteps(\alpha_k,\beta_q) \rVert_\infty \leq  \precision_{1,k}/8$ and $\lVert \tilde{p}^{(\Siter)} - \nabla_1 \oteps(\alpha_k,\alpha_k) \rVert_\infty \leq  \precision_{1,k}/8$ and we set
\begin{equation}
\tilde{u}^{(\Siter)}:= \sum_{j=1}^m \omega_j 
    \tilde{u}_j^{(\Siter)} - \tilde{p}^{(\Siter)}.
\end{equation}
Then,
\begin{equation}
    \lVert \tilde{u}^{(\Siter)}
     - \nabla \bary(\alpha_k)\rVert_\infty \leq \frac{\precision_{1,k}}{4}.
\end{equation}
Now, Frank-Wolf \cref{alg:abstract-FW} (in the version 
considered in \cref{p:inexactgrad}\ref{p:inexactgrad_i}) requires finding
\begin{equation}
\label{eq:20190508a}
    \eta_{k+1} \in \argmin_{\eta \in \prob(\X)} \langle \tilde{u}^{(\Siter)}, \eta - \alpha_k \rangle
\end{equation}
and make the update
\begin{equation}
    \alpha_{k+1} = (1 - \gamma_k) \alpha_k + \gamma_k \eta_{k+1}.
\end{equation}
Since the solution of \cref{eq:20190508a} is a Dirac measure (see \cref{sec:algorithm-practice} in the paper),
the algorithm reduces to
\begin{equation}
    \begin{cases}
    \text{find } x_{k+1} \in \X \text{ such that }
    \tilde{u}^{(\Siter)}(x_{k+1}) \leq
    \min_{x \in \X} \tilde{u}^{(\Siter)}(x) + \dfrac{\precision_{2,k}}{2} \\[1ex]
    \alpha_{k+1} = (1 - \gamma_k) \alpha_k + \gamma_k \delta_{x_{k+1}}.
    \end{cases}
\end{equation}
So, if we initialize the algorithm with $\alpha_0 = \delta_{x_0}$, then any $\alpha_k$ will be a discrete probability measure with support contained in $\{x_0, \dots, x_k\}$.  This implies that
if all the $\beta_j$'s are probability measures with finite support, the computation of
 $\nabla_1 \oteps(\alpha_k,\beta_j)$ by the Sinkhorn algorithm can be reduced to a fully discrete algorithm, as showed in \cref{rmk:discrete-sinkhorn}. More precisely,
 assume that
 \begin{equation}
(\forall\, j=1,\dots, m)\quad
\beta_j = \sum_{i_2=0}^n b_{j,i_2} \delta_{y_{j,i_2}}.
 \end{equation}
 and that at iteration $k$ we have
 \begin{equation}
 \alpha_k = \sum_{i_1=0}^k a_{k,i_1} \delta_{x_{i_1}}.
 \end{equation}
 Set 
 \begin{equation}
\mathsf{a}_k =
\begin{bmatrix}
a_{k,0} \\
\vdots\\
a_{k,k}
\end{bmatrix} \in \R^{k+1},
\ \ \mathsf{M}_0 = 
\begin{bmatrix}
a_{k,0}\kerfun(x_{0},x_{0}) a_{k,0} & \dots & a_{k,0}\kerfun(x_{0},x_{k}) a_{k,k}\\
\vdots & \ddots & \vdots\\
a_{k,k}\kerfun(x_{k},x_{0}) a_{k,0} & \dots & a_{k,k}\kerfun(x_{k},x_{k}) a_{k,k}
\end{bmatrix} \in \R^{(k+1)\times(k+1)}
\end{equation}
and, for every $j=1\,\dots, m$, 
\begin{equation}
\mathsf{b}_j =
\begin{bmatrix}
b_{j,0} \\
\vdots\\
b_{j,n}
\end{bmatrix} \in \R^{n+1},
\ \ 
\mathsf{M}_j = 
\begin{bmatrix}
a_{k,0}\kerfun(x_{0},y_{j,0}) b_{j,0} & \dots & a_{k,0}\kerfun(x_{0},y_{j,n}) b_{j,n}\\
\vdots & \ddots & \vdots\\
a_{k,k}\kerfun(x_{k},y_{j,0}) b_{j,0} & \dots & a_{k,n}\kerfun(x_{k},y_{j,n}) b_{j,n}
\end{bmatrix} \in \R^{(k+1)\times(n+1)}.
\end{equation}
Then, run \cref{algo:sinkalgo_disc},
with input $\mathsf{a}_k$, $\mathsf{a}_k$,
and $\mathsf{M}_0$ 
to get $(\mathsf{e}^{(\Siter)},\mathsf{h}^{(\Siter)})$,
and, for every $j=1,\dots, m$,
with input $\mathsf{a}_k$, $\mathsf{b}_j$,
and $\mathsf{M}_j$ to get $(\mathsf{f}^{(\Siter)}_j,\mathsf{g}^{(\Siter)}_j)$.
So, we have, 
 \begin{equation}
(\forall\, \Siter \in \N) \quad
\begin{cases}
\phantom{(\forall\,j=1,\dots,m)\ \ }
\mathsf{h}^{(\Siter+1)}= \dfrac{\mathsf{a}_k}{\mathsf{M}_0^\top \mathsf{e}^{(\Siter)}},
\quad
\mathsf{e}^{(\Siter+1)}= \dfrac{\mathsf{a}_k}{\mathsf{M}_0 \mathsf{h}^{(\Siter+1)}}\\[2ex]
(\forall\,j=1,\dots,m)\ \  \mathsf{g}_j^{(\Siter+1)}= \dfrac{\mathsf{b}_j}{\mathsf{M}_j^\top \mathsf{f}_j^{(\Siter)}},
\quad
\mathsf{f}_j^{(\Siter+1)}= \dfrac{\mathsf{a}_k}{\mathsf{M}_j \mathsf{g}_j^{(\Siter+1)}}.
\end{cases}
 \end{equation}
Then, according to \cref{rmk:discrete-sinkhorn}, for every $\Siter \in \N$, we have
\begin{equation}
 \label{eq:20190523b}
(\forall\, x \in \X)\ 
\begin{cases}
\displaystyle e^{(\Siter)}(x)^{-1} = \sum_{i_2=0}^k \kerfun(x,x_{i_2}) \mathsf{h}^{(\Siter-1)}_{i_2} a_{k,i_2},\\[1ex]
\displaystyle  p^{(\Siter)}(x) = \varepsilon \log e^{(\Siter)}(x) = - \varepsilon \log \sum_{i_2=0}^k \kerfun(x,x_{i_2}) \mathsf{h}^{(\Siter-1)}_{i_2} a_{k,i_2}\\[1ex]
  \tilde{p}^{(\Siter)}(x) = p^{(\Siter)}(x) - p^{(\Siter)}(x_o).
  \end{cases}
 \end{equation}
 and, for every $j=1,\dots, m$,
 \begin{equation}
 \label{eq:20190523a}
(\forall\, x \in \X)\ 
\begin{cases}
\displaystyle f_j^{(\Siter)}(x)^{-1} = \sum_{i_2=0}^n \kerfun(x,y_{i_2}) \mathsf{g}^{(\Siter-1)}_{j,i_2} b_{j,i_2},\\[1ex]
\displaystyle  u_j^{(\Siter)}(x) = \varepsilon \log f_j^{(\Siter)}(x) = - \varepsilon \log \sum_{i_2=0}^n \kerfun(x,y_{i_2}) \mathsf{g}^{(\Siter-1)}_{j,i_2} b_{j,i_2}\\[1ex]
  \tilde{u}_j^{(\Siter)}(x) = u_j^{(\Siter)}(x) - u_j^{(\Siter)}(x_o).
  \end{cases}
 \end{equation}

Since the 
$\tilde{u}_j^{(\Siter)}$'s and $u_j^{(\Siter)}$'s,
and $\tilde{p}^{(\Siter)}$ and $p^{(\Siter)}$,
differ for a constant only, the final algorithm can be written as in \cref{algo:FW-Baricenters}.
We stress that this algorithm is even more general than
\cref{alg:practical-FW} since, in the computation of the Sinkhorn potentials and in their minimization, errors have been taken into account.

\begin{algorithm}
\caption{Frank-Wolfe algorithm for Sinkhorn barycenter}
\label{algo:FW-Baricenters}
Let $\alpha_0 = \delta_{x_0}$ for some $x_0 \in \X$.
Let $(\precision_k)_{k \in \N} \in \R_{+}^\N$ be such that
$\Delta_k/\gamma_k$ is nondecreasing. Define
\begin{equation*}
\begin{array}{l}
\text{for}\;k=0,1,\dots\\[0.7ex]
\left\lfloor
\begin{array}{l}
\text{run \cref{algo:sinkalgo_disc} 
with input $\mathsf{a}_k, \mathsf{a}_k, \mathsf{M}_0$  till } \lambda^{2\Siter} \diameps \leq \frac{\precision_{1,k}}{8} \rightarrow  \mathsf{h} \in \R^{k+1}  \\[0.7ex]
\text{compute } p \text{ via \cref{eq:20190523b}} \text{ with } \mathsf{h}\\[0.7ex]
\text{for}\;j=1,\dots m\\[0.7ex]
\left\lfloor
\begin{array}{l}
\text{run \cref{algo:sinkalgo_disc} 
with input $\mathsf{a}_k, \mathsf{b}_j, \mathsf{M}_j$  till } \lambda^{2\Siter} \diameps \leq \frac{\precision_{1,k}}{8} \rightarrow  \mathsf{g}_j \in \R^{n+1}  \\[0.7ex]
\text{compute } u_j \text{ via \cref{eq:20190523a}} \text{ with } \mathsf{g}_j \\[0.7ex]
\end{array}
\right.\\[1ex]
\text{set } u
= \sum_{j=1}^m \omega_j u_j - p\\[1ex]
\text{find } x_{k+1} \in \X \text{ such that }
    u(x_{k+1}) \leq
    \min_{x \in \X} u(x) + \dfrac{\precision_{2,k}}{2}\\[1ex]
\alpha_{k+1} = (1 - \gamma_k) \alpha_k + \gamma_k \delta_{x_{k+1}}.
\end{array}
\right.
\end{array}
\end{equation*}
\end{algorithm}

We now give a final converge theorem, of which 
\cref{thm:sinkhorn-barycenters-finite-case} in the paper is a special case.
\begin{theorem}\label{thm:full_convergence_FW_with_error}
Suppose that $\beta_1, \dots, \beta_m \in \prm(\X)$ are probability measures with finite support, each of cardinality $n \in \N$. Let $(\alpha_k)_{k \in \N}$ be generated by \cref{algo:FW-Baricenters}. Then, for every $k \in \N$,
\begin{equation}
\bary(\alpha_k) - \min_{\alpha \in \prm(\X)} \bary(\alpha) \leq \gamma_k 
24\varepsilon e^{3\diameps} + 2 \Delta_{1,k} + \Delta_{2,k}
\end{equation}
\end{theorem}
\begin{proof}
It follows from \cref{thm:FWA}, \cref{p:inexactgrad}, and \cref{eq:OTcurvature},
recalling that $\Diam(\prm(\X))=2$.
\end{proof}


\section{Sample complexity of Sinkhorn potential}\label{sec:sample-complexity-sinkhorn-potentials}

In the following we will denote by $\cont^s(\X)$ the space of $s$-differentiable functions with continuous derivatives and by $W^{s,p}(\X)$ the Sobolev space of functions $f\colon \X \to \R$ with $p$-summable weak derivatives up to order $s$ \cite{adams2003sobolev}. We denote by $\nor{\cdot}_{s,p}$ the corresponding norm. 

The following result shows that under suitable smoothness assumptions on the cost function $\dist$, the Sinkhorn potentials are uniformly bounded as functions in a suitable Sobolev space of corresponding smoothness. This fact will play a key role in approximating the Sinkhorn potentials of general distributions in practice. 

\begin{theorem}[Proposition 2 in \cite{genevay2018sample}]\label{thm:sinkhorn-potentials-uniformly-bounded}
Let $\X$ be a closed bounded domain with Lipschitz boundary
 in $\R^d$ (\cite[Definition 4.9]{adams2003sobolev}) and let $\dist \in \cont^{s+1}(\X\times\X)$. Then for every $(\alpha,\beta)\in\prob(\X)^2$, the associated Sinkhorn potentials $(u,v)\in\cont(\X)^2$ are functions in $W^{s,\infty}(\X)$. Moreover, let $x_o\in\X$. Then there exists a constant $\sinkconst>0$, depending only on $\varepsilon,s$ and $\X$, such that for every $(\alpha,\beta)\in\prob(\X)^2$ the associated Sinkhorn potentials $(u,v)\in\cont(\X)^2$ with $u(x_o) = 0$ satisfies $\nor{u}_{s,\infty},\nor{v}_{s,\infty}\leq\sinkconst$. 
\end{theorem}

In the original statement of \cite[Proposition 2]{genevay2018sample} the above result is formulated for $\dist\in\cont^{\infty}(\X)$ for simplicity. However, as clarified by the authors, it holds also for the more general case $\dist\in\cont^{s+1}(\X)$.

\begin{lemma}\label{lem:s-infty-norm-of-products-and-exponentials}
Let $\X\subset\R^d$ be a closed bounded domain with Lipschitz boundary 
and let $u,u'\in W^{s,\infty}(\X)$. Then the following holds
\begin{enumerate}[{\rm(i)}]
    \item \label{item:timesconst} $\nor{uu'}_{s,\infty} \leq \timesconst \nor{u}_{s,\infty}\nor{u'}_{s,\infty}$,
    \item \label{item:expconst} $\nor{e^{u}}_{s,\infty} 
    \leq \nor{e^{u}}_{\infty} (1 + \expconst \nor{u}_{s,\infty})$,
\end{enumerate}
where $\timesconst = \timesconst(s,d)$ and $\expconst = \expconst(s,d)>0$ depend only
 on the dimension $d$ and the order of differentiability $s$
but not on $u$ and $u'$.
\end{lemma}
\begin{proof}
\cref{item:timesconst} follows directly from Leibniz formula. To see \cref{item:expconst}, let $\mathbf{i} = (i_1,\dots,i_d)\in\N^d$ be a multi-index with $|\mathbf{i}| = \sum_{\ell=1}^d i_\ell \leq s$ and note that by chain rule the derivatives of $e^{u}$ 
\eqals{
    D^\mathbf{i} ~e^{u} = e^{u}~ P_\mathbf{i}\Big((D^\mathbf{j} u)_{\mathbf{j} \leq \mathbf{i}}\Big),
}
where $P_\mathbf{i}$ is a polynomial of degree $|\mathbf{i}|$ and $\mathbf{j} \leq \mathbf{i}$ is the ordering associated to the cone of non-negative vectors in $\R^d$. Note that $P_0 = 1$, while for $|\mathbf{i}|>0$, the associated polyomial $P_\mathbf{i}$ has a root in zero (i.e. it does not have constant term). Hence 
\eqals{
    \nor{e^{u}}_{s,\infty} & \leq \nor{e^{u}}_\infty \left(1 + |P|\Big((\nor{D^\mathbf{i} u}_\infty)_{|\mathbf{i}|\leq s}\Big)~\right),
}
where we have denoted by $P = \sum_{0<|\mathbf{i}|\leq s} P_\mathbf{i}$ and by $|P|$ the polynomial with coefficients corresponding to the absolute value of the coefficients of $P$. Therefore, since $\nor{D^\mathbf{i} u}_{\infty}\leq \nor{u}_{s,\infty}$ for any $|\mathbf{i}|\leq s$, by taking 
\eqals{
    \expconst = |P|\Big((1)_{|\mathbf{i}|\leq s}\Big),
}
namely the sum of all the coefficients of $|P|$, we obtain the desired result. Indeed note that the coefficients of $P$ do not depend on $u$ but only on the smoothness $s$ and dimension $d$. 
\end{proof}

\begin{lemma}\label{lem:uniform-bound-products-exponential-sinkhorn-potentials}
Let $\X\subset\R^d$ be a closed bounded domain with Lipschitz boundary and let $ x_o\in\X$. Let $\cost\in\cont^{s+1}(\X\times\X)$, 
for some $s\in\N$. Then for any $\alpha,\beta\in\prob(\X)$ and corresponding pair of Sinkhorn potentials $(u,v)\in\cont(\X)^2$ with $u(x_o) = 0$, the functions $\kerfun(x,\cdot)e^{u/\varepsilon}$ and $\kerfun(x,\cdot)e^{v/\varepsilon}$ belong to $W^{s,2}(\X)$ for every $x\in\X$. Moreover, they admit an extension to $\hh = W^{s,2}(\R^d)$ and  there exists a constant $\bar\sinkconst$ independent on $\alpha$ and $\beta$, such that for every $x\in\X$
\begin{equation}\label{eq:uniform-bound-products-exponential-sinkhorn-potentials}
    \big\lVert \kerfun(x,\cdot) e^{u/\varepsilon}\big\rVert_{\hh},~ \big\lVert \kerfun(x,\cdot) e^{v/\varepsilon} \big\rVert_{\hh} \leq \bar\sinkconst
\end{equation}
(with some abuse of notation, we have identified $\kerfun(x,\cdot) e^{u/\varepsilon}$ and $\kerfun(x,\cdot)e^{v/\varepsilon}$ with their extensions to $\R^d$).
\end{lemma}

\begin{proof}
In the following we denote by  $\nor{\cdot}_{s,2} = \nor{\cdot}_{s,2,\X}$ the norm of $W^{s,2}(\X)$ and by $\nor{\cdot}_\hh=\nor{\cdot}_{s,2,\R^d}$  the norm of $\hh = W^{s,2}(\R)$. Let $x\in \X$.
Then, since $u - \cost(x,\cdot) \in W^{s,\infty}(\X)$ and $\nor{u}_{s,\infty} \leq \sinkconst$,
it follows from \cref{lem:s-infty-norm-of-products-and-exponentials} that
\begin{align*}
    \big\lVert \kerfun(x,\cdot) e^{u/\varepsilon} \big\rVert_{s,\infty} & = 
    \big\lVert e^{(u - \cost(x,\dot))/\varepsilon} \big\rVert_{s,\infty}\\
    & \leq     \big\lVert e^{(u - \cost(x,\dot))/\varepsilon} \big\rVert_{\infty}
    (1 + \expconst \nor{u - \cost(x,\cdot)}_{s,\infty})\\
    & =     \big\lVert \kerfun(x,\cdot) e^{u/\varepsilon}  \big\rVert_{\infty}
    (1 + \expconst \nor{u - \cost(x,\cdot)}_{s,\infty})\\
    & \leq  \big\lVert e^{u/\varepsilon}\big\rVert_{\infty}
    (1 + \expconst (\sinkconst+ \nor{\cost}_{s,\infty}))\\
    & \leq e^{\diameps}(1 + \expconst (\sinkconst+ \nor{\cost}_{s,\infty})),
\end{align*}
where we used the fact that $D^\mathbf{i} [\cost(x,\cdot)] = (D^\mathbf{i} \cost )(x,\cdot)$.
This implies
\begin{equation*}
    \big\lVert\kerfun(x,\cdot) e^{u/\varepsilon}\big\rVert_{s,2} \leq 
    |\X|^{1/2}e^{\diameps}(1 + \expconst (\sinkconst+ \nor{\cost}_{s,\infty}))
\end{equation*}
where $|\X|$ is the Lebesgue measure of $\X$. 
Now, we can proceed analogously to \cite[Proposition 2]{genevay2018sample}, and use  Stein's Extension Theorem \cite[Theorem 5.24]{adams2003sobolev},\cite[Chapter 6]{stein2016singular}, to guarantee the existence of a {\em total extension operator} \cite[Definition 5.17]{adams2003sobolev}. In particular, there exists a constant $\extensionconst = \extensionconst(s,2,\X)$ such that for any $\varphi\in W^{s,2}(\X)$ there exists $\tilde \varphi \in W^{s,2}(\R^d)$ such that 
\begin{equation}
    \nor{\tilde \varphi}_{\hh} = \nor{\tilde \varphi}_{s,2,\R^d}  \leq \extensionconst\nor{\varphi}_{s,2,\X} = \extensionconst\nor{\varphi}_{s,2}.
\end{equation}
Therefore, we conclude
\begin{equation}
    \big\lVert\kerfun(x,\cdot) e^{u/\varepsilon}\big\rVert_{\hh} \leq \extensionconst |\X|^{1/2}e^{\diameps}(1 + \expconst (\sinkconst+ \nor{\cost}_{s,\infty})) =: \bar\sinkconst.
\end{equation}
The same argument applies to $\kerfun(x,\cdot) e^{v/\varepsilon}$ with the only exception that now, in virtue of \cref{cor:dad-solutions-bounded},
we have $\lVert e^{v/\varepsilon}\rVert_{\infty} \leq e^{2\diameps}$. Note that $\bar\sinkconst$ is a constant depending only on $\X$, $\cost$, $s$ and $d$ but it is independent on the probability distributions $\alpha$ and $\beta$.
\end{proof}

\paragraph{Sobolev spaces and reproducing kernel Hilbert spaces} 
Recall that for $s>d/2$ the space $\hh = W^{s,2}(\R^d)$, is a reproducing kernel Hilbert space (RKHS) \cite[Chapter 10]{wendland2004scattered}. In this setting we denote by $\kersob:\X\times\X\to\R$ the associated reproducing kernel,
which is continuous and bounded and satisfies the reproducing property
\begin{equation}
    (\forall\,x\in\X)(\forall\,f\in\hh)\qquad \scal{f}{\kersob(x,\cdot)}_{\hh} = f(x).
\end{equation}
We can also assume that $\kersob$ is {\em normalized}, namely,  $\nor{\kersob(x,\cdot)}_\hh=1$ for all $x\in\X$ \cite[Chapter 10]{wendland2004scattered}.

\paragraph{Kernel mean embeddings}
For every $\beta\in\prob(\X)$, we denote by $\kersob_\beta\in\hh$ the {\em Kernel Mean Embedding} of $\beta$ in $\hh$ \cite{smola2007hilbert,muandet2017kernel}, that is, the vector
\begin{equation}
    \kersob_\beta = \int \kersob(x,\cdot) ~d\beta(x).
\end{equation}
In other words, the kernel mean embedding of a distribution $\beta$ corresponds to the expectation of $\kersob(x,\cdot)$ with respect to $\beta$. By the linearity of the inner product and the integral, for every $f\in\hh$, the inner product
\begin{equation}
    \scal{f}{\kersob_\beta}_\hh = \int \scal{f}{\kersob(x,\cdot)}~d\beta(x) = \int f(x)~d\beta(x), 
\end{equation}
corresponds to the expectation of $f(x)$ with respect to $\beta$.
The {\em Maximum Mean Discrepancy (MMD)} \cite{song2008learning,sriperumbudur2011universality,muandet2017kernel} between two probability distributions $\beta,\beta^\prime\in\prob(\X)$ is defined as
\begin{equation}
    \mmd(\beta,\beta^\prime) = \nor{\kersob_\beta - \kersob_{\beta^\prime}}_\hh.
\end{equation}
In the case of the Sobolev space $\hh = W^{s,2}(\R^d)$, the MMD  metrizes the weak-$*$ topology of $\prob(\X)$ \cite{sriperumbudur2010hilbert,sriperumbudur2011universality}. 

A well-established approach to approximate a distribution $\beta\in\prob(\X)$ is to independently sample a set of points $x_1,\dots,x_n\in\X$ from $\beta$ and consider the empirical distribution $\beta_n = \frac{1}{n}\sum_{i=1}^n \delta_{x_i}$. The following result shows that $\beta_n$ converges to $\beta$ in MMD with high probability. The original version of this result can be found in \cite{song2008learning}, we report an independent proof for completeness. 

\begin{lemma}\label{lem:mmd-concentration-inequality}
Let $\beta\in\prob(\X)$. Let $x_1,\dots,x_n\in\X$ be indepedently sampled according to $\beta$ and denote by $\beta_n = \frac{1}{n} \sum_{i=1}^n \delta_{x_i}$. Then, for any $\tau\in(0,1]$, we have
\begin{equation}
    \mmd(\beta_n,\beta) \leq \frac{4\log \frac{3}{\tau}}{\sqrt{n}}
\end{equation}
with probability at least $1-\tau$. 
\end{lemma}

\begin{proof}
The proof follows by applying  Pinelis' inequality \cite{yurinskiui1976exponential,pinelis1994optimum,smale2007learning} for random vectors in Hilbert spaces. More precisely, for $i=1,\dots,n$, denote by $\zeta_i = \kersob(x_i,\cdot)\in\hh$ and recall that $\nor{\zeta_i} = \nor{\kersob(x,\cdot)}=1$ for all $x\in\X$. We can therefore apply \cite[Lemma 2]{smale2007learning} with constants $\widetilde{M}=1$ and $\sigma^2 =\sup_i \mathbb{E} \|\zeta_i\|^2 \leq 1$, which guarantees that, for every $\tau\in(0,1]$
\begin{equation}\label{eq:pinelis-general-form}
    \nor{\frac{1}{n}\sum_{i=1}^n \Big[\zeta_i - \mathbb{E}~\zeta_i\Big]}_\hh \leq \frac{2\log\frac{2}{\tau}}{n} + \sqrt{\frac{2\log\frac{2}{\tau}}{n}} \leq \frac{4\log \frac{3}{\tau}}{\sqrt{n}},
\end{equation}
holds with probability at least $1-\tau$. Here, for the second inequality we have used the fact that $\log\frac{2}{\tau}\leq\log\frac{3}{\tau}$ and $\log\frac{3}{\tau}\geq1$ for every $\tau\in(0,1]$. The desired result follows by observing that
\begin{equation}
    \kersob_\beta = \int \kersob(x,\cdot)~d\beta(x) = \mathbb{E}~\zeta_i
\end{equation}
for all $i=1,\dots,n$, and 
\begin{equation}
    \kersob_{\beta} = \frac{1}{m}\sum_{i=1}^m \kersob(x_i,\cdot) = \frac{1}{m}\sum_{i=1}^m \zeta_i.
\end{equation}
Therefore, 
\begin{equation}
    \mmd(\beta_k,\beta) = \nor{\kersob_{\beta_k} - \kersob_\beta}_\hh 
     = \nor{\frac{1}{n}\sum_{i=1}^n \Big[\zeta_i - \mathbb{E}~\zeta_i\Big]}_\hh,
\end{equation}
which combined with \cref{eq:pinelis-general-form} leads to the desired result. 
\end{proof}

\begin{proposition}[Lipschitz continuity of the Sinkhorn Potentials with respect to the MMD]\label{prop:lipschitz-continuity-mmd}
Let $\X\subset\R^d$ be a compact Lipschitz domain and $\dist\in\cont^{s+1}(\X\times\X)$, with $s>d/2$. 
Let $\alpha,\beta,\alpha',\beta'\in\prob(\X)$.
Let $x_o\in\X$ and let $(u,v),(u^\prime,v')\in\cont(\X)^2$ be the two Sinkhorn potentials corresponding to the solution of the regularized OT problem in \cref{eq:ot-dual-problem} for $(\alpha,\beta)$ and $(\alpha',\beta')$ respectively such that $u(x_o) = u^\prime(x_o) = 0$. Then
\begin{equation}
\label{eq:lipschitz-continuity-mmd}
    \supnor{u - u^\prime} \leq 2\varepsilon\bar\sinkconst e^{3\diameps}\left(\mmd(\alpha,\alpha') + \mmd(\beta,\beta')\right),
\end{equation}
with $\bar\sinkconst$ from \cref{lem:uniform-bound-products-exponential-sinkhorn-potentials}.
In other words, the operator $\nabla_1\oteps\colon\prob(\X)^2\to\cont(\X)$, defined in \cref{rem:directional-derivatives-oteps},
 is $2\varepsilon\bar\sinkconst e^{3\diameps}$-Lipschitz continuous with respect to the \textnormal{MMD}.
\end{proposition}
\begin{proof}
Let $f= e^{u/\varepsilon}$ and $g = e^{v/\varepsilon}$. By relying on \cref{lem:uniform-bound-products-exponential-sinkhorn-potentials} we can now refine the analysis in \cref{prop:lipschitz-continuity-total-variation2}. More precisely, we observe that in \cref{eq:pairing-for-TV-lipschitz} we have
\begin{equation*}
\begin{aligned}{}
    [(\lmap_{\beta^\prime}-\lmap_\beta) g](x)  & = \int \kerfun(x,z) g(z)~d(\beta-\beta^\prime)(z) \\
    & = \int \scal{\kerfun(x,\cdot)g}{\kersob(z,\cdot)}_\hh~d(\beta-\beta')(z) \\
    & = \scal{\kerfun(x,\cdot) g}{\kersob_\beta-\kersob_{\beta'}}_{\hh}\\
    & \leq \nor{\kerfun(x,\cdot) g}_\hh~ \nor{\kersob_\beta-\kersob_{\beta'}}_{\hh} \\
    & \leq \bar\sinkconst~ \mmd(\beta,\beta'),
\end{aligned}
\end{equation*}
where in the first equality, with some abuse of notation, we have implicitly considered the extension of $\kerfun(x,\cdot) g$ to $\hh = W^{s,2}(\R^d)$ as discussed in \cref{lem:uniform-bound-products-exponential-sinkhorn-potentials}. The rest of the analysis in \cref{prop:lipschitz-continuity-total-variation2} remains invaried, eventually leading to \cref{eq:lipschitz-continuity-mmd}.
\end{proof}

It is now clear that \cref{thm:sample-complexity-sinkhorn-gradients} in the paper is just a consequence of
\cref{lem:mmd-concentration-inequality} and
 \cref{prop:lipschitz-continuity-mmd}.
 We give the statement of the theorem for reader's convenience.
 
\TLSampleComplexitySinkhornPotentials*


We finally provide the proof of \cref{thm:sinkhorn-barycenters-infinite-case} in the paper.
\TLSInkhornBarycenterInfiniteDim*
\begin{proof}
Let $\widehat\bary(\alpha) = \sum_{j=1}^m \omega_j \sink(\alpha,\hat\beta_j)$. Then,
it follows from the definition of $\bary$ and \cref{thm:sample-complexity-sinkhorn-gradients}
that, for every $k \in \N$, and with probability larger than $1-\tau$, we have
\begin{align*}
    \lVert \nabla \widehat\bary(\alpha_k) - \nabla \bary(\alpha_k) \rVert_{\infty}
    &\leq \sum_{j=1}^m \omega_j 
    \lVert \nabla 
    [\sink(\cdot, \hat{\beta}_j)](\alpha_k) -
    \sink(\cdot, \beta_j)](\alpha_k) 
    \rVert_{\infty}\\
    & = \sum_{j=1}^m \omega_j 
    \lVert \nabla_1 \oteps(\alpha_k, \hat\beta_j)- \nabla_1\oteps(\alpha_k,\beta_j)
    \rVert_{\infty}\\
    &\leq 
    \frac{8\varepsilon~\overline\sinkconst e^{3\diameps}\log\frac{3}{\tau}}{\sqrt{n}}\\
    &= \frac{\Delta_{1}}{4},
\end{align*}
where
\begin{equation*}
    \Delta_1 := \frac{32\varepsilon~\overline\sinkconst e^{3\diameps}\log\frac{3}{\tau}}{\sqrt{n}}.
\end{equation*}
Now, let $\gamma_k = 2/(k+2)$. Since
 \cref{alg:practical-FW} is applied to 
$\hat{\beta}_1,\dots \hat{\beta}_m$, we have
\begin{equation*}
    \delta_{x_{k+1}} \in \argmin_{\prm(\X)} \langle \nabla \widehat\bary(\alpha_k),\cdot \rangle\quad\text{and}\quad
    \alpha_{k+1} = (1- \gamma_k) \alpha_k + \gamma_k \delta_{x_{k+1}}.
\end{equation*}
Therefore, it follows from \cref{thm:FWA}, \cref{p:inexactgrad},
and \cref{thm:lip-continuity-total-variation-informal} that,
with probability larger than $1-\tau$, we have
\begin{equation*}
    \bary(\alpha_k) - \min_{\prm(\X)} \bary \leq 6 \varepsilon \bar{\sinkconst} e^{3\diameps}\Diam(\prm(\X))^2 \gamma_k  + \Delta_1 \Diam(\prm(\X)).
\end{equation*}
The statement follows by noting that $\Diam(\prm(\X))=2$.
\end{proof}

\newpage

\section{Additional experiments}\label{sec:additional_exp}
\paragraph{Sampling of continuous measures:  mixture of Gaussians}
We perform the barycenter of 5 mixtures of two Gaussians $\mu_j$, centered at $(j/2, 1/2)$ and $(j/2, 3/2)$ for $j-0,\dots,4$ respectively. Samples are provided in \cref{fig:input_mixture_gauss}. We use different relative weights pairs in the mixture of Gaussians, namely $(1/10,9/10), (1/4,3/4), (1/2,1/2)$. At each iteration, a sample of $n=500$ points is drawn from $\mu_j$, $j=0\dots,4$. Results are reported in \cref{fig:bary_mixture_gauss}.
\begin{figure}[ht]

\centering
  \includegraphics[scale=0.28]{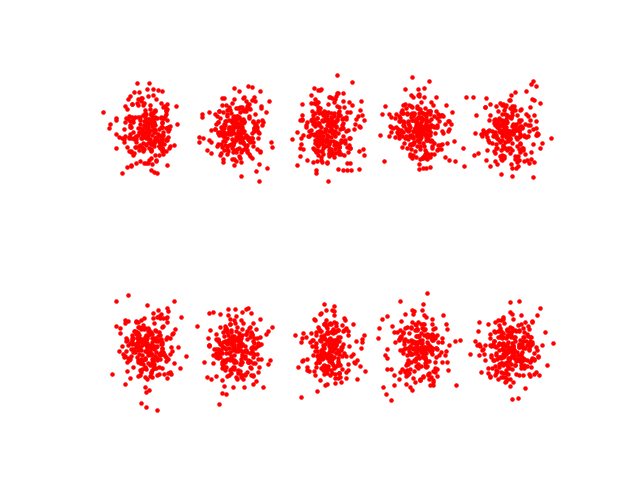}
  \includegraphics[scale=0.28]{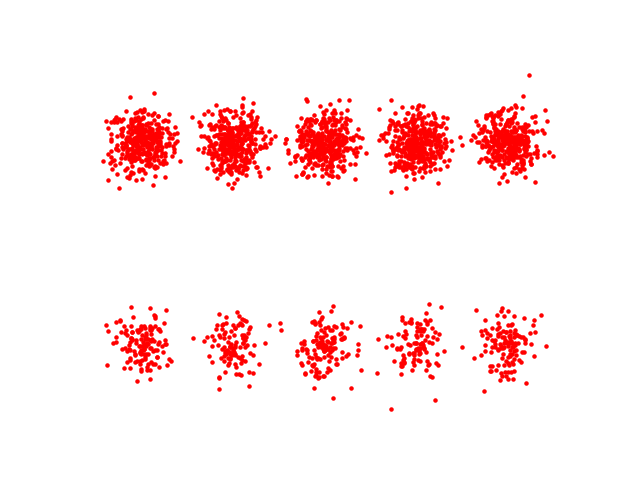}
  \includegraphics[scale=0.28]{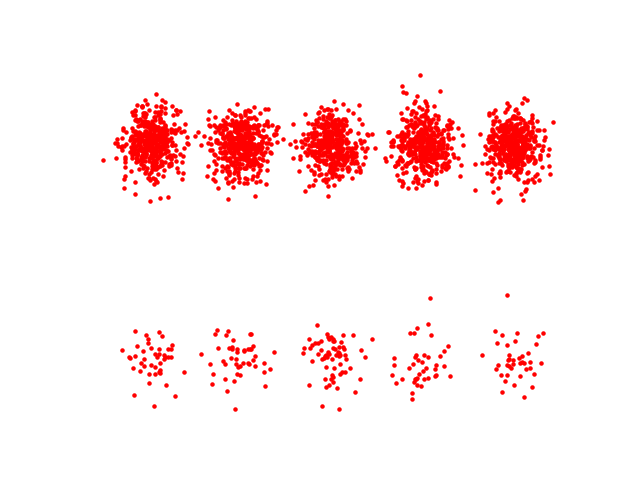}
\caption{Samples of input measures}
\label{fig:input_mixture_gauss}
\end{figure}

\begin{figure}[ht]

\centering
  \includegraphics[scale=0.25]{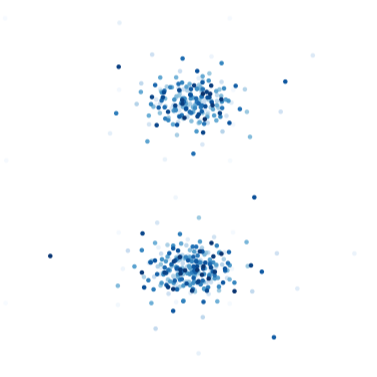}
  \includegraphics[scale=0.25]{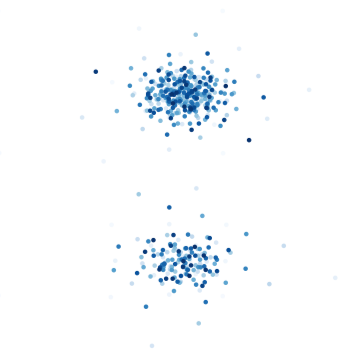}  
  \includegraphics[scale=0.25]{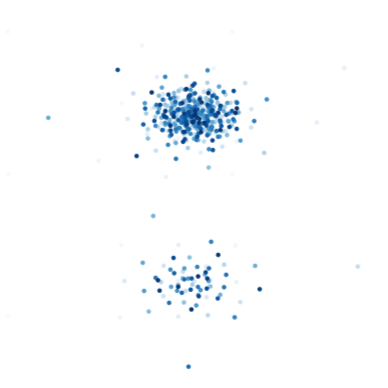}
\caption{Barycenters of Mixture of Gaussians}
\label{fig:bary_mixture_gauss}
\end{figure}

\begin{figure}[b]
\centering
  \includegraphics[scale=0.3]{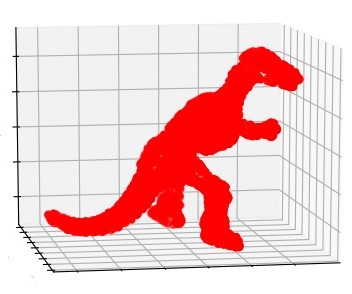}\qquad\qquad
  \includegraphics[scale=0.3]{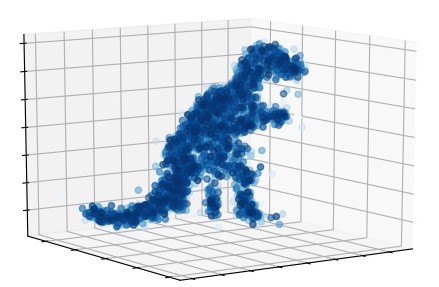}
  \caption{3D dinosaur mesh (left), barycenter of 3D meshes (right)\label{fig:dinosaur}}
  \end{figure}

\paragraph{Propagation} We extend the description on the experiment about propagation in \cref{sec:experiments}.  Edges $\mathcal{E}$ are selected as follows: we  created a matrix $D$ such that $D_{ij}$ contains the distance between station at vertex $i$ and station at vertex $j$, computed using the geographical coordinates of the stations. Each node $v$ in  $\mathcal{V}$, is connected to those nodes $u\in\mathcal{V}$ such that $D_{vu} \leq 3$. If the number of nodes $u$ that meet this condition is \textit{less} than $5$, we connect $v$  with its $5$ nearest nodes. If the number of nodes $u$ that meet this condition is \textit{more} than $10$, we connect $v$  with its $10$ nearest nodes.  Each edge $e_{uv}$ is weighted with $\omega_{uv}:=D_{uv}$. Since intuitively we may expect that nearer nodes should have more influence in the construction of the histograms of unknown nodes, in the propagation functional we weight $\sink(\rho_v,\rho_u)$ with  use $\exp(-\omega_{uv}/\sigma)$ or $1/\omega_{vu}$ suitably normalized.

\paragraph{Large scale discrete measures: meshes} We perform the barycenter of two discrete measures with support in $\R^3$. Meshes of the dinosaur are taken from \cite{solomon2015convolutional} 
and rescaled by a 0.5 factor. The internal problem in Frank-Wolfe algorithm is solved using L-BFGS-B SciPy optimizer. Formula of the Jacobian is passed to the method. The barycenter is displayed in \cref{fig:dinosaur} together with an example of the input.



\end{document}